\documentclass[twoside,11pt]{article}

\usepackage{jmlr2e}
\usepackage{times}
\usepackage{bm}
\usepackage{graphicx,xspace} 
\usepackage{float,subfigure,enumitem}
\usepackage{url}
\usepackage{epstopdf}
\usepackage{algorithm,algorithmic}
\usepackage{amsmath,mathrsfs,amssymb,array,amsfonts,mathtools}
\usepackage[dvipsnames]{xcolor}
\usepackage{booktabs,multirow,rotating}
\usepackage[normalem]{ulem}
\usepackage[flushleft]{threeparttable}
\usepackage{hyperref}
\hypersetup{backref,colorlinks=true,linkcolor=red,citecolor=green,urlcolor=blue}

\graphicspath{{./}{./figures/}}
\DeclareGraphicsExtensions{.pdf,.jpeg,.png,.jpg}

\makeatletter
\DeclareRobustCommand\onedot{\futurelet\@let@token\@onedot}
\def\@onedot{\ifx\@let@token.\else.\null\fi\xspace}

\def\eg{\emph{e.g}\onedot} 
\def\ie{\emph{i.e}\onedot}

\def\etal{\emph{et al}\onedot}
\makeatother

\def\Vec#1{{\boldsymbol{#1}}}
\def\Mat#1{{\boldsymbol{#1}}}

\newcommand{\tr}{\mathop{\rm  Tr}\nolimits}

\ShortHeadings{Analyzing Linear Dynamical Systems: From Modeling to Coding and Learning}{Huang \etal}
\firstpageno{1}

\useunder{\uline}{\ul}{}

\begin{document}
\title{Analyzing Linear Dynamical Systems:\\ From Modeling to Coding and Learning}

\author{\name Wenbing Huang \email huangwb12@mails.tsinghua.edu.cn \\
       \addr Department of Computer Science and Technology\\
       Tsinghua University\\
       Beijing, China
       \AND
       \name Fuchun Sun \email fcsun@mail.tsinghua.edu.cn \\
       \addr Department of Computer Science and Technology\\
       Tsinghua University\\
       Beijing, China
       \AND
       \name Lele Cao \email caoll12@mails.tsinghua.edu.cn \\
       \addr Department of Computer Science and Technology\\
       Tsinghua University\\
       Beijing, China
       \AND
       \name Mehrtash Harandi \email mehrtash.harandi@data61.csiro.au \\
       \addr College of Engineering and Computer Science\\
       Australian National University \& Data61, CSIRO\\
       Canberra, Australia
       }


\maketitle

\begin{abstract}
Encoding time-series with Linear Dynamical Systems (LDSs) leads to rich models with applications ranging from
dynamical texture recognition to video segmentation to name a few.
In this paper, we propose to represent LDSs with infinite-dimensional subspaces and derive an analytic solution
to obtain stable LDSs. We then devise efficient algorithms to perform sparse coding and dictionary learning on the space of infinite-dimensional subspaces. In particular, two solutions are developed to sparsely  encode an LDS. In the first method, we map the subspaces into a Reproducing Kernel Hilbert Space (RKHS) and achieve our goal through kernel sparse coding. As for the second solution, we propose to embed the infinite-dimensional subspaces into the space of symmetric matrices and formulate the sparse coding accordingly in the induced space. For dictionary learning, we encode time-series by introducing a novel concept, namely the two-fold LDSs.
We then make use of the two-fold LDSs to derive an analytical form for updating atoms of an LDS dictionary, \ie, each atom is an LDS itself. Compared to several baselines and state-of-the-art methods, the proposed methods yield higher accuracies in various classification tasks including video classification and tactile recognition.
\end{abstract}


\newpage
\section{Introduction}
\label{sec:intro}

This paper introduces techniques to encode and learn from Linear Dynamical Systems (LDSs).
Analyzing, classifying and prediction from time-series is an active and multi-disciplinary research area.
Examples include financial time-series forecasting~\cite{kim2003financial}, the analysis of video data~\cite{afsari2012group} and biomedical data~\cite{brunet2011spatiotemporal}.

Inference from time-series is \textbf{not}, by any measure, an easy task~\cite{afsari2014distances,ravichandran2013categorizing}.
A reasonable and advantageous strategy, from both theoretical and computational points of view,
is to simplify the problem by assuming that time-series  are generated by models from a specific parametric class.
Modeling time-series by LDSs is of one such attempt, especially when facing high-dimensional time-series (\eg videos).
The benefits of modeling with LDSs are twofold:
\textbf{I.} the LDS model enables a rich representation, meaning LDSs can approximate a large class of stochastic
processes~\cite{afsari2014distances},
\textbf{II.} compared to vectorial ARMA models~\cite{johansen1995likelihood}, LDSs are less prone to the curse of dimensionality. This is an attractive property for vision applications where data is usually high-dimensional.

In the past decade, sparse coding has been successfully exploited in various vision tasks such as image
restoration~\cite{mairal2008sparse}, face recognition~\cite{wright2009robust}, and texture classification~\cite{mairal2009supervised} to name a few. In sparse coding,  natural signals such as images are represented by a combination of a few basis elements (or atoms of a dictionary).
While being extensively studied, little is known on combining sparse coding techniques with LDS modeling to yield robust techniques.
In this paper, we generalize sparse coding from Euclidean spaces to the space of LDSs. In particular, we show how an LDS can be reconstructed by a superposition of LDS atoms, while the coefficients of the superposition are enforced to be sparse.
We also show how a dictionary of LDS atoms can be learned from data. Sparse coding with the learned LDS dictionary can then be seamlessly used to categorize spatio-temporal data. The importance of our work lies in the fact that to achieve our goals, we need to develop techniques that work with the non-Euclidean space of LDSs~\cite{afsari2012group,ravichandran2013categorizing}.

\subsection*{Contributions.}
\begin{enumerate}[leftmargin = *]
\item Unlike previous attempts that model LDSs with finite-dimensional subspaces, we propose to describe LDSs by infinite-dimensional subspaces. It will be shown that infinite-dimensional modeling not only encodes the full evolution of the sequences but also reduces the computational cost.

\item We propose a simple, yet effective way to stabilize the transition matrix of an LDS.
We show that while the stabilization is done in closed-form, the transition matrices maintain sufficient discriminative information to accommodate classification.

\item To perform sparse coding, we propose two techniques that work with infinite-dimensional subspaces. As for the first technique,
we map the infinite-dimensional subspaces into a Reproducing Kernel Hilbert Space (RKHS) and formulate the coding problem as a kernel sparse coding problem. In the second approach, we make use of a diffeomorphic mapping to embed the infinite-dimensional subspaces into the space of symmetric matrices and formulate the sparse coding in the induced space.

\item For dictionary learning, we propose to encode the time-series with a novel concept, namely the two-fold LDS. A two-fold LDS can be understood as an structured LDS and enables us to derive an analytical form for updating the dictionary atoms.


\end{enumerate}

Before concluding this part, we would like to highlight that the proposed algorithms outperform state-of-the-art methods on various
recognition tasks including video classification and tactile recognition; Figure~\ref{Fig:LDS-modeling} demonstrates a conceptual diagram
of the methods developed in this paper.

\begin{figure*}[!ht]
\begin{center}
\subfigure{
\includegraphics[width=14cm]{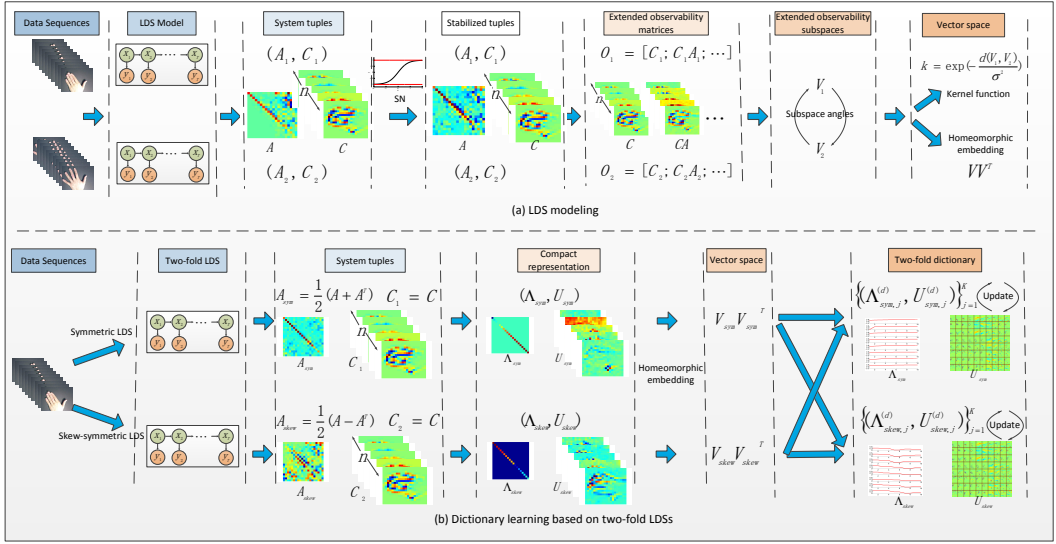}
}
\caption{The flowcharts of our LDS modeling formulation and dictionary learning algorithm.}
\label{Fig:LDS-modeling}
\end{center}
\end{figure*}

\section{Related Work}
\label{sec:related_work}

To analyze LDSs, one should start with a proper geometry.
As a result of an invariance property (will be discussed in \textsection~\ref{Sec:O}), the space of LDSs is not
Euclidean~\cite{afsari2012group,ravichandran2013categorizing}.
Worse is the fact that the proper geometry, \ie, a structure capturing the invariances imposed by LDSs,
is still not fully developed~\cite{ravichandran2013categorizing}.
Nevertheless, various metrics such as
Kullback-Leibler divergence~\cite{chan2005probabilistic}, Chernoff distance~\cite{woolfe2006shift}, Binet-Cauchy kernel~\cite{vishwanathan2007binet} and group distance~\cite{afsari2012group}
have been proposed to measure the distances between LDSs.
An alternative solution is to make use of the \emph{extended observability subspace} to represent an
LDS~\cite{saisan2001dynamic,chan2007classifying,ravichandran2013categorizing,turaga2011statistical}. Comparing LDSs is then achieved by measuring the \emph{subspace angles} as applied for example in the Martin distance~\cite{martin2000metric,de2002subspace}.

Recent studies attempt to approximate the \emph{extended observability} matrix of an LDS by a finite-dimensional subspace~\cite{turaga2011statistical,harandi2013dictionary, harandi2014extrinsic}. With such modeling, the geometry of finite-dimensional Grassmannians can be used to analyze LDSs. For example, the geometry of finite-dimensional Grassmannian is used to perform sparse coding and dictionary learning in~\cite{harandi2014extrinsic} and clustering in~\cite{turaga2011statistical}.

One obvious drawback of the aforementioned school of thought and what we avoid in this paper is modeling with finite-dimensional Grassmannians.
For example and in the context of dictionary learning, with the finite approximation, only
a dictionary of finite observability matrices can be learned.
In general, an LDS is identified by its \emph{measurement} and \emph{transition} matrices, both are necessary for further analysis
(\eg, video registration~\cite{ravichandran2011video}). To the best of our knowledge, while a finite approximation to the observability matrix can be obtained from the measurement and transition matrices, the reverse action  (\ie, obtaining the measurement and transition matrices from the finite observability matrix) is not possible. On the contrary and as we will see later, infinite-dimensional modeling enables us to learn the system parameters of the dictionary explicitly and reduces the computational cost significantly.
We draw the reader's attention to similar observations made in the context of classification (see for example
~\cite{saisan2001dynamic, chan2007classifying, ravichandran2013categorizing}), hinting that coding and dictionary learning with infinite LDSs can be fruitful.



In our preliminary study~\cite{wenbing2016sparse}, to learn an LDS dictionary, we assumed that the transition matrices of LDS models are symmetric. Encoding sequences with symmetric LDSs limits the generalization power to some extent.
In this work, we extend the methods developed in~\cite{wenbing2016sparse} and model sequences via two-fold LDSs.
A two-fold LDS enriches the symmetric models~\cite{wenbing2016sparse} by having a skew-symmetric part along with the symmetric one. We show that the system parameters of a two-fold LDS can be obtained similar to a conventional LDS, hence richer models can be obtained
without incurring heavy computations. Furthermore, by two-fold LDS modeling, we are able to derive efficient algorithms to update an LDS dictionary with two-fold LDS atoms.

\section{Notation}
Throughout the paper, bold capital letters denote matrices (\eg, $\Mat{X}$) and bold lower-case letters denote column vectors (\eg, $\Vec{x}$).
$\mathbf{I}_n$ is the $n \times n$ identity matrix,
$\|\cdot\|_1$,~$\|\cdot\|_2$ and $\|\cdot\|_F$ denote $\ell_1$, $\ell_2$ and Frobenius norm, respectively.
$\Mat{X}^{\mathrm{T}}$ computes the matrix transposition of $\Mat{X}$.
The Hermitian  transpose of a matrix is shown $^{\ast}$, \ie, $\Mat{X}^{\ast}$;
and $\tr(\Mat{X})$ is the trace operator. $[\Mat{X}]_{k}$ returns the $k$-th column of a general matrix $\Mat{X}$, and returns \emph{only} the $k$-th diagonal element if $\Mat{X}$ is diagonal; $x_{i,j}$ shows the element at $i$-th row and $j$-th column of  $\Mat{X}$;
$x_i$ returns the $i$-th element of the vector $\Vec{x}$. The symbol $1i$ is the imaginary unit.

\section{LDS Modeling}
\label{Sec: Fundamental_Concepts}
\subsection{LDSs} \label{Sec:LDS}
An LDS describes a time series through the following model:
\begin{eqnarray}\label{Eq:ARMA}
\left\{
\begin{aligned}
&\Mat{x}(t+1) &= \Mat{A}\Mat{x}(t) + \Mat{B}\Mat{v}(t), \\
&\Mat{y}(t)   &= \Mat{C}\Mat{x}(t) + \Mat{w}(t) + \overline{\Mat{y}},
\end{aligned}
\right.
\end{eqnarray}
where $\Mat{X}=[\Mat{x}(1), \cdots, \Mat{x}(\tau)] \in \mathbb{R}^{n\times \tau}$ is a sequence of $n$-dimensional hidden state vectors, and $\Mat{Y}=[\Mat{y}(1), \cdots, \Mat{y}(\tau)]\in \mathbb{R}^{m\times \tau}$ is a sequence of $m$-dimensional observed variables. The model is parameterized by $\Theta=\{\Mat{A}, \Mat{B}, \Mat{C}, \Mat{R}, \overline{\Mat{y}}\}$, where $\Mat{A}\in \mathbb{R}^{n\times n}$ is the transition matrix; $\Mat{C}\in \mathbb{R}^{m\times n}$ is the measurement matrix; $\Mat{B}\in\mathbb{R}^{n\times n_v}(n_v \leq n)$ is the noise transformation matrix; $\Mat{v}(t)\sim \mathcal{N}(0,\Mat{I}_{n_v})$ and $\Mat{w}(t)\sim \mathcal{N}(0,\Mat{R})$ denote the process and measurement noise components, respectively; $\overline{\Mat{y}}\in \mathbb{R}^m$ represents the mean of $\Mat{Y}$.

Given the observed sequence, several methods~\cite{van1994n4sid, shumway1982approach} have been proposed to learn the optimal system parameters, while the method in \cite{doretto2003dynamic} is widely used. This approach first estimates the state sequence by performing PCA on the observations, and then learns the dynamics in the state space via the least square method. We denote the centered observation matrix as $\Mat{Y}'=[\Mat{y}(1)-\overline{\Mat{y}}, \cdots, \Mat{y}(\tau)-\overline{\Mat{y}}]$. Factorizing $\Mat{Y}'$ by Singular Value Decomposition (SVD)
yields $\Mat{Y}' =\Mat{U}_Y\Mat{S}_Y\Mat{V}_Y^{\mathrm{T}}$ where $\Mat{U}_Y\in\mathbb{R}^{m\times n}$, $\Mat{U}_Y^{\mathrm{T}}\Mat{U}_Y=\Mat{I}_n$, $\Mat{S}_Y\in\mathbb{R}^{n\times n}$, $\Mat{V}_Y\in\mathbb{R}^{n \times \tau}$, and $\Mat{V}_Y^{\mathrm{T}}\Mat{V}_Y=\Mat{I}_{\tau}$.
The measurement matrix $\Mat{C}$ and the hidden states $\Mat{X}$ are then estimated as $\Mat{U}_Y$ and $\Mat{S}_Y\Mat{V}_Y^{\mathrm{T}}$, respectively. The transition matrix $\Mat{A}$ is learned by minimizing the state reconstruction error $ J^2(\Mat{A}) = \|\Mat{X}(1)-\Mat{A}\Mat{X}(0)\|_{F}^2 $, where $\Mat{X}(0)=[\Mat{x}(1),\cdots,\Mat{x}(\tau-1)]$, $\Mat{X}(1)=[\Mat{x}(2),\cdots,\Mat{x}(\tau)]$. The optimal $\Mat{A}$ is given by $\Mat{A}=\Mat{X}(1)\Mat{X}(0)^\dag$ with $\dag$ denoting the pseudo-inverse of a matrix. Other parameters of LDSs like $\Mat{B}$ and $\Mat{R}$  can be estimated when $\Mat{A}$ and $\Mat{C}$ are obtained (see for example~\cite{doretto2003dynamic} for details).

\subsection{Learning stable LDSs}
\label{Sec:stable_LDS}

An LDS is  stable if and only if the eigenvalues of its transition matrix are smaller than $1$~\cite{siddiqi2008constraint}.
The stability is an important property, because an unstable LDS can cause significant distortion to an input sequence \cite{wenbing2016learning}. Also, we will show later that having stable LDSs is required when it comes to computing the extended observability subspaces. Since the transition matrix $\Mat{A}$ leaned by the method discussed in \cite{doretto2003dynamic}
is not naturally stable, various methods have been proposed to enforce stability on LDSs \cite{lacy2003subspace1,lacy2003subspace,siddiqi2008constraint,wenbing2016learning}. All these methods iteratively update the transition matrix by minimizing the state reconstruction error while satisfying the stability constraint.
In this paper, however, we devise an analytic and light-weight method to obtain stable LDSs. To be specific, given the transition matrix $\Mat{A}$ computed by the method in \cite{doretto2003dynamic},
we factorize $\Mat{A}$ using SVD as $\Mat{A} = \Mat{U}_{\Mat{A}}\Mat{S}_{\Mat{A}}\Mat{V}_{\Mat{A}}^{\mathrm{T}}$. Then, we smooth out
the diagonal elements of $\Mat{S}_{\Mat{A}}$ to be within $(-1,1)$
using
\begin{eqnarray}\label{Eq:SN}
\Mat{S}_{\Mat{A}}^{\prime} &=& 2\mathrm{Sig}(a\Mat{S}_{\Mat{A}})-1.
\end{eqnarray}
Here $\mathrm{Sig}(\cdot)$ is the Sigmoid function and $a>0$ is a scale factor. The new transition matrix is $\Mat{A}^\prime = \Mat{U}_{\Mat{A}}\Mat{S}_{\Mat{A}}^\prime\Mat{V}_{\Mat{A}}^{\mathrm{T}}$ and is obviously \emph{strictly stable}\footnote{Here, \emph{strictly stable} means that the magnitude of the eigenvalues is strictly less than 1.}. We call this method as Soft-Normalization (SN). Compared to previous methods, SN involves no optimization process, making it scalable to large-scale problems. Besides, due to the saturation property of the Sigmoid function, SN penalizes the singular values  that are near or outside the stable bound, without heavily sacrificing the information encoded in the original transition matrix. The effectiveness of SN will be further demonstrated by our experiments in \textsection~\ref{Sec:exp_SC}.

\begin{figure}[!h]
\begin{center}
\subfigure{
\includegraphics[width=\columnwidth]{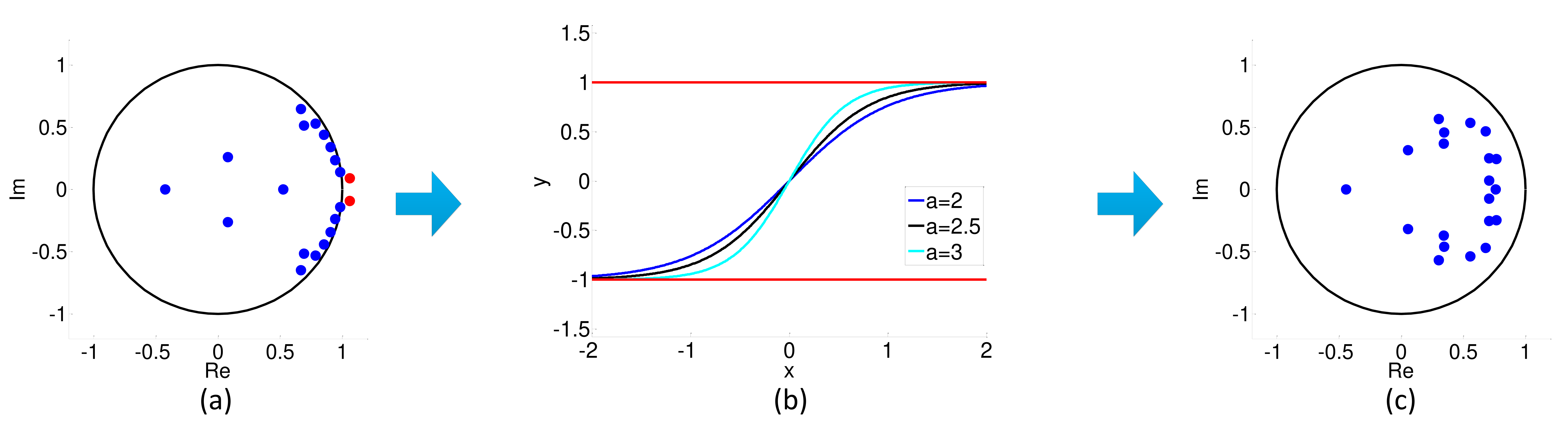}
}
\caption{Illustration of SN on a sequence from the \emph{Cambridge} dataset~\cite{kim2009canonical}. (a) The complex plane displaying the original  eigenvalues;
(b) The SN function with various scaling values $a$; (c) The distribution of the transition eigenvalues after SN stabilization with $a=2.5$.}
\label{Fig:SN}
\end{center}
\end{figure}

\subsection{LDS Descriptors}
\label{Sec:O}

In an LDS, $\Mat{C}$ describes the spatial appearance ($\Mat{C}$ needs to be orthogonal) and $\Mat{A}$ represents the dynamics ($\Mat{A}$ needs to be stable). Therefore, the tuple $(\Mat{A},\Mat{C})$ can be  used describe an LDS. A difficulty in analyzing LDSs stems from the fact that the tuple $(\Mat{A}$,$\Mat{C})$ does not lie in a vector space~\cite{turaga2011statistical}.
In particular, for any orthogonal matrix $\Mat{P}\in \mathbb{R}^{n\times n}$, $(\Mat{A}, \Mat{C})$ is equivalent to $(\Mat{P}^{\mathrm{T}}\Mat{A}\Mat{P}, \Mat{CP})$ as they represent the very same system. To circumvent this difficulty, a family of approaches opts for the extended observability subspace to represent an LDS \cite{saisan2001dynamic,chan2007classifying,ravichandran2013categorizing,turaga2011statistical}. However and to our best of knowledge, the exact form of the extended observability subspace has never been used before. Below, we will derive the exact form of the extended observability subspace in a systematic way.

\noindent\textbf{The Extended Observability Matrix.} \\
Starting from the initial state $\Mat{x}(1)$, the expected observation sequence is obtained as
\begin{eqnarray}\label{Eq:O}
[\mathrm{E}[\Mat{y}(1)];\mathrm{E}[\Mat{y}(2)];\mathrm{E}[\Mat{y}(3)];\cdots] = [\Mat{C};\Mat{C}\Mat{A};\Mat{C}\Mat{A}^{2};\cdots]\Mat{x}(1),
\end{eqnarray}
where the extended observability matrix is given by $\Mat{O}=[\Mat{C};\Mat{C}\Mat{A};\Mat{C}\Mat{A}^{2};\cdots]\in \mathbb{R}^{\infty\times n}$. It soon becomes clear that the extended observability matrix can encode the (expected) temporal evolution of the LDS till the infinity. Besides, the column space of $\Mat{O}$, \emph{i.e.}  the extended observability subspace, is invariant to the choice of the basis of the state space. Such two properties make the extended observability subspaces suitable  for describing LDSs. From here onwards, we show the set of extended observability matrices with $n$-dimensional hidden states by $\mathcal{O}(n, \infty)$.

\noindent\textbf{Inner-product between Observability Matrices.}\\
The inner-product between two extended observability matrices $\Mat{O}_1$ and $\Mat{O}_2$, \ie, $\Mat{O}_{12}=\Mat{O}_1^{\mathrm{T}}\Mat{O}_2=\sum_{t=0}^{\infty}(\Mat{A}_1^t)^{\mathrm{T}}\Mat{C}_1^{\mathrm{T}} \Mat{C}_2\Mat{A}_2^t$,  can be computed by solving the following Lyapunov Equation
\begin{equation}\label{Eq:Lyapunov1}
\Mat{A}_1^{\mathrm{T}}\Mat{O}_{12}\Mat{A}_2-\Mat{O}_{12}=-\Mat{C}_1^{\mathrm{T}}\Mat{C}_2,
\end{equation}
whose solution exists and is unique if both $\Mat{A}_1$ and $\Mat{A}_2$ are strictly stable \cite{de2002subspace}.

\noindent\textbf{Extended Observability Matrix with Orthonormal Columns.}\\
To derive the extended observability subspace, we need to orthonormalize $\Mat{O}$.
In our preliminary study, orthonormalization was done using the Cholesky decomposition~\cite{wenbing2016sparse}.
Different from~\cite{wenbing2016sparse}, we propose to perform orthonormalization using SVD as we observed that
SVD is more flexible than the Cholesky decomposition even when the system is unstable. To this end, we use SVD to
factorize $\Mat{O}^{\mathrm{T}}\Mat{O}=\Mat{U}_o\Mat{S}_o\Mat{U}_o^{\mathrm{T}}$. The columns of $\Mat{V}=\Mat{O}\Mat{L}^{-\mathrm{T}}$ are orthogonal and span the same subspace as the columns of $\Mat{O}$, where the factor matrix $\Mat{L}=\Mat{U}_o\Mat{S}_o^{1/2}$.
Thus, $\Mat{V}$ is the orthonormal extended observability matrix. We denote the set of the orthogonal extended observability matrices
by $\mathcal{V}(n,\infty)$.

\noindent\textbf{Extended Observability Subspaces.}\\
Let $\mathcal{S}(n,\infty)$ be the set of extended observability subspaces. $\mathcal{S}(n,\infty)$ is the quotient space of $\mathcal{V}(n,\infty)$ with the equivalence relation $\sim$ being: for any $\Mat{V}_1,\Mat{V}_2\in \mathcal{V}(n,\infty)$, $\Mat{V}_1\sim\Mat{V}_2$ if and only if $\mathrm{Span}(\Mat{V}_1) = \mathrm{Span}(\Mat{V}_2)$, where $\mathrm{Span}(\Mat{V})$ denotes the subspace spanned by the columns of $\Mat{V}$.
From the definition of $\mathcal{S}(n,\infty)$, one can conclude immediately that $\mathcal{S}(n,\infty)$ is a special
case of infinite Grassmannian $\mathcal{G}(n,\infty)$~\cite{ye2014distance} with an extra intrinsic structure due to the stability and orthonormality constraints on $\Mat{A}$ and $\Mat{C}$, respectively.

A valid distance between subspaces must be a function of their principle angles~\cite{ye2014distance}. The definition of principle angles for extended observability subspaces have been provided in \cite{de2002subspace}, where the $n$ principal angles $0\leq\alpha_1\leq\alpha_2\leq\cdots\leq\alpha_n\leq\pi/2$ between two extended observability subspaces $\mathrm{Span}(\Mat{V}_1)$ and $\mathrm{Span}(\Mat{V}_2)$ are defined recursively by
\begin{eqnarray}\label{Eq:subspace_angles}
\begin{aligned}
 & \cos\alpha_k  & = & \max_{\Mat{u}_k\in\mathrm{Span}(\Mat{V}_1)}\max_{\Mat{v}_k\in\mathrm{Span}(\Mat{V}_2)}\Mat{u}_k^{\mathrm{T}}\Mat{v}_k\\
 & ~~~\mathrm{s.t. } & &
 \Mat{u}_k^{\mathrm{T}}\Mat{u}_k=\Mat{v}_k^{\mathrm{T}}\Mat{v}_k = 1 \\
 &&& \Mat{u}_k^{\mathrm{T}}\Mat{u}_j=\Mat{v}_k^{\mathrm{T}}\Mat{v}_j = 0, j=1,2,\cdots,k-1
\end{aligned}
\end{eqnarray}
Principal angles can be calculated more efficiently using $ \cos\alpha_k = s_k(\Mat{V}_1^{\mathrm{T}}\Mat{V}_2)$,
with $s_k(\cdot)$ denoting the $k$-th singular values and $\Mat{V}_1^{\mathrm{T}}\Mat{V}_2=\Mat{L}_1^{-1}\Mat{O}_1^{\mathrm{T}}\Mat{O}_2\Mat{L}_2^{\mathrm{-T}}$. Having principal angles at our disposal, various distances between LDSs can be defined.
A widely used one is the Martin distance~\cite{martin2000metric} defined as:
\begin{eqnarray}\label{Eq:Martin}
d(\Mat{V}_1,\Mat{V}_2)=\sqrt{-\log\prod_{k=1}^n\cos^2\alpha_k}.
\end{eqnarray}

\section{Sparse Coding with LDSs}
\label{Sec:SC}

In this section, we will show how a given LDS can be sparsely coded if an LDS dictionary (\ie, each atom is an LDS)
is at our disposal. We recall that the extended observability subspaces are points on an infinite-dimensional Grassmannian. As such,
conventional sparse coding techniques designed for vectorial data cannot be used  to solve the coding problem.
Here, we propose two strategies to perform sparse coding on LDSs.
First, we propose a kernel solution by making use of an implicit mapping from the infinite-dimensional Grassmannian to RKHS.
In doing so, we design a subspace kernel function which enables us to perform kernel sparse coding~\cite{gao2010kernel} on LDSs. One drawback of the kernel solution is that, the feature map is implicit. This makes learning an LDS dictionary intractable, when the goal is to have explicit atoms. We thus develop another method for sparse coding by embedding the infinite Grassmannian into the space of symmetric matrices
by a diffemorphic mapping. We will then show how sparse coding and dictionary learning can be performed in the space defined by the diffemorphic mapping.

\subsection{Kernel Sparse Coding}\label{Sec:Kernel_SC}
The idea of sparse coding is to approximate a given sample $\Mat{x} \in \mathbb{R}^{n}$ with atoms $\{\Mat{d}_j\in \mathbb{R}^{n}\}_{j=1}^J$ of an over-complete dictionary.  Moreover, we want the solution to be sparse, meaning that only a few elements of the solution are non-zero.
In the Euclidean space, sparse codes can be obtained by solving the following optimization problem
\begin{eqnarray}\label{Eq:Euler_SC}
\min_{\Mat{z}} l(\Mat{z},\Mat{d}) = \| \Mat{x}-\sum_{j=1}^{J}z_j\Mat{d}_j \|_2^2 + \lambda\| \Mat{z}\|_1,
\end{eqnarray}
where $\lambda$ is the sparsity penalty factor.

Directly plugging the extended observability matrices into~\eqref{Eq:Euler_SC} is obviously not possible. To get around the difficulty
of working with points in $\mathcal{S}(n,\infty)$,  we propose to implicitly map the subspaces in $\mathcal{S}(n,\infty)$ to an RKHS
$\mathcal{H}$. Let us denote the implicit mapping and its associated kernel by $\phi: \mathcal{S}(n,\infty)\rightarrow\mathcal{H}$
and $k(\cdot,\cdot):\mathcal{S}(n,\infty) \times \mathcal{S}(n,\infty) \to \mathbb{R}$ with the property
$k(\Mat{X}_1,\Mat{X}_2) = \phi(\Mat{X}_1)^T \phi(\Mat{X}_2)$, respectively. This enables us to formulate sparse coding in $\mathcal{H}$ as
\begin{eqnarray}\label{Eq:Kernel_SC1}
\min_{\Mat{z}} l(\Mat{z},\Mat{D}) = \| \phi(\Mat{X})-\sum_{j=1}^{J}z_j\phi(\Mat{D}_j) \|_2^2 + \lambda
\| \Mat{z} \|_1,
\end{eqnarray}
where $\Mat{X}\in\mathcal{S}(n,\infty)$ and $\{\Mat{D}_j\in\mathcal{S}(n,\infty)\}_{j=1}^J$ is the LDS dictionary.
The problem in~(\ref{Eq:Kernel_SC1}) can be simplified into~\cite{gao2010kernel}
\begin{eqnarray}\label{Eq:Kernel_SC2}
\min_{\Mat{z}} l(\Mat{z},\Mat{D}) = \Mat{z}^{\mathrm{T}}\Mat{K}_{\Mat{D}}\Mat{z}-2\Mat{z}^{\mathrm{T}}\Mat{k}_{\Mat{XD}} + \lambda\parallel\Mat{z}\parallel_1,
\end{eqnarray}
with $\Mat{K}_{\Mat{D}}\in\mathbb{R}^{J \times J}, \Mat{K}_{\Mat{D}}(i,j)=k(\Mat{D}_i, \Mat{D}_j)$ and
$\Mat{k}_{\Mat{XD}}\in\mathbb{R}^J, \Mat{k}_{\Mat{XD}}(j)=k(\Mat{X},\Mat{D}_j)$. In our experiments, we use the Radial-Basis-Function (RBF) kernel based on Martin distance as
\begin{eqnarray}\label{Eq:Gaussian_Kernel}
k(\Mat{X}_1,\Mat{X}_2) = \exp(-\frac{d^2(\Mat{X}_1,\Mat{X}_2)}{\sigma^2}),
\end{eqnarray}
where $d(\Mat{X}_1,\Mat{X}_2)$ is the Martin distance defined in Eq.~\eqref{Eq:Martin}. Once the kernel values are computed,  methods like homotopy-LARS algorithm~\cite{donoho2008fast} can be used to solve \eqref{Eq:Kernel_SC2}. The RBF kernel was also used by Chan \etal~\cite{chan2007classifying}; however, the positive definiteness of this kernel was neither proven nor
discussed. Fortunately, this kernel are always positive definite for the experiments in the paper.

\subsection{Sparse Coding by Diffemorphic Embedding}\label{Sec:diffemorphic_embedding}

Different form the kernel-based method mentioned in the previous section, we construct an explicit and diffemorphic mapping to facilitate sparse coding. Inspired by the method proposed in~\cite{harandi2014extrinsic}, we embed $\mathcal{S}(n,\infty)$ into the space of symmetric matrices via the mapping $\Pi:\mathcal{S}(n,\infty)\rightarrow \text{Sym}(\infty),\Pi(\Mat{V})=\Mat{V}\Mat{V}^{\mathrm{T}}$. The metric on Sym($\infty$) is naturally induced by the Frobenius norm: $\| \Mat{W} \|_{F}^2=\text{Tr}(\Mat{W}^{\mathrm{T}}\Mat{W})$, $\Mat{W}\in\text{Sym}(\infty)$. We note that in general the Frobenius norm of a point on $\mathrm{Sym}(\infty)$ is infinite as a result of infinite dimensionality. However, the Frobenius norm of an embedded point $\Pi(\mathcal{S}(n,\infty))$ is guaranteed to be finite. To prove this, we make use of the following theorem.

\begin{theorem}\label{Th:linear_combination}
Suppose $\Mat{V}_1,\Mat{V}_2,\cdots,\Mat{V}_M\in \mathcal{S}(n,\infty)$, and  $y_1,y_2,\cdots,y_M\in \mathbb{R}$. Then,
\begin{eqnarray}
\nonumber \|\sum_{i=1}^{M}y_i\Pi(\Mat{V}_i)\|_{F}^2 &=& \sum_{i,j=1}^{M}y_iy_j\|\Mat{V}_i^{\mathrm{T}}\Mat{V}_j\|_{F}^2,
\end{eqnarray}
\end{theorem}
\begin{proof}
The proof is provided in the appendix.
\end{proof}
\noindent Based on this Theorem, the following corollary can be drawn:
\begin{corollary}\label{metric}
For any $\Mat{V}_1,\Mat{V}_2\in\mathcal{S}(n,\infty)$, we have
\begin{eqnarray}
\nonumber \|\Pi(\Mat{V}_1)-\Pi(\Mat{V}_2)\|_{F}^2 &=& 2\left(n -\|\Mat{V}_1^{\mathrm{T}}\Mat{V}_2\|_F^2\right).
\end{eqnarray}
Furthermore, $\| \Pi(\Mat{V}_1)-\Pi(\Mat{V}_2)\|_{F}^2=2\sum_{k=1}^{n}\sin^2\alpha_k$, where $\{\alpha_k\}_{k=1}^n$ are subspace angles between $\Mat{V}_1$ and $\Mat{V}_2$. This also indicates that $0 \leq \| \Pi(\Mat{V}_1)-\Pi(\Mat{V}_2)\|_{F}^2 \leq 2n$.
\end{corollary}

We note that the mapping $\Pi(\Mat{V})$ is diffeomorphism (a one-to-one, continuous, and differentiable mapping with a continuous and differentiable inverse), meaning that $\mathcal{S}(n,\infty)$ is topologically isomorphic to the embedding $\Pi(\mathcal{S}(n,\infty))$,
\ie, $\mathcal{S}(n,\infty)\cong\Pi(\mathcal{S}(n,\infty))$.

With $\Pi(\cdot)$ at our disposal, the sparse coding can be formulated as $\min\limits_{\Mat{z}}l(\Mat{z},\Mat{D})$ where,
\begin{eqnarray}\label{Eq:SC1}
l(\Mat{z},\Mat{D}) = \| \Mat{X}\Mat{X}^{\mathrm{T}}-\sum_{j=1}^{J}z_j\Mat{D}_j\Mat{D}_j^{\mathrm{T}}\|_{F}^2+\lambda\|\Mat{z}\|_1,
\end{eqnarray}
Here $\Mat{X}\in\mathcal{S}(n,\infty)$, $\{\Mat{D}_j\in\mathcal{S}(n,\infty)\}_{j=1}^J$ is the LDS dictionary,
and $\Mat{z}=[z_1;z_2;\cdots;z_J]$ is the vector of sparse codes.
We note that by making use of Thm.~\ref{Th:linear_combination}, we can rewrite $l(\Mat{z},\mathbb{D})$ as
\begin{eqnarray}\label{Eq:SC2}
l(\Mat{z},\mathbb{D}) = \Mat{z}^{\mathrm{T}}\Mat{K}_{\Mat{D}}\Mat{z}-2\Mat{z}^{\mathrm{T}}\Mat{k}_{\Mat{X}\Mat{D}}+\lambda\parallel\Mat{z}\parallel_1,
\end{eqnarray}
where $\Mat{K}_{\Mat{D}}(i,j)=\parallel\Mat{D}_i^{\mathrm{T}}\Mat{D}_j\parallel_F^2$ and $\Mat{k}_{\Mat{X}\Mat{D}}(j) =\parallel\Mat{X}^{\mathrm{T}}\Mat{D}_j\parallel_F^2$. This problem is convex as $\Mat{K}_{\Mat{D}}$ is positive semi-definite.

Interestingly, Eq.~\eqref{Eq:SC2} has a very similar form to the general kernel sparse coding presented in~\eqref{Eq:Kernel_SC2}.
The difference is that here we explicitly construct the feature mapping, \ie, $\Pi(\mathcal{S}(n,\infty))$.
The kernel function $k(\Mat{X}_1,\Mat{X}_2)=\|\Mat{X}_1^{\mathrm{T}}\Mat{X}_2\|_F^2$, known as projection kernel~\cite{harandi2014expanding}
is well-known for finite-dimensional Grassmannian. We emphasize that this mapping will enable us to devise an algorithm for dictionary learning in \textsection~\ref{Sec:DL}.

\subsection{Prediction with Labelled Dictionary}\label{Sec:reconstruction_error_approach}
Generally speaking, sparse codes obtained by any of the aforementioned methods can be fed into a generic classifier (resp. regressor) for
classification (resp. regression) purposes. However, if a labeled dictionary (\ie, a dictionary where each atom comes with a label)
is at our disposal, the reconstruction error can be utilized for classification. This is inspired by the Sparse Representation  Classifier (SRC)
introduced in~\cite{wright2009robust}. Below, we customize SRC for our case of interest, \ie, LDSs. Let us denote all the
dictionary atoms belonging to class $c$ by $\{\Mat{D}_k^{(c)}\}_{k=1}^{J_c}$, with $J_c$ showing the total number of atoms in class $c$.
The reconstruction error of a sample $\Mat{X}$ with respect to class $c$ is defined as
\begin{eqnarray}\label{Eq:Reconstruction_error}
e_c(\Mat{X}) &=&  \Big\| \Pi(\Mat{X}) - \sum_{j=1}^{J_c} \Mat{z}_{k}^{(c)}\Pi(\Mat{D}_k^{(c)}) \Big \|_{F}^2,
\end{eqnarray}
where $\Mat{z}_{k}^{(c)}$ is the coefficient associated with atom $\Mat{D}_k^{(c)}$. The label of $\Mat{X}$ is then determined by the class that has the minimum reconstruction error, \ie, $y = \arg\min_{c} e_c(\Mat{X})$.

\section{Dictionary learning with LDSs}
\label{Sec:DL}

Assume a set of LDS models $\{\Mat{X}_i\in\mathcal{S}(n,\infty)\}_{i=1}^{N}$ is given.
The problem of dictionary learning is to
identify a set $\{\Mat{D}_j\in\mathcal{S}(n,\infty)\}_{j=1}^{J}$ to best reconstruct the observations according to the cost defined
in~\eqref{Eq:SC1}.  Formally, this can be written as:
\begin{eqnarray}\label{Eq:DL-OBJ}
\min_{\Mat{Z},\Mat{D}} L(\Mat{Z},\Mat{D})=\sum_{i=1}^N l([\Mat{Z}]_{i}, \Mat{D}),
\end{eqnarray}
with $\Mat{Z}\in \mathbb{R}^{J\times N}$ denoting the coding matrix and $l([\Mat{Z}]_{i}, \Mat{D})$ as in~\eqref{Eq:SC2}.
In general, dictionary learning is an involved problem~\cite{aharon2006k,mairal2009online}.
The case here is of course as a result of non-Euclidean structure and infinite-dimensionality of the LDS space.

In the following, we first make use of the diffemorphic embedding
proposed in \textsection~\ref{Sec:diffemorphic_embedding} to derive a general formulation for the problem of dictionary learning.
As it turns out, the general form of dictionary learning is still complicated to work with so we impose further
structures to simplify the problem. This leads to the notion of two-fold LDSs and thus a neat and tractable optimization problem for obtaining the
LDS dictionary.

\subsection{LDS Dictionary: The General Form}
\label{Sec:DL_general}
A general practice in dictionary learning is to solve the problem alternatively, meaning learning by repeating the following two steps
\textbf{1)} \emph{optimizing the codes when the dictionary is fixed} and \textbf{2)} \emph{updating the dictionary atoms when the codes are given}.
The first is indeed the sparse coding problem in~\eqref{Eq:SC2} which we already know how to solve. As for the second step, we
break the minimization problem into $J$ sub-minimization problems by updating each atom independently. To update
the r-th atom,  $\Mat{D}_r$, rearranging~\eqref{Eq:DL-OBJ} and keeping the terms that are dependent on $\Mat{D}_r$ leads to
the sub-problem $\min\limits_{\Mat{D}_r}\Gamma_r$ with
\begin{equation}\label{Eq:gamma}
\Gamma_r = \sum_{i=1}^{N}z_{r,i}\sum_{\substack{j=1, j\neq r}}^{J}\left(z_{j,i}k(\Mat{D}_r,\Mat{D}_j)-k(\Mat{D}_r,\Mat{X}_i)\right).
\end{equation}

Based on the modeling done in  \textsection~\ref{Sec:O}, each atom $\Mat{D}_r$ is an infinite subspace and can be parameterized by its  transition and measurement matrices, \ie , the tuple $(\Mat{A}_{r}^{(d)},\Mat{C}_{r}^{(d)})$. Our goal is to determine this tuple to identify
$\Mat{D}_r$. Imposing stability and orthonormality constraints on $\Mat{A}_{r}^{(d)}$ and $\Mat{C}_{r}^{(d)}$, respectively, results in
\begin{eqnarray}\label{Eq:subproblem}
\min_{\Mat{A}_{r}^{(d)}, \Mat{C}_{r}^{(d)}}\Gamma_r,~~~~\mathrm{s.t. }~~(\Mat{C}_{r}^{(d)})^{\mathrm{T}}\Mat{C}_{r}^{(d)}=\mathbf{I}_n;~|\mu(\Mat{A}_{r}^{(d)})|<1.
\end{eqnarray}
to determine the r-th atom. Here, $\mu(\Mat{A}_{r}^{(d)})$ denotes the eigenvalue of $\Mat{A}_{r}^{(d)}$ with the largest magnitude.
Our goal in dictionary learning is to find the optimal tuples of the dictionary atoms.

There are mainly two difficulties in solving~\eqref{Eq:subproblem}:
\textbf{1.}
As discussed in \textsection~\ref{Sec:O}, $(\Mat{A}_{r}^{(d)},\Mat{C}_{r}^{(d)})$ do not lie in an Euclidean space.
This is because for any orthogonal $\Mat{P}$,
$(\Mat{P}^{\mathrm{T}}\Mat{A}_{r}^{(d)}\Mat{P},\Mat{C}_{r}^{(d)}\Mat{P})$ results in the same objective $\Gamma_r$ as
that of $(\Mat{A}_{r}^{(d)},\Mat{C}_{r}^{(d)})$. \textbf{2.} As discussed in~\cite{wenbing2016learning}, the stability constraint on $\Mat{A}_{r}^{(d)}$ makes the feasible region non-convex. Below, we show how the aforementioned difficulties can be addressed.

\subsection{Dictionary learning with Two-fold LDSs}
\label{Sec:DL_two-fold}

To facilitate the optimization in~\eqref{Eq:subproblem}, we propose to impose further, yet beneficial structures on the original problem.
In particular, we propose to  \textbf{1.} \emph{encoding data sequences with two-fold LDSs where the transition matrix is decomposed into symmetric and skew-symmetric parts}; and \textbf{2.} \emph{updating symmetric and skew-symmetric dictionaries separately}.
We will show that with these modifications, not only stable transition matrices can be obtained but also the measurement matrices can be updated in closed-form.

\subsubsection{\textbf{Encoding Sequences with Two-fold LDSs}}

The two-fold LDS models the input time series with two independent sub-LDSs
\begin{eqnarray}
\label{Eq:sym-LDS}
\left\{
\begin{aligned}
\Mat{x}_1(t+1) &= \Mat{A}_{sym}\Mat{x}_1(t) + \Mat{B}_{1}\Mat{v}_1(t), \\
\Mat{y}(t)   &= \Mat{C}_{1}\Mat{x}_1(t) + \Mat{w}(t) + \overline{\Mat{y}},
\end{aligned}
\right.
\end{eqnarray}
\begin{eqnarray}
\label{Eq:skew-LDS}
\left\{
\begin{aligned}
\Mat{x}_2(t+1) &= \Mat{A}_{skew}\Mat{x}_2(t) + \Mat{B}_{2}\Mat{v}_2(t), \\
\Mat{y}(t)     &= \Mat{C}_{2}\Mat{x}_2(t) + \Mat{w}(t) + \overline{\Mat{y}},
\end{aligned}
\right.
\end{eqnarray}
where $\Mat{A}_{sym}$ and $\Mat{A}_{skew}$, as the names imply, are symmetric and skew-symmetric matrices, respectively.
Obviously the two-fold LDS encodes a sequence with two decoupled processes, whose transition matrices are constrained to be symmetric and skew-symmetric, respectively.
%
%
To learn the system tuples for the two-fold LDS, we can apply the method introduced in \cite{doretto2003dynamic} by additionally enforcing the symmetric and skew-symmetric constraints on the transition matrices of the symmetric and skew-symmetric LDSs, respectively. However, this will involve heavy computations. In this paper, we simply choose 
$\Mat{C}_1=\Mat{C}_2=\Mat{C}$, $\Mat{A}_{sym} = \frac{1}{2}(\Mat{A}+\Mat{A}^{\mathrm{T}})$ and $\Mat{A}_{skew} = \frac{1}{2}(\Mat{A}-\Mat{A}^{\mathrm{T}})$ where $\Mat{C}$ and $\Mat{A}$ are given by the original LDS defined in Eq.~\eqref{Eq:ARMA}. Our experiments show that such solutions still obtain satisfied results. Furthermore, if all $\Mat{A}$'s singular-values have the magnitude less than 1 so do those of $\Mat{A}_{sym}$ and $\Mat{A}_{skew}$. The details are presented in the appendix.


\begin{table}[!ht]
\centering
\caption{The canonical tuples for data and dictionary systems.}
\label{Tab:canonical-tuples}
\begin{tabular}{|c|c|}
\hline
Denotation                                         & Definition                                                    \\ \hline
$(\Mat{\Lambda}_{1,i}^{(x)}, \Mat{U}_{1,i}^{(x)})$ & The symmetric canonical tuples of data $\Mat{X}_i$            \\ \hline
$(\Mat{\Lambda}_{2,i}^{(x)}, \Mat{U}_{2,i}^{(x)})$ & The skew-symmetric canonical tuples of data $\Mat{X}_i$       \\ \hline
$(\Mat{\Lambda}_{1,j}^{(d)}, \Mat{U}_{1,j}^{(d)})$ & The canonical tuples of symmetric dictionary $\Mat{D}_j$      \\ \hline
$(\Mat{\Lambda}_{2,j}^{(d)}, \Mat{U}_{2,j}^{(d)})$ & The canonical tuples of skew-symmetric dictionary $\Mat{D}_j$ \\ \hline
\end{tabular}
\end{table}

\begin{table}[ht!]
\centering
\caption{Intermediate variables for dictionary learning. $(\Mat{\Lambda}_{j}^{(d)}, \Mat{U}_{j}^{(d)})$ and $(\Mat{\Lambda}_{i}^{(x)}, \Mat{U}_{i}^{(x)})$ denote the canonical tuples of the $j$-th dictionary atoms and the $i$-th data, respectively. They are corresponded to the symmetric parts of two-fold LDSs when updating symmetric dictionary and skew-symmetric parts when updating skew-symmetric dictionary.}
\label{Tab:denotations}
\begin{tabular}{|c|c|}
\hline
Denotation                           & Definition                                                                                                                                                                                                                                                                                           \\ \hline
$\Mat{E}(\lambda,\Mat{\Lambda})$ & \begin{tabular}[c]{@{}c@{}}$\Mat{E}(\lambda,\Mat{\Lambda})=\mathrm{diag}( [\frac{(1-|\lambda|^2)(1-|\lambda_{1}|^2)}{|1-\lambda\lambda_{1}^{*}|^2},\cdots,$ \\ $\frac{(1-|\lambda|^2)(1-|\lambda_{n}|^2)}{|1-\lambda\lambda_{n}^{*}|^2}])$ and $\lambda_{k}=[\Mat{\Lambda}]_k$\end{tabular} \\ \hline
$\Mat{F}_{r,j,k}$                    & $\Mat{F}_{r,j,k}=\Mat{U}_{j}^{(d)}\Mat{E}([\Mat{\Lambda}_{r}^{(d)}]_k,\Mat{\Lambda}_{j}^{(d)})(\Mat{U}_{j}^{(d)})^{*}$                                                                                                                                                                                       \\ \hline
$\Mat{F'}_{r,i,k}$            & $\Mat{F'}_{r,i,k} = \Mat{U}_{i}^{(x)}\Mat{E}([\Mat{\Lambda}_{r}^{(d)}]_k,\Mat{\Lambda}_{i}^{(x)})(\Mat{U}_{i}^{(x)})^{*}$                                                                                                                                                                 \\ \hline
$k(\Mat{D}_r,\Mat{D}_j)$             & $k(\Mat{D}_r,\Mat{D}_j) =\sum_{k=1}^n[\Mat{U}_{r}^{(d)}]_k^{*}\Mat{F}_{r,j,k}[\Mat{U}_{r}^{(d)}]_k$                                                                                                                                                                                                  \\ \hline
$k(\Mat{D}_r,\Mat{X}_i)$             & $k(\Mat{D}_r,\Mat{X}_i) =\sum_{k=1}^n[\Mat{U}_{r}^{(d)}]_k^{*}\Mat{F'}_{r,i,k}[\Mat{U}_{r}^{(d)}]_k$                                                                                                                                                             \\ \hline
$\Mat{S}_{r,k}$                      &  $\Mat{S}_{r,k} = \sum_{i=1}^{N}z_{r,i}(\sum_{j=1,j\neq r}^{J}z_{j,i}\Mat{F}_{r,j,k} -\Mat{F'}_{r,i,k})$ \\ \hline
\end{tabular}
\end{table}

\subsubsection{\textbf{Learning a Two-fold Dictionary}}
%

Assuming that the dictionary atoms are also generated by two-fold LDSs, we define the reconstruction error over a set of data two-fold LDSs described by $\{\Mat{X}_i\}_{i=1}^{N}$ as
\begin{align}\label{Eq:RE2}
&L(\Mat{Z},\Mat{D}_{1,\cdot},\Mat{D}_{2,\cdot}) \coloneqq  \\
&\sum_{i=1}^N\Big\|\Pi(\Mat{X}_{1,i})-\sum_{j=1}^J z_{j,i}^{(1)}\Pi(\Mat{D}_{1,j})\Big\|_F^2 + \|\Pi(\Mat{X}_{2,i}) - \sum_{j=1}^J z_{j,i}^{(2)}\Pi(\Mat{D}_{2,j})\Big\|_F^2. \notag
\end{align}
Here, $\Mat{X}_{1,i}$ and $\Mat{X}_{2,i}$ are respectively the symmetric and skew-symmetric parts of data $\Mat{X}_i$;
$\Mat{D}_{1,j}$ and $\Mat{D}_{2,j}$ are respectively the symmetric and skew-symmetric parts of the dictionary atom $\Mat{D}_j$.
Note that the codes for $\Mat{D}_{1,j}$ and $\Mat{D}_{2,j}$ are not necessary equal to each other so as to improve the modeling flexibility.
For simplicity, we concatenate the codes for the symmetric and skew-symmetric dictionary atoms into one single code matrix as $\Mat{Z}=[\Vec{z}_{\cdot,1}^{(1)};\cdots;\Vec{z}_{\cdot,N}^{(1)};\Vec{z}_{\cdot,1}^{(2)};\cdots;\Vec{z}_{\cdot,N}^{(2)}]$.
As discussed in \textsection~\ref{Sec:DL_general}, Problem~\eqref{Eq:RE2} can be solved by alternatively updating the codes and dictionary atoms.
Given the codes, the dictionary atoms can be learned one by one, leading to the sub-problem given by \eqref{Eq:gamma} and \eqref{Eq:subproblem}. To address the sub-problem, we make use of the following lemma (see the appendix for the proof).

\begin{lemma}\label{canonical-representaion}
For a symmetric or skew-symmetric transition matrix $\Mat{A}$, the system tuple $(\Mat{A},\Mat{C})\in \mathbb{R}^{n \times n}\times\mathbb{R}^{m \times n}$ has the canonical form $(\Mat{\Lambda},\Mat{U})\in \mathbb{R}^{n}\times\mathbb{R}^{m \times n}$, where $\Mat{\Lambda}$ is diagonal and $\Mat{U}$ is unitary, \ie $\Mat{U}^\ast\Mat{U}= \mathbf{I}$. Moreover, both $\Mat{\Lambda}$ and $\Mat{U}$ are real if $\Mat{A}$ is symmetric and complex if $\Mat{A}$ is skew-symmetric. For the skew-symmetric $\Mat{A}$, $(\Mat{\Lambda},\Mat{U})$ is parameterized by a real matrix-pair $(\Mat{\Theta},\Mat{Q})\in \mathbb{R}^{n/2}\times\mathbb{R}^{m \times n}$ where $\Mat{\Theta}$ is diagonal, $\Mat{Q}$ is orthogonal, and $n$ is even\footnote{When $n$ is odd, our developments can be applied verbatim to this case as presented in the appendix.}.
\end{lemma}

For the sake of convenience, we denote the canonical tuples for data and dictionary as in Table~\ref{Tab:canonical-tuples}.
We denote the system tuple of $\Mat{D}_j$ by $(\Mat{\Lambda}_{j}^{(d)}, \Mat{U}_{j}^{(d)})$
if specifying the symmetric or skew-symmetric part is not required.

In the conventional LDS modeling, the exact form of the kernel functions in Eq.~\eqref{Eq:gamma} cannot be obtained due to the implicit calculation of the Lyapunov equation. In contrast, the extra structure imposed on the transition matrices enables us to compute the inner-products between observability matrices and derive the kernel values as required by Eq.~\eqref{Eq:gamma}.
In Table~\ref{Tab:denotations}, we provide the form of $k(\Mat{D}_r,\Mat{D}_j)$ and $k(\Mat{D}_r,\Mat{X}_i)$ along
$\Mat{E}(\lambda,\Mat{\Lambda}_{j})$, $\Mat{F}_{r,j,k}$, $\Mat{F}_{r,i,k}^{(1)}$ and $\Mat{F}_{r,i,k}^{(2)}$
which are required for computing the kernel functions (see the appendix for the details).

With the kernel functions at our disposal, we make use of the following theorem to recast Eq.~\eqref{Eq:subproblem} into a form
that suits our purpose better.
\begin{theorem}\label{Th:DL}
For the two-fold LDS modeling defined in Eq.~\eqref{Eq:sym-LDS} and Eq.~\eqref{Eq:skew-LDS}, the sup-problem in Eq.~\eqref{Eq:subproblem} is equivalent to
\begin{eqnarray}\label{Eq:DL}
\begin{aligned}
 & \min_{\Mat{U}_{r}^{(d)}, \Mat{\Lambda}_{r}^{(d)}}  &  &\sum_{k=1}^{n}[\Mat{U}_{r}^{(d)}]_k^*\Mat{S}_{r,k}[\Mat{U}_{r}^{(d)}]_k\\
 & ~~~\mathrm{s.t. }                       &  & (\Mat{U}_{r}^{(d)})^*\Mat{U}_{r}^{(d)}=\mathbf{I}_n;\\
 &&& -1<[\Mat{\Lambda}_{r}^{(d)}]_k<1,~1\leq k\leq n.
\end{aligned}
\end{eqnarray}
\vskip -0.1in \noindent
Here,  the matrix $\Mat{S}_{r,k}$ (see Table~\ref{Tab:denotations}) is not dependent on the measurement matrix $\Mat{U}_{r}^{(d)}$.
\end{theorem}


\subsubsection*{\textbf{Updating the Symmetric Dictionary}}

For a symmetric dictionary atom $(\Mat{\Lambda}_{1,r}^{(d)},\Mat{U}_{1,r}^{(d)})$, both $\Mat{\Lambda}_{1,r}^{(d)}$ and  $\Mat{U}_{1,r}^{(d)}$ are real matrices. The matrix $\Mat{S}_{r,k}$ in~\eqref{Eq:DL} is also guaranteed to be real and symmetric. We further break the optimization in~\eqref{Eq:DL} into $n$ sub-minimization problems. That is, to find the optimal pair  $([\Mat{\Lambda}_{1,r}^{(d)}]_k, [\Mat{U}_{1,r}^{(d)}]_k)$, we fix all the other pairs $\{[\Mat{\Lambda}_{1,r}^{(d)}]_o, [\Mat{U}_{1,r}^{(d)}]_o)\}_{o=1,o\neq k}^{n}$. As such, we need to optimize the following sub-problem,
\begin{eqnarray}\label{Eq:DL-sym-U}
\begin{aligned}
 &\min_{[\Mat{U}_{1,r}^{(d)}]_k, [\Mat{\Lambda}_{1,r}^{(d)}]_k} &  &[\Mat{U}_{1,r}^{(d)}]_k^{\mathrm{T}}\Mat{S}_{r,k}[\Mat{U}_{1,r}^{(d)}]_k\\
 &~~~~~~\mathrm{s.t. }                                                                          &  &
 (\Mat{U}_{1,r}^{(d)})^{\mathrm{T}}\Mat{U}_{1,r}^{(d)}=\mathbf{I}_n,\\
 & & &
-1<[\Mat{\Lambda}_{1,r}^{(d)}]_k<1.
\end{aligned}
\end{eqnarray}
The optimal $[\Mat{U}_{1,r}^{(d)}]_k$ is obtained by the following theorem.
\begin{theorem}\label{Th:sym-column-U}
Let $[\Mat{U}_{1,r}^{(d)}]_{-k}\in \mathbb{R}^{m\times (n-1)}$ denote the sub-matrix obtained from $\Mat{U}_{1,r}^{(d)}$
by removing the $k$-th column, \ie,
\begin{align*}
[\Mat{U}_{1,r}^{(d)}]_{-k} = \Big([\Mat{U}_{1,r}^{(d)}]_1;\cdots;[\Mat{U}_{1,r}^{(d)}]_{k-1}; [\Mat{U}_{1,r}^{(d)}]_{k+1};\cdots;[\Mat{U}_{1,r}^{(d)}]_n \Big).
\end{align*}
Define $\Mat{W}=[\Mat{w}_{1},\cdots,\Mat{w}_{m-n+1}]\in\mathbb{R}^{m\times (m-n+1)}$ as the orthogonal complement basis of $[\Mat{U}_{1,r}^{(d)}]_{-k}$. If $\Mat{u}\in \mathbb{R}^{(m-n+1)}$ is the eigenvector of  $\Mat{W}^{\mathrm{T}}\Mat{S}_{r,k}\Mat{W}$ corresponding to the smallest eigenvalue, then $\Mat{W}\Mat{u}$ is the optimal solution of $[\Mat{U}_{1,r}^{(d)}]_{-k}$ for Eq. (\ref{Eq:DL-sym-U}).
\end{theorem}
\begin{proof}
See the appendix for the proof of this theorem.
\end{proof}
In terms of the transition term $[\Mat{\Lambda}_{1,r}^{(d)}]_k$, solving~\eqref{Eq:DL-sym-U} is actually to solve a minimization problem with bound constraints. Here, we transform~\eqref{Eq:DL-sym-U} to an unconstrained problem using an auxiliary variable $\rho$ satisfying
\begin{eqnarray}\label{Eq:sym-lambda}
[\Mat{\Lambda}_{1,r}^{(d)}]_k = 2\mathrm{Sig}(a\rho)-1.
\end{eqnarray}
This is indeed similar to the SN stabilization method introduced in~\textsection~\ref{Sec:stable_LDS}. With this trick
$[\Mat{\Lambda}_{1,r}^{(d)}]_k$ is naturally bounded in $(-1,1)$, hence enabling us to apply gradient-descent methods to
update $[\Mat{\Lambda}_{1,r}^{(d)}]_k$. More specifically, let
$\Phi(r,k) = [\Mat{U}_{1,r}^{(d)}]_k^{\mathrm{T}}\Mat{S}_{r,k}[\Mat{U}_{1,r}^{(d)}]_k$. Then,
\begin{align*}
\frac{\partial\Phi(r,k)}{\partial \rho} = 2\frac{\partial\Phi(r,k)}{\partial[\Mat{\Lambda}_{1,r}^{(d)}]_k}\frac{\partial [\Mat{\Lambda}_{1,r}^{(d)}]_k}{\partial \rho}\;.
\end{align*}
In practice,  We first update $\rho$ with its gradient, and then return back to $[\Mat{\Lambda}_{1,r}^{(d)}]_k$ via Eq. (\ref{Eq:sym-lambda}).

\subsubsection*{\textbf{Updating the Skew-symmetric Dictionary}}

For a skew-symmetric atom $(\Mat{\Lambda}_{2,r}^{(d)},\Mat{U}_{2,r}^{(d)})$, both $\Mat{\Lambda}_{2,r}^{(d)}$ and $\Mat{U}_{2,r}^{(d)}$ are complex matrices. As such, Theorem~\ref{Th:sym-column-U} and Eq.\eqref{Eq:sym-lambda} cannot be directly used to
update the dictionary. However, since  $(\Mat{\Lambda}_{2,r}^{(d)},\Mat{U}_{2,r}^{(d)})$ can be parametrized by a real tuple
$(\Mat{\Theta},\Mat{Q})$, a similar approach to that of symmetric dictionary can be utilized for updates. More specifically,
fixing the contribution of all the other pairs, to update the pairs $([\Mat{\Theta}]_k, [\Mat{Q}]_{2k-1})$ and $([\Mat{\Theta}]_k, [\Mat{Q}]_{2k})$, we need to solve
\begin{eqnarray}\label{Eq:DL-skew-U}
\begin{aligned}
 &\min_{\substack{[\Mat{Q}]_{2k-1:2k}\\ [\Mat{\Theta}]_k}} &  & [\Mat{Q}]_{2k-1}^{\mathrm{T}}\Mat{S'}_{r,k}[\Mat{Q}]_{2k-1} + [\Mat{Q}]_{2k}^{\mathrm{T}}\Mat{S'}_{r,k}[\Mat{Q}]_{2k} + \delta_{r,k}\\
 &~~~~~~\mathrm{s.t. }  &  & \Mat{Q}^{\mathrm{T}}\Mat{Q}=\mathbf{I}_n,\\
 & & & -1<[\Mat{\Theta}]_k<1.
\end{aligned}
\end{eqnarray}
Here, $\Mat{S'}_{r,k}=\frac{1}{2}(\Mat{S}_{r,2k-1}+\Mat{S}_{r,2k})$ is a real symmetric matrix. As shown in
the appendix, $\delta_{r,k}$  is small and can be neglected in practice. As such, the optimal $[\Mat{Q}]_{2k-1}$ and $[\Mat{Q}]_{2k}$ are given by the following theorem.
\begin{theorem}\label{Th:skew-column-U}
Let $[\Mat{Q}]_{-2}\in \mathbb{R}^{m\times (n-2)}$ be a sub-matrix of $\Mat{Q}$ obtained by removing the $(2k-1)$-th and $2k$-th columns, \ie,
\begin{align*}
[\Mat{Q}]_{-2}= \Big([\Mat{Q}]_1;\cdots;[\Mat{Q}]_{2k-2};[\Mat{Q}]_{2k+1};\cdots;[\Mat{Q}]_n \Big).
\end{align*}
Define $\Mat{W}=[\Mat{w}_{1},\cdots,\Mat{w}_{m-n+2}]\in\mathbb{R}^{m\times (m-n+2)}$ as the orthogonal complement basis of $[\Mat{Q}]_{-2}$. If $\Mat{u}_1, \Mat{u}_2\in \mathbb{R}^{(m-n+1)}$ are the eigenvectors of  $\Mat{W}^{\mathrm{T}}\Mat{S'}_{r,k}\Mat{W}$ corresponding to the smallest two  eigenvalues, then $\Mat{W}\Mat{u}_1$ and $\Mat{W}\Mat{u}_2$ are the solutions of $[\Mat{Q}]_{2k-1}$ and $[\Mat{Q}]_{2k}$
in Eq.~\eqref{Eq:DL-skew-U}, respectively.
\end{theorem}
\begin{proof}
See the appendix for the proof of this theorem.
\end{proof}
To update $\Mat{\Theta}_r$, we also apply the gradient-based method by first smoothing out its value using
\begin{eqnarray}\label{Eq:skew-lambda}
[\Mat{\Theta}]_k = 2\mathrm{Sig}(a\rho)-1,
\end{eqnarray}
and then passing the gradient of the objective with respect to $\rho$.
Once $([\Mat{\Theta}]_k, [\Mat{Q}]_{2k-1})$ are updated, the dictionary atom $(\Mat{\Lambda}_{2,r}^{(d)},\Mat{U}_{2,r}^{(d)})$ can be  obtained from Lemma~\ref{canonical-representaion}.
All the aforementioned details are summarized in Algorithm \ref{Alg:DL}.

\begin{algorithm}[tb]
   \caption{Dictionary learning with LDSs}
   \label{Alg:DL}
\begin{algorithmic}
   \STATE {\bfseries Input:} $\Mat{X}$
   \STATE Extract the system parameters of the data sequences by the two-fold LDS:
   $\Mat{C}_1 = \Mat{C}_2=\Mat{C}, \Mat{A}_{sym} = \frac{1}{2}(\Mat{A}+\Mat{A}^{\mathrm{T}}), \Mat{A}_{skew} = \frac{1}{2}(\Mat{A}-\Mat{A}^{\mathrm{T}})$;
   \STATE Initialize the symmetric and skew-symmetric dictionary atoms by random;
   \FOR{$t=1$ {\bfseries to} $MaxNumIters$}
   \STATE Learn the sparse codes $\Mat{Z}$ by solving Eq. (\ref{Eq:SC2}), where the kernels are computed as shown in Table~\ref{Tab:denotations};
   \FOR{$r=1$ {\bfseries to} $J$}
   \STATE \% update the symmetric dictionary
   \FOR{$k=1$ {\bfseries to} $n$}
   \STATE Compute $\Mat{S}_{r,k}$ as defined in Eq. (\ref{Eq:DL});
   \STATE Update $[\Mat{U}_{1,r}^{(d)}]_k$ according to Theorem \ref{Th:sym-column-U};
   \STATE Update $[\Mat{\Lambda}_{1,r}^{(d)}]_k$ according to Eq. (\ref{Eq:sym-lambda});
   \ENDFOR
   \STATE \% update the skew-symmetric dictionary
   \FOR{$k=1$ {\bfseries to} $\frac{n}{2}$}
   \STATE Compute $\Mat{S'}_{r,k}$ as defined in Eq. (\ref{Eq:DL-skew-U});
   \STATE Update $[\Mat{Q}]_{2k-1:2k}$ according to Theorem \ref{Th:skew-column-U};
   \STATE Update $[\Mat{\Theta}]_k$ according to Eq. (\ref{Eq:skew-lambda});
   \STATE Compute $[\Mat{U}_{2,r}^{(d)}]_{2k-1:2k}$ via Lemma~\ref{canonical-representaion};
   \STATE Compute $[\Mat{\Lambda}_{2,r}^{(d)}]_{2k-1:2k}$via Lemma~\ref{canonical-representaion};
   \ENDFOR
   \ENDFOR
   \ENDFOR
\end{algorithmic}
\end{algorithm}

\section{Models Considering State Covariances}
\label{Sec:state_covariance}

In our preliminary study~\cite{wenbing2016sparse}, we incorporate the state covariance matrix $\Mat{B}$
into kernel functions as the symmetric structure on the transition matrices might be restrictive in certain cases.
Generally speaking and compared to conventional LDSs, the two-fold LDS can model a time-series better. However, we observe that adding the state covariance matrix $\Mat{B}$ can boost the performances further (see Fig.~\ref{Fig:Cambridge_SymSkew_vs_SymGrass} for empirical evaluations).

From Eq.~\eqref{Eq:ARMA}, the conditional probability of $\Mat{y}_{t+1}$ given $\Mat{x}(t)$ is $p(\Mat{y}(t+1)\mid \Mat{x}(t))=\mathcal{N}(\Mat{y}(t+1);\Mat{CA}\Mat{x}(t)+\Mat{\bar{y}},\Mat{CB}\Mat{B}^{\mathrm{T}}\Mat{C}^{\mathrm{T}}+\Mat{R})$.
Since, the covariance of the state dynamic is our main concern, we assume $\Mat{R} \simeq \Mat{0}$.
Empirically, we observe that stable performances can be attained if $\Mat{B}$ is orthogonalized. This can be done by first factorizing
$\Mat{B}$ using SVD as $\Mat{B}=\Mat{U}_B\Mat{S}_B\Mat{V}_B^{\mathrm{T}}$, followed by orthogonalizing in the form $\Mat{B}=\Mat{B}\Mat{U}_B\Mat{S}_B^{-1/2}$.
We define a hybrid kernel function by making use of the covariance term as
\begin{eqnarray}\label{Eq:hybrid-kernel}
k(\Mat{X}_1,\Mat{X}_2)=\beta k_{m}(\Mat{X}_1,\Mat{X}_2) + (1-\beta)k_{cov}(\Mat{X}_1,\Mat{X}_2).
\end{eqnarray}
Here, $k_{m}(\Mat{X}_1,\Mat{X}_2)$ is the kernel for the extended observability subspaces defined in Eq. (\ref{Eq:SC2}),
$k_{cov}(\Mat{X}_1,\Mat{X}_2)=\| \Mat{H}_1^{\mathrm{T}}\Mat{H}_2\|_F^2$ is the kernel value between covariances with
$\Mat{H}=\Mat{CB}^\prime \in\mathcal{R}^{m\times n_v}$ being an orthogonal matrix
, and $\beta$ is a weight to trade-off between $k_{m}(\Mat{X}_1,\Mat{X}_2)$ and $k_{cov}(\Mat{X}_1,\Mat{X}_2)$.


Replacing the kernel function in Eq.~\eqref{Eq:gamma} with the hybrid kernel in Eq.~\eqref{Eq:hybrid-kernel} does not change
the algorithm  described in Alg.~\ref{Alg:DL} dramatically. The only additional calculation is the update of the covariance term $\Mat{H}_r^{(d)}$ for the $r$-th atom. We can derive that, when the covariances of other atoms are given. That is to obtain $\Mat{H}_r^{(d)}$, we can  minimize $ \mathrm{Tr}((\Mat{H}_r^{(d)})^{\mathrm{T}}\Mat{S}_{cov,r}\Mat{H}_r^{(d)})$, where $\Mat{S}_{cov,r}=\sum_{i=1}^Nz_{r,i}\sum_{j=1,j\neq r}^{J}(z_{j,i}\Mat{H}_j^{(d)}(\Mat{H}_j^{(d)})^{\mathrm{T}}-\Mat{H}_i^{(x)}(\Mat{H}_i^{(x)})^{\mathrm{T}})$. The optimal $\Mat{H}_r^{(d)}$ is obtained as the eigenvectors of $\Mat{S}_{cov,r}$ corresponding to the smallest $n_v$ eigenvalues.

\section{Computational Complexity}
\label{Sec:computational_complexity} 

For sparse coding (Eq. (\ref{Eq:SC2})), computing the kernel-values is required. In doing so, we need to
\textbf{I)} perform the SVD decomposition,
\textbf{II)} solve the Lyapunov equation and
\textbf{III)} calculate the matrix multiplication, with the complexity of $O(n^3)$, $O(n^3)$ and $O(mn^2)$, respectively.
Since usually $n\ll m$, all these computations cost $O(mn^2)$. Thus computing $\Mat{K}_{\Mat{D}}$ and $\Mat{k}_{\Mat{DX}}$
costs $O((NJ+J^2)mn^2)$.

As shown in Algorithm \ref{Alg:DL}, for each dictionary atom, we need to calculate $\Mat{S}_{r,k}$ and $\Mat{U}_{1,r}^{(d)}$ to update the symmetric dictionary and $\Mat{S}^\prime_{r,k}$ and $\Mat{U}_{2,r}^{(d)}$ to update the skew-symmetric dictionary. Computing $\Mat{S}_{r,k}$ and $\Mat{S'}_{r,k}$ requires $O( 2J(N+nm^2)+\gamma nm^2)$ flops, where $\gamma$ denotes the number of non-zero elements in the $r$-th row of $\Mat{Z}$. We apply the Grassmannian-based Conjugate Gradient Method~\cite{edelman1998geometry} to find the smallest eigenvector of $\Mat{W}^{\mathrm{T}}\Mat{S}_{r,k}\Mat{W}$, which has a computational cost of $O(m^2)$. This operation needs to be repeated $n$ times until all the columns of $\Mat{U}_{1,r}^{(d)}$ are updated. Thus updating $\Mat{U}_{1,r}^{(d)}$ costs $O(nm^2)$ in total. Similarly, updating $\Mat{U}_{2,r}^{(d)}$ costs $O(nm^2)$. Computing the terms associated to a transition matrix, \ie, $\Mat{\Lambda}_{1,r}^{(d)}$ and $\Mat{\Lambda}_{2,r}^{(d)}$ and the covariance terms, \ie, $\Mat{H}_r^{(d)}$ are much cheaper than those of the measurement terms. Hence, updating a dictionary atom costs $O( 2J(N+nm^2)+\gamma nm^2)$ for one iteration.

Compared to the finite-approximation method proposed in~\cite{harandi2014extrinsic}, our sparse coding and dictionary learning algorithms
can scale up better  when the $L$-order observability is employed to representat LDSs. To be precise, the complexity of the finite method
is $O(L(NJ+J^2)mn^2)$ for sparse coding and $O( J(N+nL^2m^2)+\gamma nL^2m^2)$ for updating one dictionary atom, respectively. Roughly, our methods are $L$ times faster than the finite-approximation method for sparse coding; and $L^2/2$ times faster for dictionary learning.

\section{Empirical Evaluations}
\label{sec:experiments}

In this section, we assess and contrast the proposed coding and dictionary learning methods against state-of-the-art techniques on two groups of experiments. First, we study the performance of the sparse coding algorithms on various tasks including hand gesture recognition, dynamical scene classification, dynamic texture categorization and tactile recognition.
Later, we turn our attention to dictionary learning and evaluate the effectiveness of the proposed learning method.
Hereafter, we refer to
\textbf{1.} the kernel-based sparse coding on LDSs with Martin kernel (\textsection~\ref{Sec:Kernel_SC}) by \emph{LDS-SC-Martin},
\textbf{2.} sparse coding based on the diffeomorphic-embedding (\textsection~\ref{Sec:diffemorphic_embedding}) by
\emph{LDS-SC-Grass},
\textbf{3.} the dictionary learning on LDSs (\textsection~\ref{Sec:DL}) by \emph{LDS-DL} and
\textbf{4.} the dictionary learning equipped with the state covariance (\textsection~\ref{Sec:state_covariance}) by \emph{covLDS-DL}.

The baselines are \textbf{1.} the basic LDS model~\cite{saisan2001dynamic,chan2007classifying} with the Martin distance denoted by
\emph{LDS-Martin} and \emph{LDS-SVM}, where the Nearest-Neighbor (NN) method and SVM are utilized as the classifier, respectively,
\textbf{2.} sparse coding and dictionary learning on finite Grassmannian~\cite{harandi2014extrinsic} denoted by
\emph{gLDS-SC} and \emph{gLDS-DL}, respectively and
\textbf{3.} our preliminary work on dictionary learning with symmetric transition matrices~\cite{wenbing2016sparse} denoted by \emph{LDSST-DL}.

In all our experiments, $n_v = 2$ (see Eq.~\eqref{Eq:ARMA}); the scaling parameter of the Gaussian kernel, \ie, $\sigma^2$  in
Eq.~\eqref{Eq:Gaussian_Kernel} is set to 200 and the sparsity penalty factor $\lambda = 0.1$ (see Eq.~\eqref{Eq:SC2}). Other parameters,
\eg, $n$ in Eq.~\eqref{Eq:ARMA} are application-dependent and their values are reported accordingly later. All experiments are carried out with Matlab 8.1.0.604 (R2013a) on an Intel Core i5, 2.20-GHz CPU with 8-GB RAM.

\begin{table}[!ht]
\centering
\caption{The specification of the benchmark datasets.}
\label{Tab:datasets}
\tabcolsep 3pt 
\begin{tabular}{|c|c|c|c|c|}
\hline
Datasets         & \#Sequences & Spatial size          & \#Frames per Seq. & \#Classes \\ \hline
\emph{Cambridge} & 900         & $320\times240$        & 37-119   & 9         \\ \hline
\emph{UCSD }     & 254         & $320\times240$        & 42-52    & 3         \\ \hline
\emph{UCLA}      & 200         & $160\times110$        & 75       & 50         \\ \hline
\emph{DynTex++}  & 3600        & $50\times50$          & 50       & 36        \\ \hline
\emph{SD }       & 100         & $27\times18$          & 325-526  & 10        \\ \hline
\emph{SPr}       & 97          & $8\times16$           & 503-549  & 10        \\ \hline
\emph{BDH }      & 100         & $8\times9$            & 203-486  & 2         \\ \hline
\end{tabular}
\end{table}

\begin{figure*}[!t]
\begin{center}
\subfigure{
\includegraphics[width=15cm]{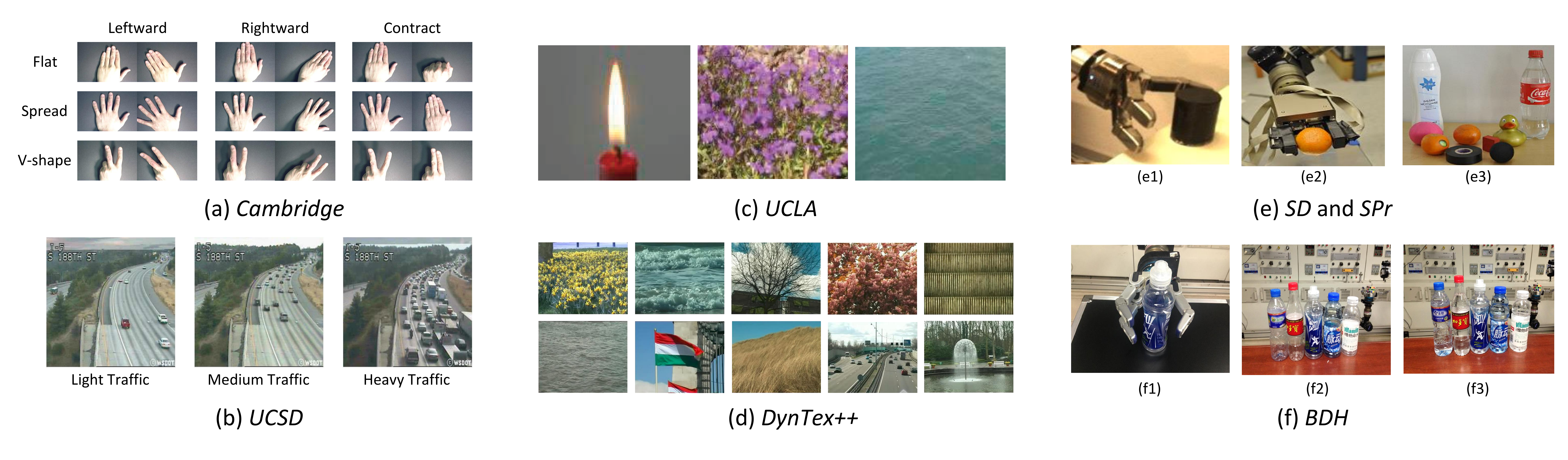}
}
\caption{Samples of the benchmark datasets. (a) \emph{Cambridge}: the image sequences are performed by  $3$ primitive hand shapes with $3$ primitive motions; (b) \emph{UCSD}: representative examples of the three classes; (c) \emph{UCLA}: candle, flower and Sea; (d) \emph{DynTex++}: flowers, sea, naked trees, foliage, escalator, calm water, flags, grass, traffic, and fountains; (e) \emph{SD} and  \emph{SPr}: (e1) the 3-finger Schunk Dextrous hand, (e2) the 2-finger Schunk Parallel hand, (e3) objects in \emph{SD} and  \emph{SPr} including  rubber ball, balsam bottle, rubber duck, empty bottle, full bottle, bad orange, rresh orange, joggling ball, tape, and wood block; (f) \emph{BDH}: (f1) the BH8-280 Hand, (f2) bottles without water, (f3) bottles with water.}
\label{Fig:SamplesOfData}
\end{center}
\end{figure*}

\subsection{Benchmark Datasets}
We consider two types of benchmark datasets, namely \textbf{vision datasets} and \textbf{tactile datasets} in our experiments.
The specification of each dataset is provided in Table~\ref{Tab:datasets}.

\subsection*{Vision datasets}
As for the vision tests, we make use of four widely used datasets, namely, \emph{Cambridge}, \emph{UCSD}, \emph{UCLA}, and \emph{DynTex++}
(see Fig.~\ref{Fig:SamplesOfData} for examples).

\textbf{\emph{Cambridge}.}
The \emph{Cambridge} hand gesture dataset~\cite{kim2009canonical} consists of $900$ images sequences of $9$ gesture classes generated by $3$ primitive hand shapes and $3$ primitive motions. Each class contains $100$ image sequences performed by $2$ subjects, with $10$ arbitrary camera motions and under $5$ illumination conditions. Sample images are demonstrated in Fig.~\ref{Fig:SamplesOfData} (a).
Similar to \cite{Harandi_2015_CVPR}, we resize all images to $20 \times 20$, and use the first $80$ images of each class for testing while the remaining images are used for training purposes.

\textbf{\emph{UCSD}.}
The experiment of scene analysis is performed using the \emph{UCSD} traffic dataset~\cite{chan2005probabilistic}
(see Fig.~\ref{Fig:SamplesOfData} (b) for sample images).
{UCSD} dataset consists of 254 video sequences of highway traffic with a variety of traffic patterns in various weather conditions. Each video is recorded with a resolution of $320 \times 240$ pixels for a duration between $42$ and $52$ frames.
We use the cropped version of the dataset where each video is cropped and resized to $48 \times 48$. The dataset is labeled into three classes with respect to the severity of traffic congestion in each video. The total number of sequences with heavy, medium and light traffic is 44, 45 and 165, respectively. We use the four splits suggested in~\cite{chan2005probabilistic} in our tests.

\textbf{\emph{UCLA}.}
The \emph{UCLA} dataset~\cite{saisan2001dynamic} contains 50 categories of dynamic textures
(see Fig.~\ref{Fig:SamplesOfData} (c) for sample images). Each category consists of 4 gray-scale videos captured from different viewpoints. Each video has 75 frames and cropped to $48\times 48$ by keeping the associated motion. Four random splits of the dataset, as provided  in~\cite{saisan2001dynamic} are used in our experiments.

\textbf{\emph{DynTex++}.}
Dynamic textures are video sequences of complex scenes that exhibit certain stationary properties in the time domain,  such as water on the surface of a lake, a flag fluttering in the wind, swarms of birds, humans in crowds, etc. The constant change poses a challenge for applying traditional vision algorithms to these videos.  The dataset \emph{DynTex++}~\cite{ghanem2010maximum}
(see Fig.~\ref{Fig:SamplesOfData} (d) for sample images) is a variant of the original
\emph{DynTex}~\cite{dyntex} dataset. It contains $3600$ videos of $36$ classes ($100$ videos of size $50\times50\times50$ per class). In this paper, we apply the same test protocol as \cite{ghanem2010maximum}, namely, half of the videos are applied as the training set and the other half as the testing set over 10 trials. We utilize the histogram of LBP from Three Orthogonal Planes (LBP-TOP)~\cite{zhao2007dynamic} as the input feature by splitting each video into sub-videos of length $8$, with a 6-frame overlap.

\subsection*{Tactile datasets}
Recognizing objects that a robot grasps via the tactile series
is an active research area in robotics~\cite{madry2014st}. The tactile series obtained from the force sensors can be used to determine  properties of an object (\eg, shape or softness). In our experiments, the recognition tasks are evaluated on  three datasets:
namely \emph{SD}~\cite{jingwei15tactile},  \emph{SPr}~\cite{madry2014st} and  \emph{BDH}~\cite{madry2014st}.  The \emph{SD} dataset contains $100$ tactile series of $10$ household objects grasped by a 3-finger Schunk Dextrous (SD) hand. The \emph{SPr} dataset composed of $97$ sequences  with the same object classes as \emph{SD}, but is captured with a 2-finger Schunk Parallel (SPr) hand (see Fig.~\ref{Fig:SamplesOfData} (e) for sample images). \emph{BDH} consists of $100$ tactile sequences generated by controlling the BH8-280 Hand to grasp $5$ different bottles with water or without water (see Fig.~\ref{Fig:SamplesOfData} (f) for sample images).  The task is to predict whether the bottle is empty or is filled with water. All datasets are split randomly into the training and testing sets with a ratio of $9:1$ over $10$ trials~\cite{madry2014st, jingwei15tactile}.

\begin{figure*}[!ht]
\begin{center}
\subfigure{
\includegraphics[width=3.4cm,height=2.8cm]{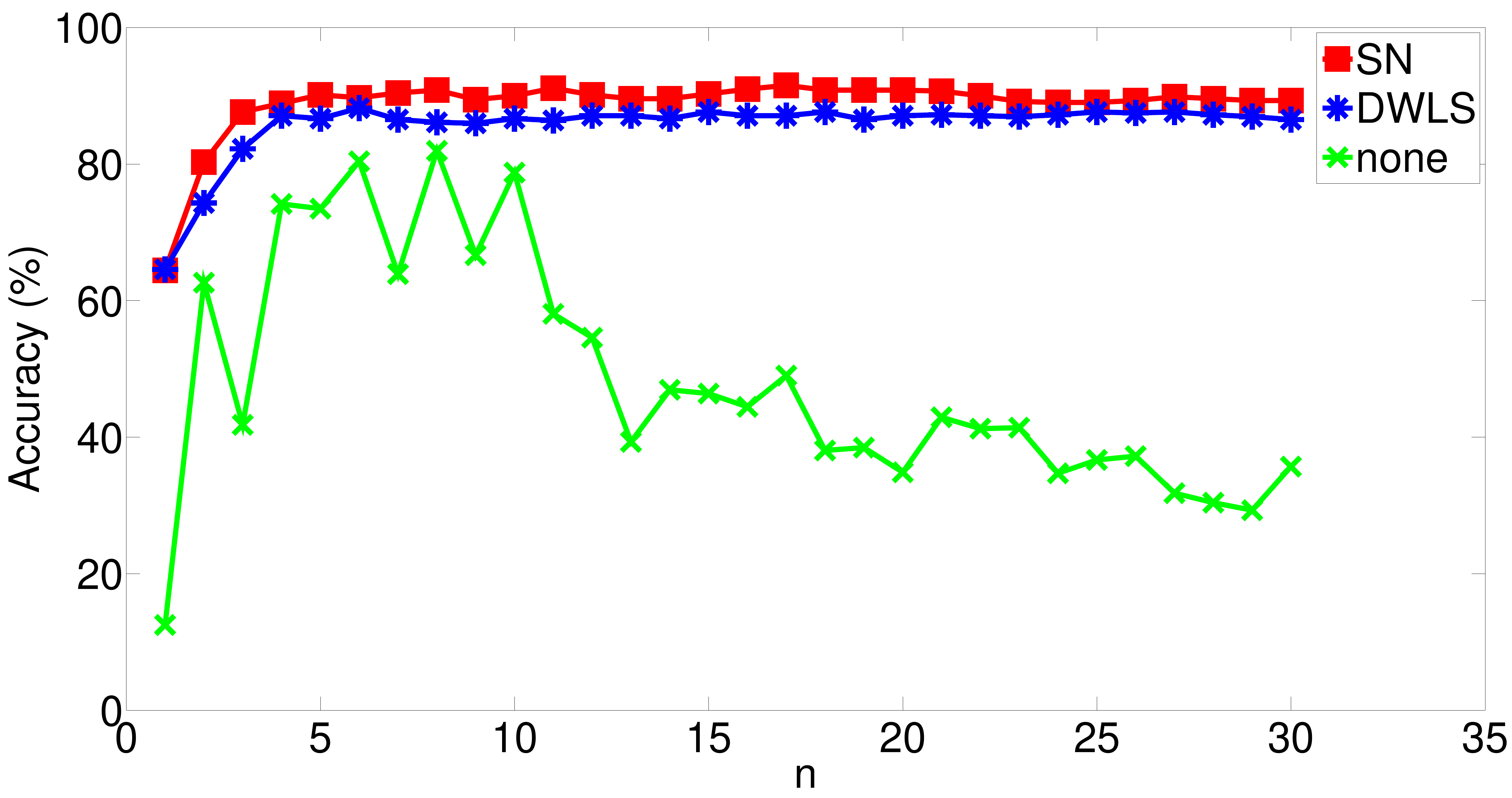}
}
\subfigure{
\includegraphics[width=3.4cm,height=2.8cm]{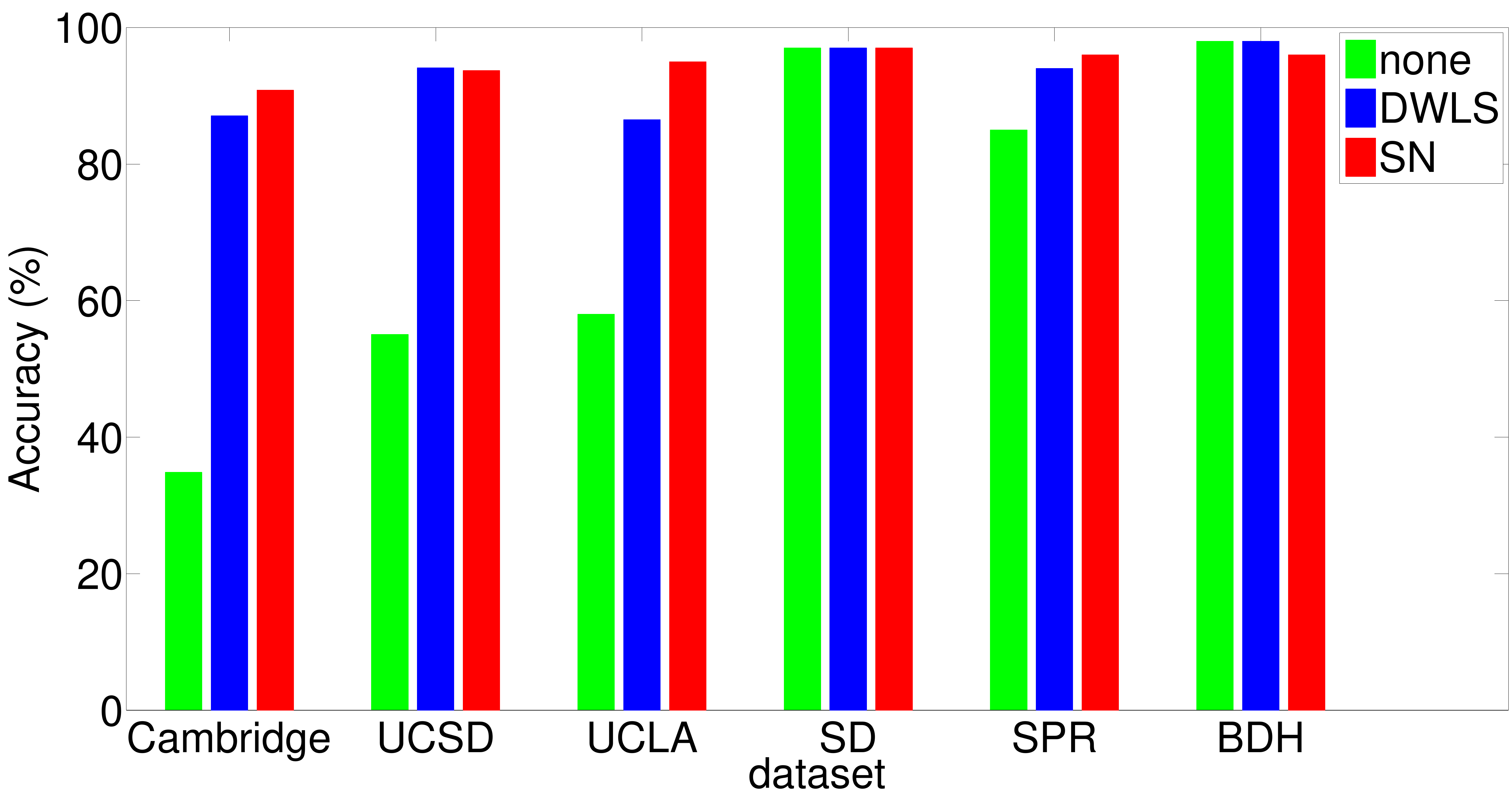}
}
\subfigure{
\includegraphics[width=3.4cm,height=2.8cm]{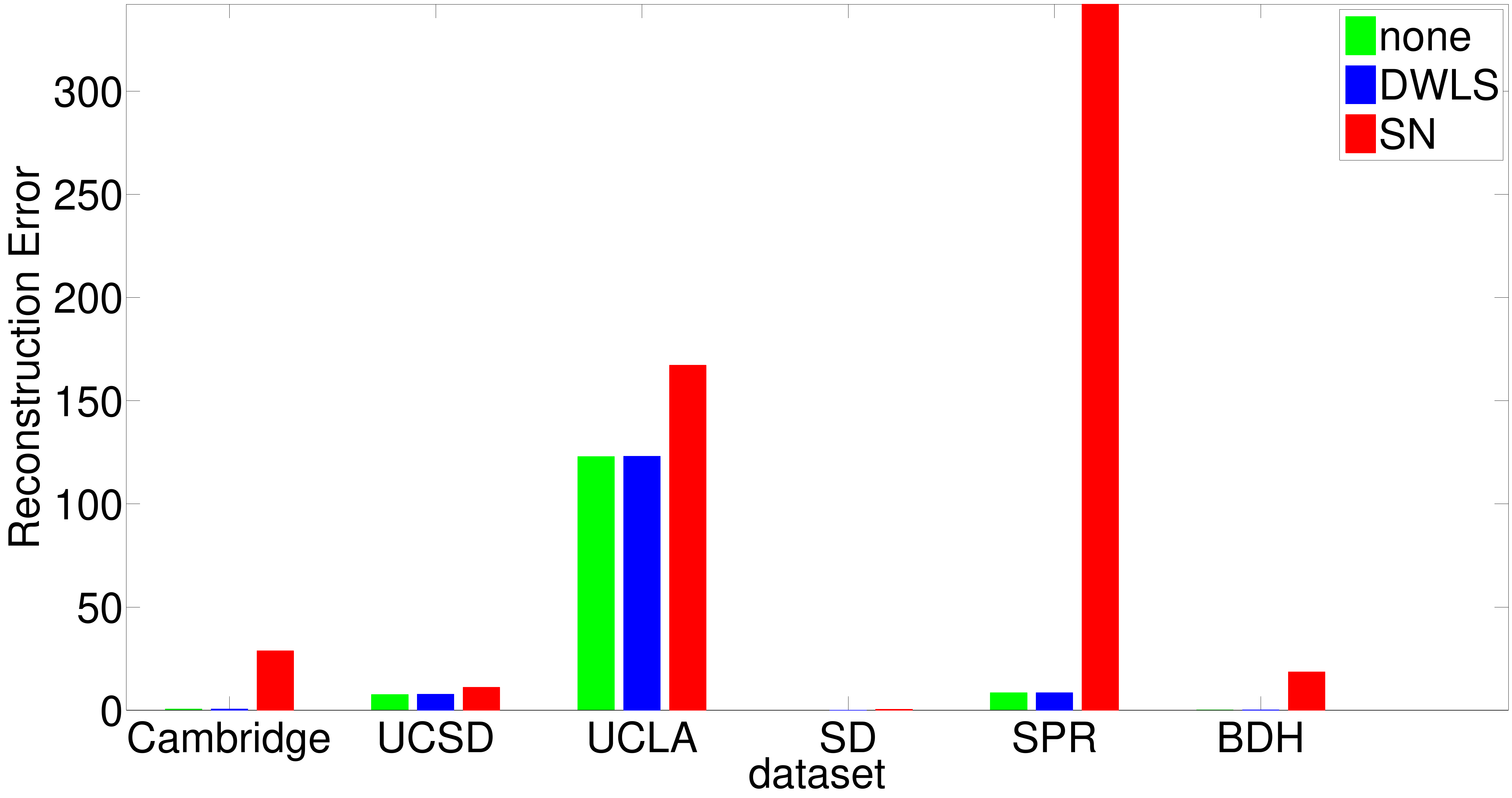}
}
\subfigure{
\includegraphics[width=3.4cm,height=2.8cm]{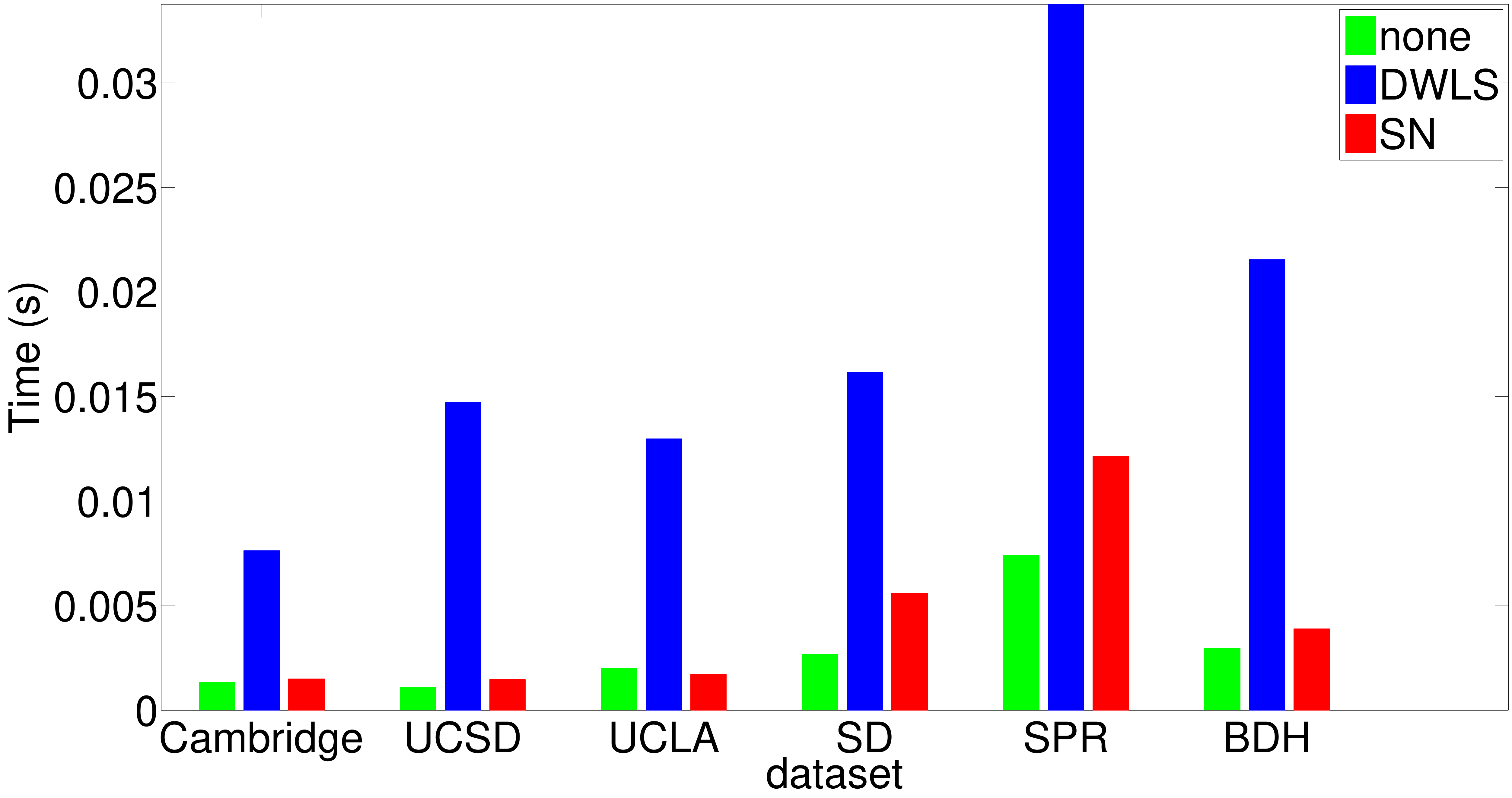}
}
\caption{Comparisons between stable and non-stable LDSs. The left panel shows the average classification accuracies for non-stable LDSs (denoted by none) and stabilized LDSs by SN and DWLS methods on the \emph{Cambridge} dataset. The horizontal axis represents the state
dimensionality $n$ here. The second panel from left shows the average classification accuracies for
the \emph{Cambridge}, \emph{UCSD}, \emph{UCLA}, \emph{SD}, \emph{SPR}, and \emph{BDH} datasets with $n=20$.
The third and fourth diagrams show the average reconstruction errors and average learning times for the studied datasets.
The parameter $a$ in SN is set to $4$ for the \emph{UCSD}, \emph{UCLA}, \emph{SD}, \emph{SPR} and \emph{BDH} datasets and $2.5$ for the \emph{Cambridge} dataset.}

\label{Fig:stability}
\end{center}
\end{figure*}

\begin{figure*}[!ht]
\begin{center}
\subfigure[]{
\includegraphics[width=4.8cm,height=3.6cm]{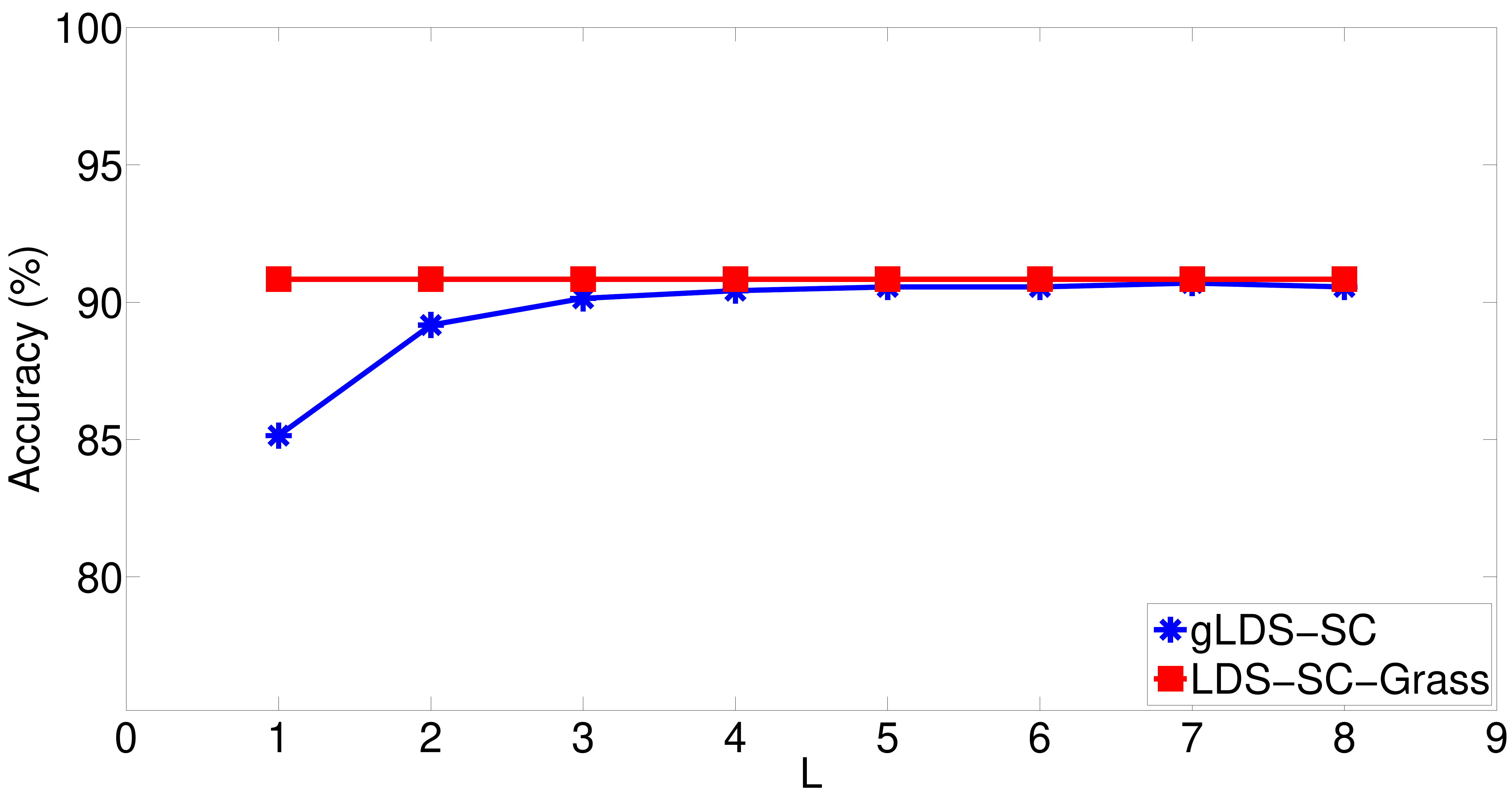}
}
\subfigure[]{
\includegraphics[width=4.8cm,height=3.6cm]{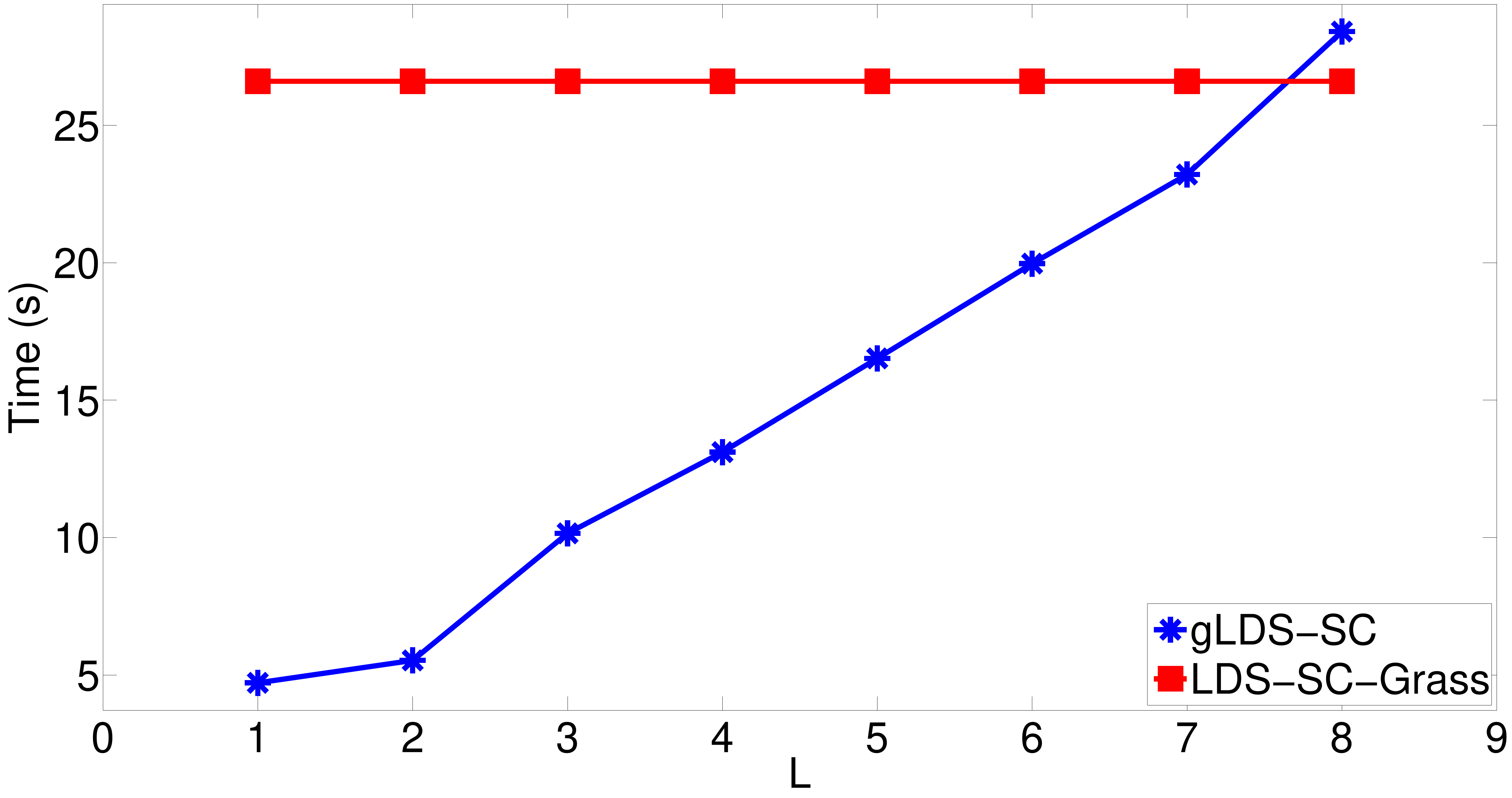}
}
\subfigure[]{
\includegraphics[width=4.8cm,height=3.6cm]{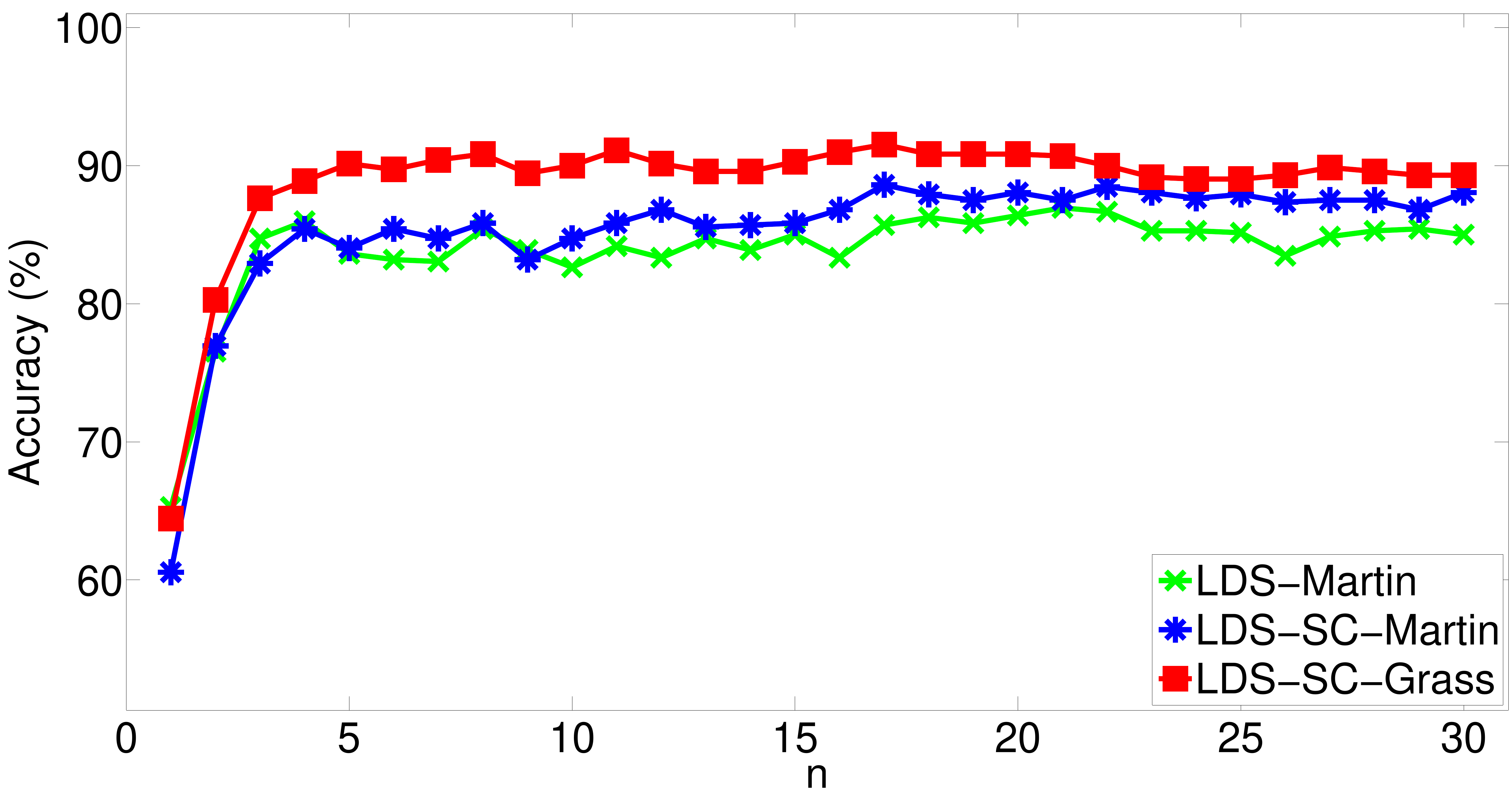}
}
\caption{Performance of the proposed sparse coding methods on \emph{Cambridge}. (a) The classification accuracies of LDS-SC-Grass and gLDS-SC versus observability order $L$ (here, $n=20$). (b) Training time of
LDS-SC-Grass and gLDS-SC versus the observability order $L$. (c) Performances of LDS-Martin, LDS-SC-Martin and LDS-SC-Grass versus state dimensionality $n$.}
\label{Fig:Cambridge}
\end{center}
\end{figure*}

\begin{figure*}[!ht]
\begin{center}
\subfigure[]{
\includegraphics[width=4.8cm,height=3.6cm]{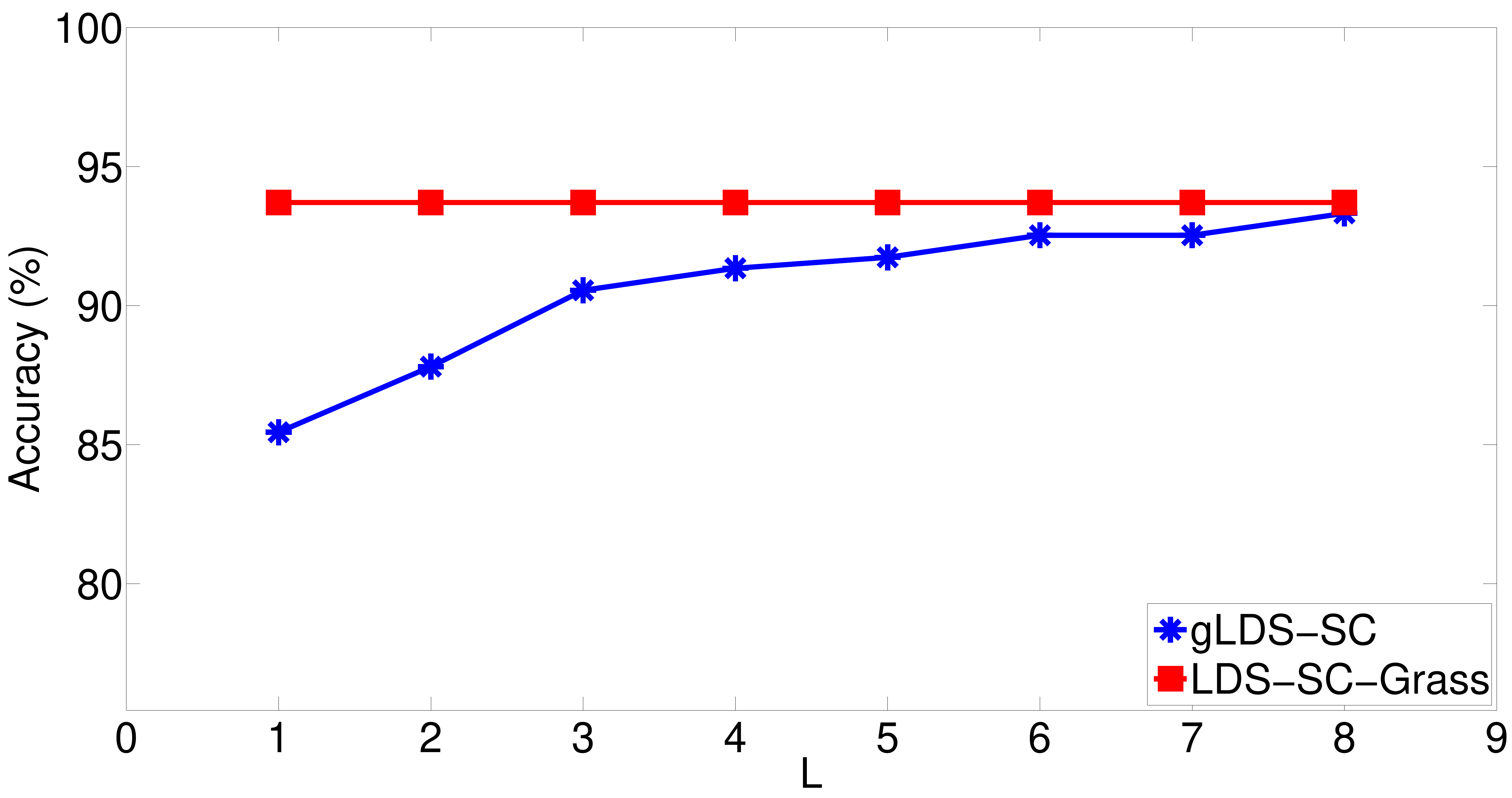}
}
\subfigure[]{
\includegraphics[width=4.8cm,height=3.6cm]{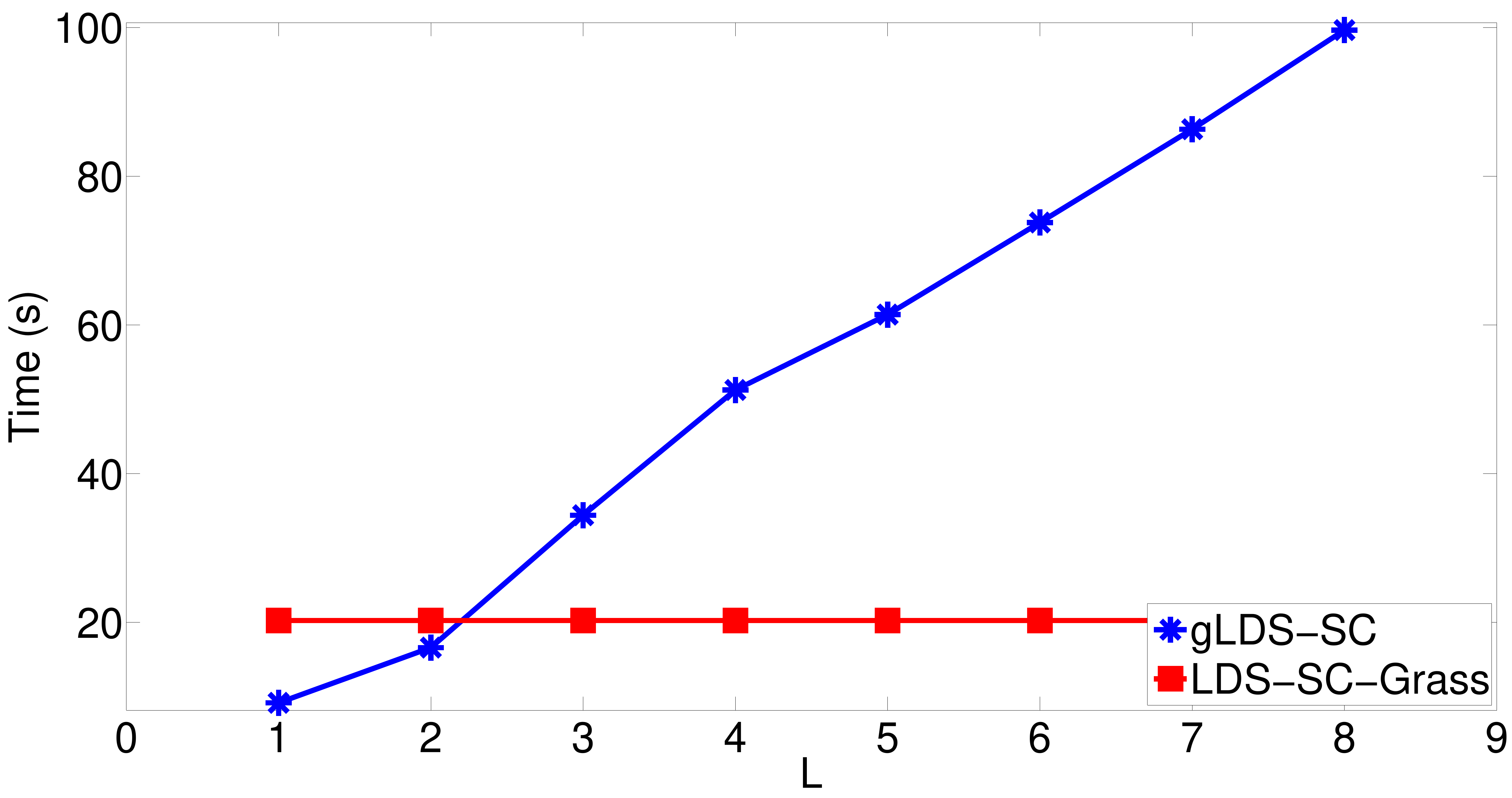}
}
\subfigure[]{
\includegraphics[width=4.8cm,height=3.6cm]{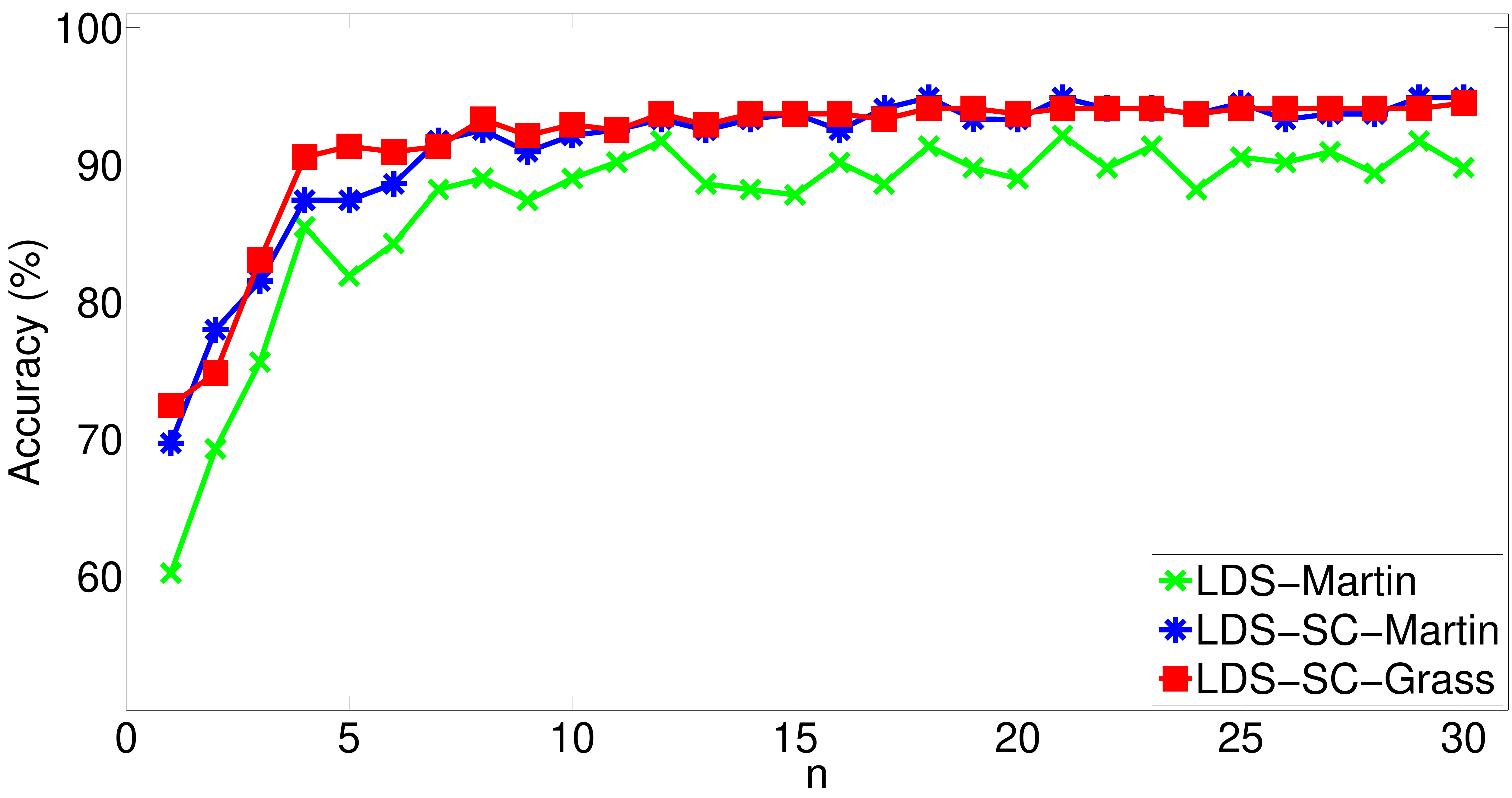}
}
\caption{Performance of the proposed sparse coding methods on \emph{UCSD}. (a) The classification accuracies of LDS-SC-Grass and gLDS-SC versus observability order $L$ (here, $n=20$). (b) Training time of
LDS-SC-Grass and gLDS-SC versus the observability order $L$. (c) Performances of LDS-Martin, LDS-SC-Martin and LDS-SC-Grass versus state dimensionality $n$.}
\label{Fig:UCSD}
\end{center}
\end{figure*}
\begin{figure*}[!ht]
\begin{center}
\subfigure[]{
\includegraphics[width=4.8cm,height=3.6cm]{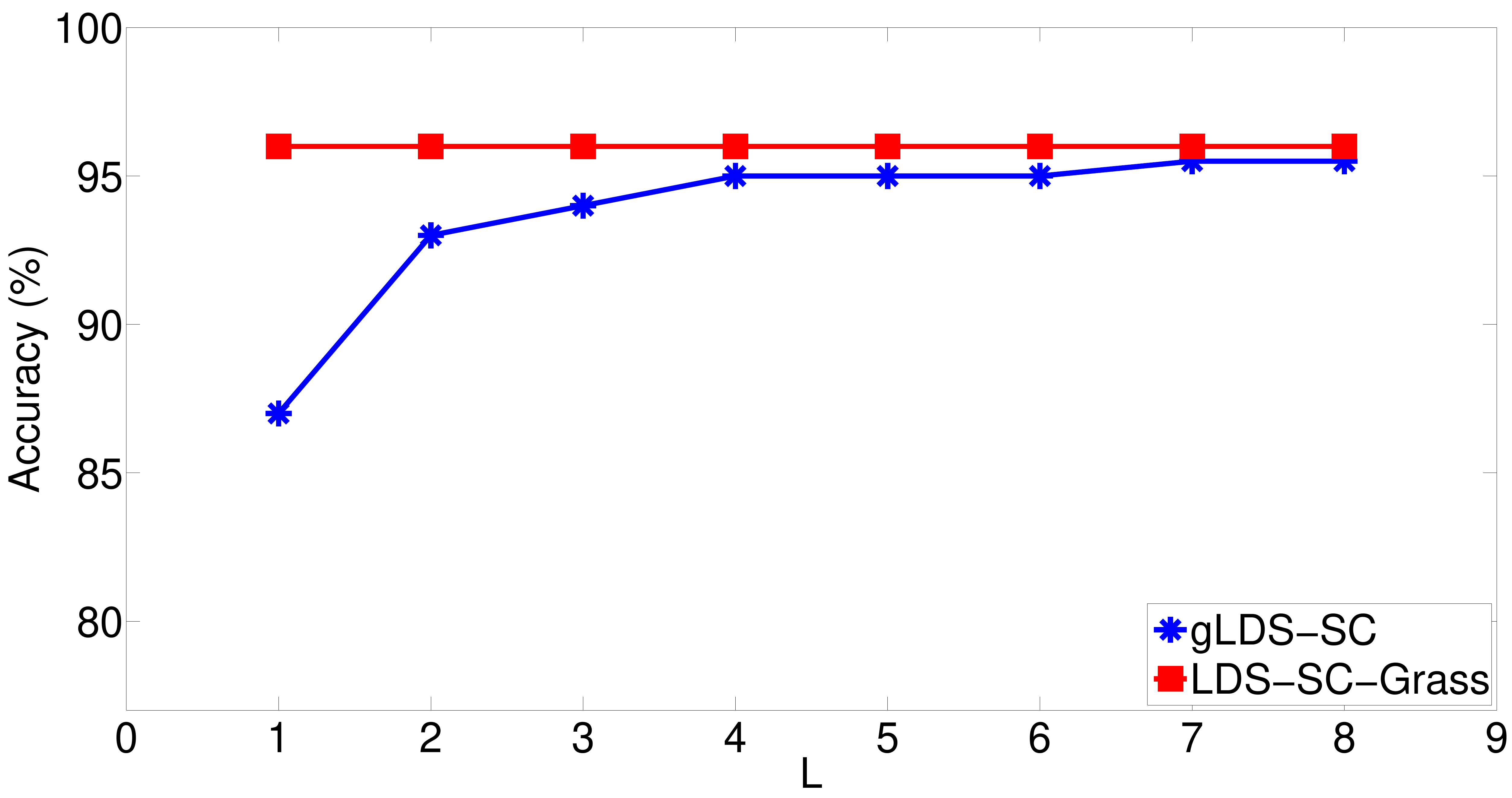}
}
\subfigure[]{
\includegraphics[width=4.8cm,height=3.6cm]{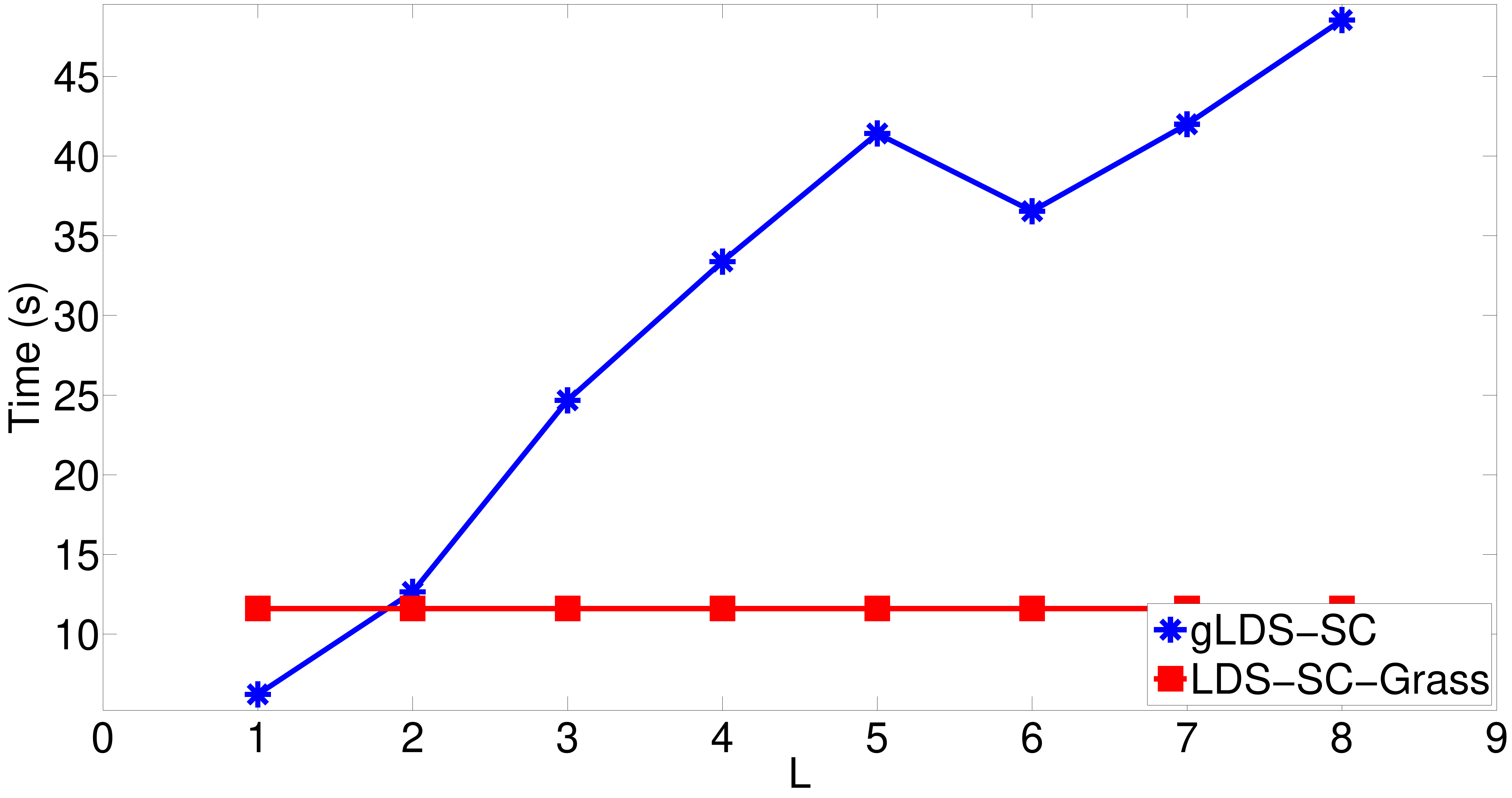}
}
\subfigure[]{
\includegraphics[width=4.8cm,height=3.6cm]{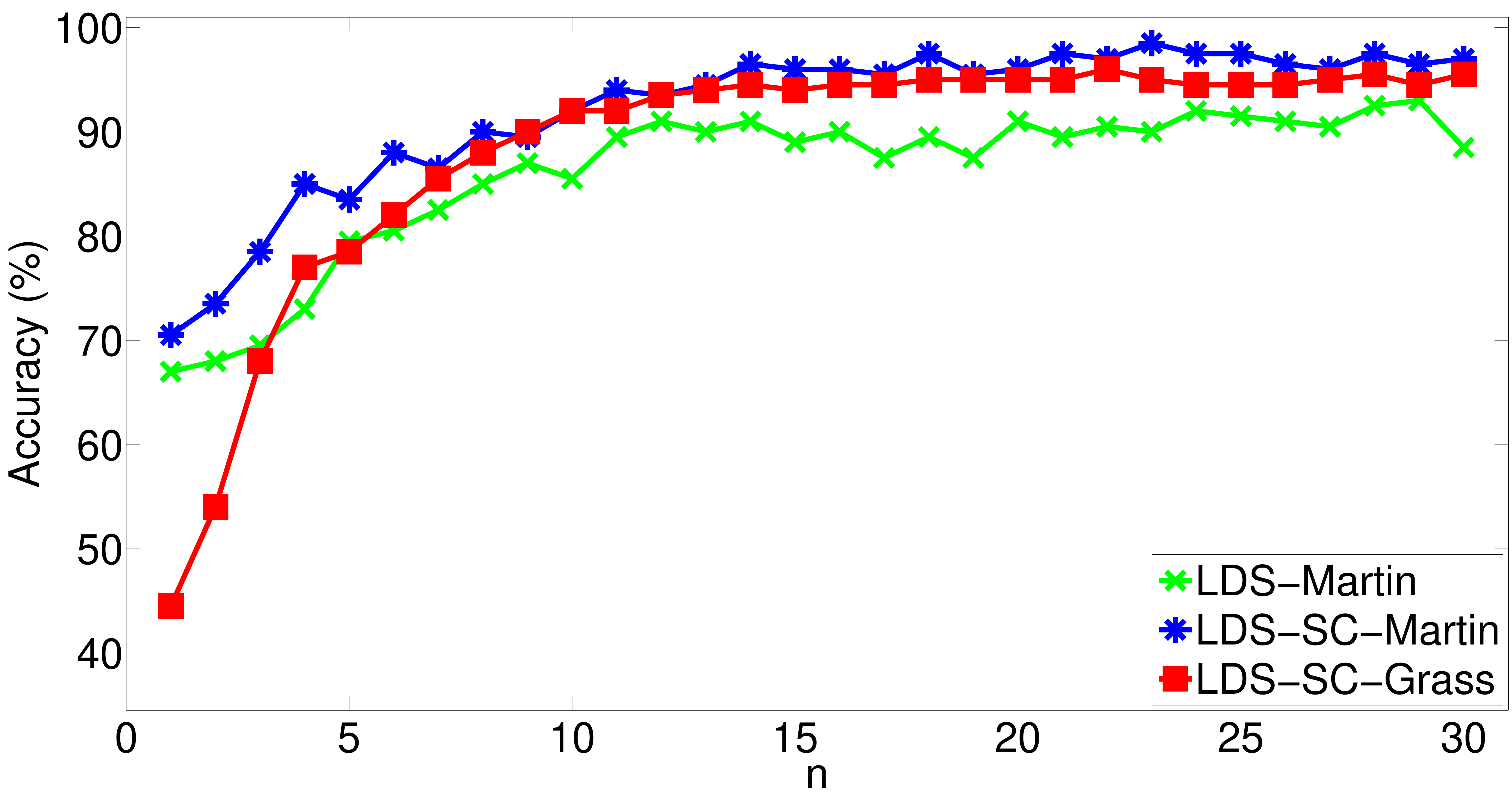}
}
\caption{Performance of the proposed sparse coding methods on the \emph{UCLA} dataset.
(a) The classification accuracies of LDS-SC-Grass and gLDS-SC versus observability order $L$ (here, $n=20$). (b) Training time of
LDS-SC-Grass and gLDS-SC versus the observability order $L$. (c) Performances of LDS-Martin, LDS-SC-Martin and LDS-SC-Grass versus state dimensionality $n$.}
\label{Fig:SC}
\end{center}
\end{figure*}

\begin{figure*}[!ht]
\begin{center}
\subfigure[]{
\includegraphics[width=4.8cm,height=3.6cm]{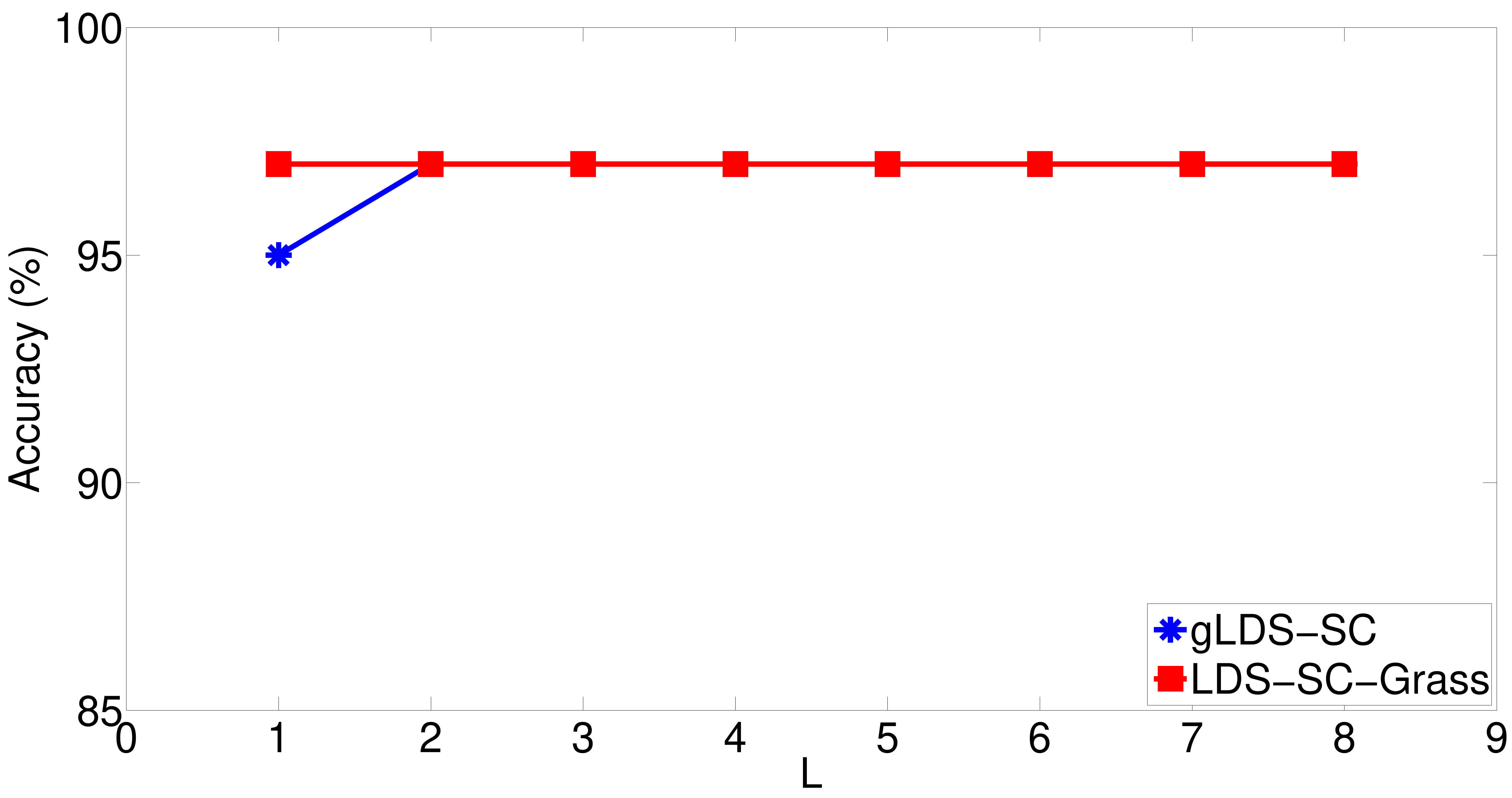}
}
\subfigure[]{
\includegraphics[width=4.8cm,height=3.6cm]{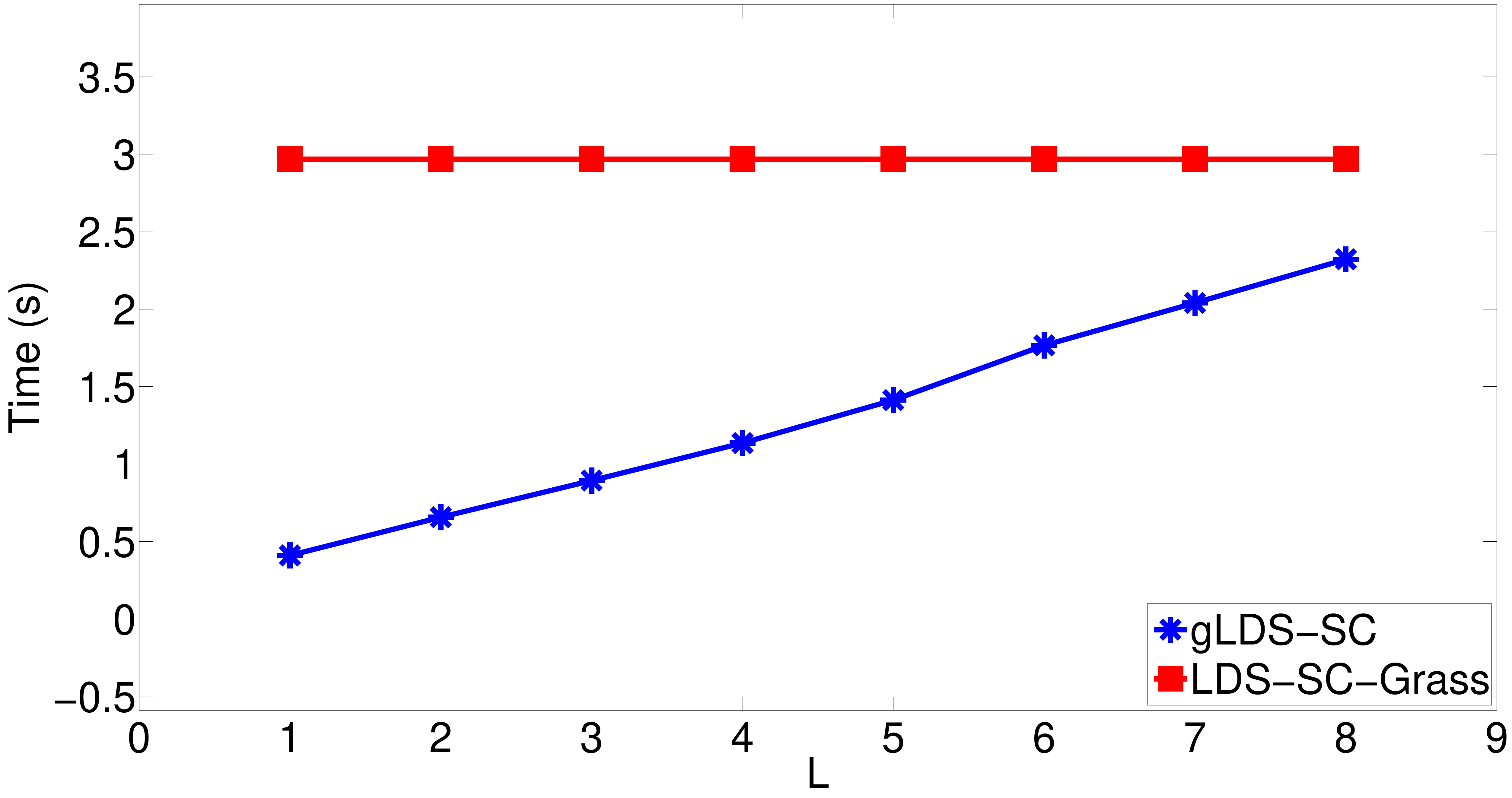}
}
\subfigure[]{
\includegraphics[width=4.8cm,height=3.6cm]{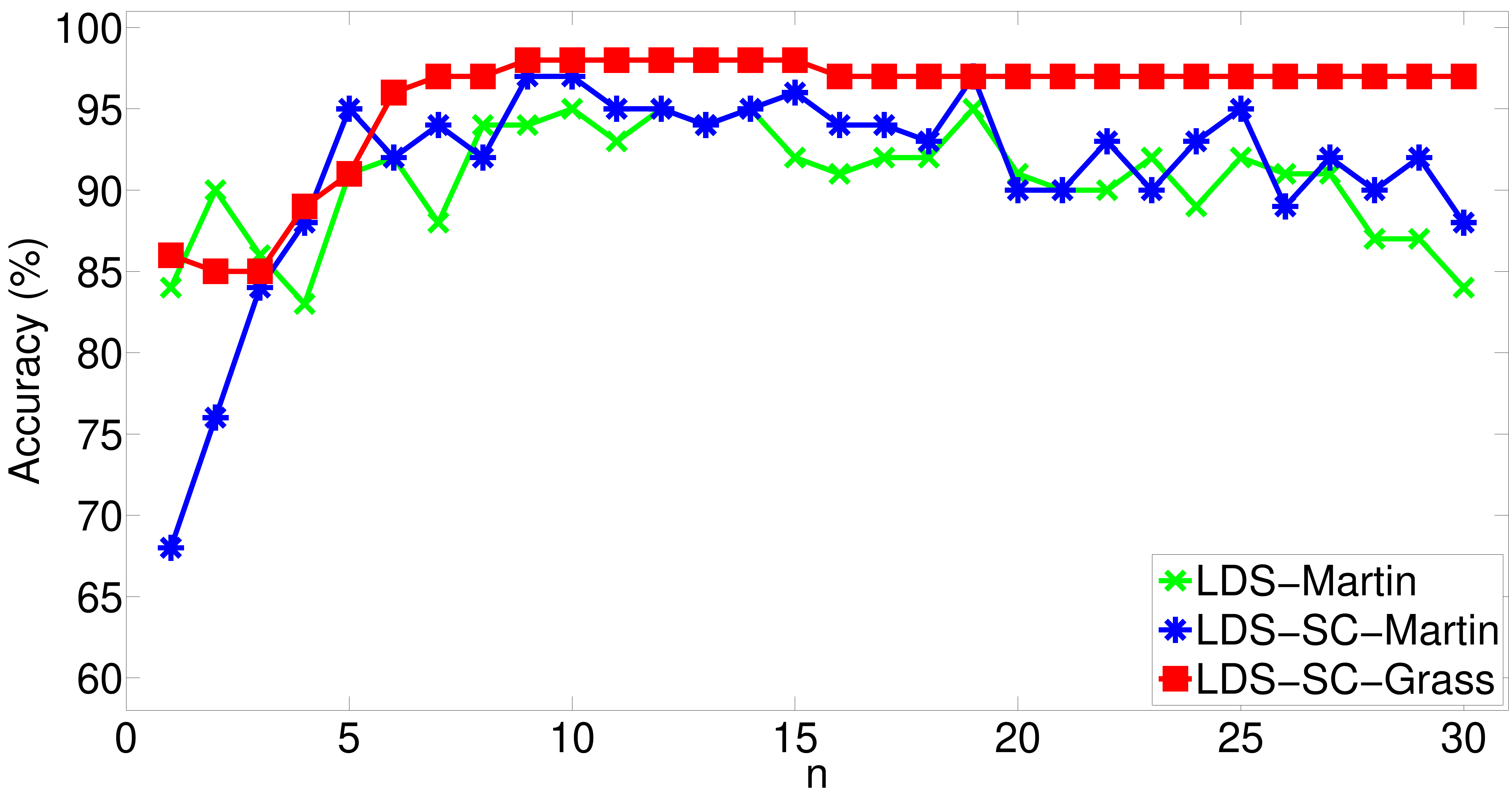}
}
\caption{Performance of the proposed sparse coding methods on \emph{SD}. (a) The classification accuracies of LDS-SC-Grass and gLDS-SC versus observability order $L$ (here, $n=20$). (b) Training time of
LDS-SC-Grass and gLDS-SC versus the observability order $L$. (c) Performances of LDS-Martin, LDS-SC-Martin and LDS-SC-Grass versus state dimensionality $n$.}
\label{Fig:SD}
\end{center}
\end{figure*}

\begin{figure*}[!ht]
\begin{center}
\subfigure[]{
\includegraphics[width=4.8cm,height=3.6cm]{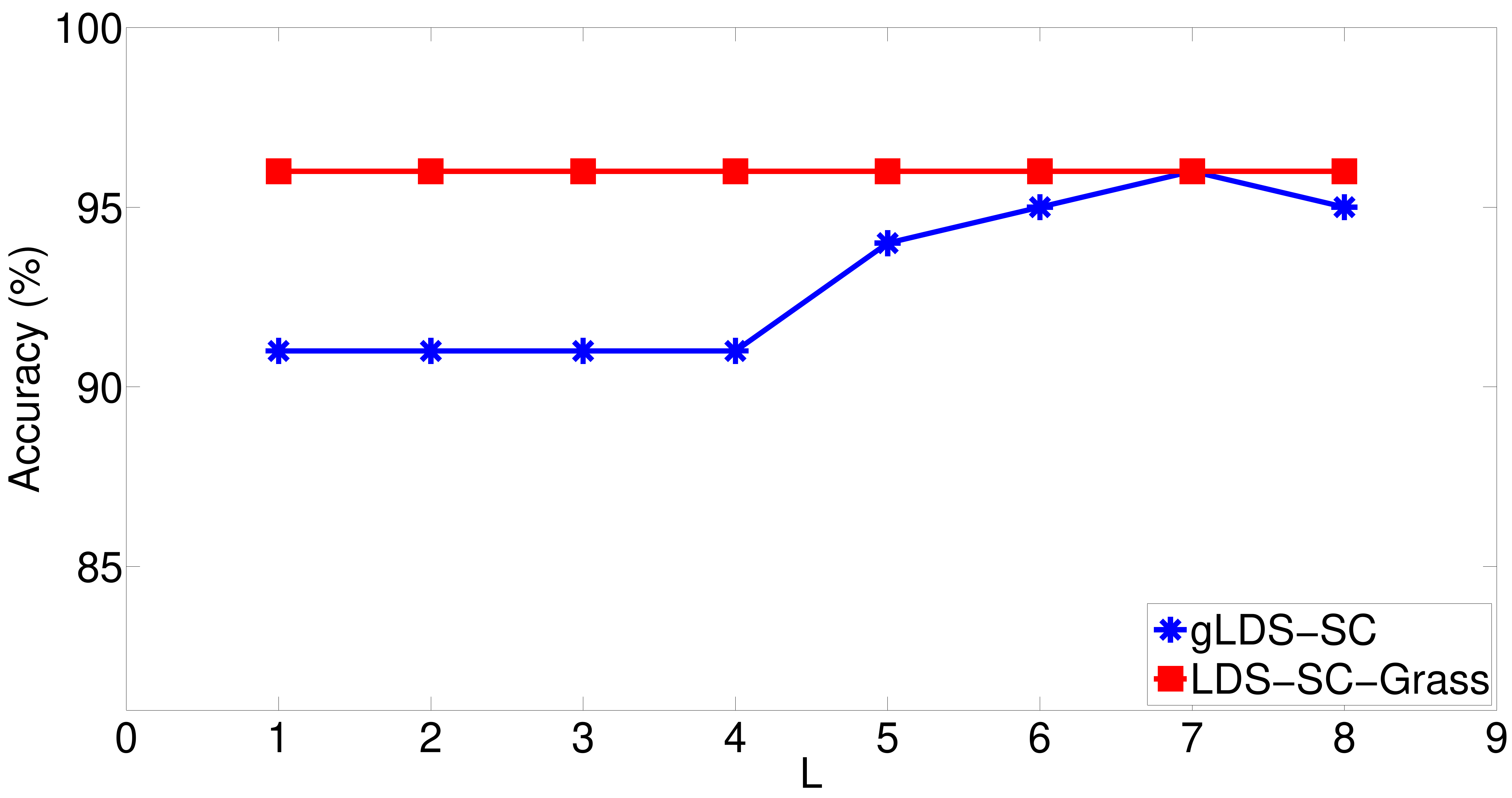}
}
\subfigure[]{
\includegraphics[width=4.8cm,height=3.6cm]{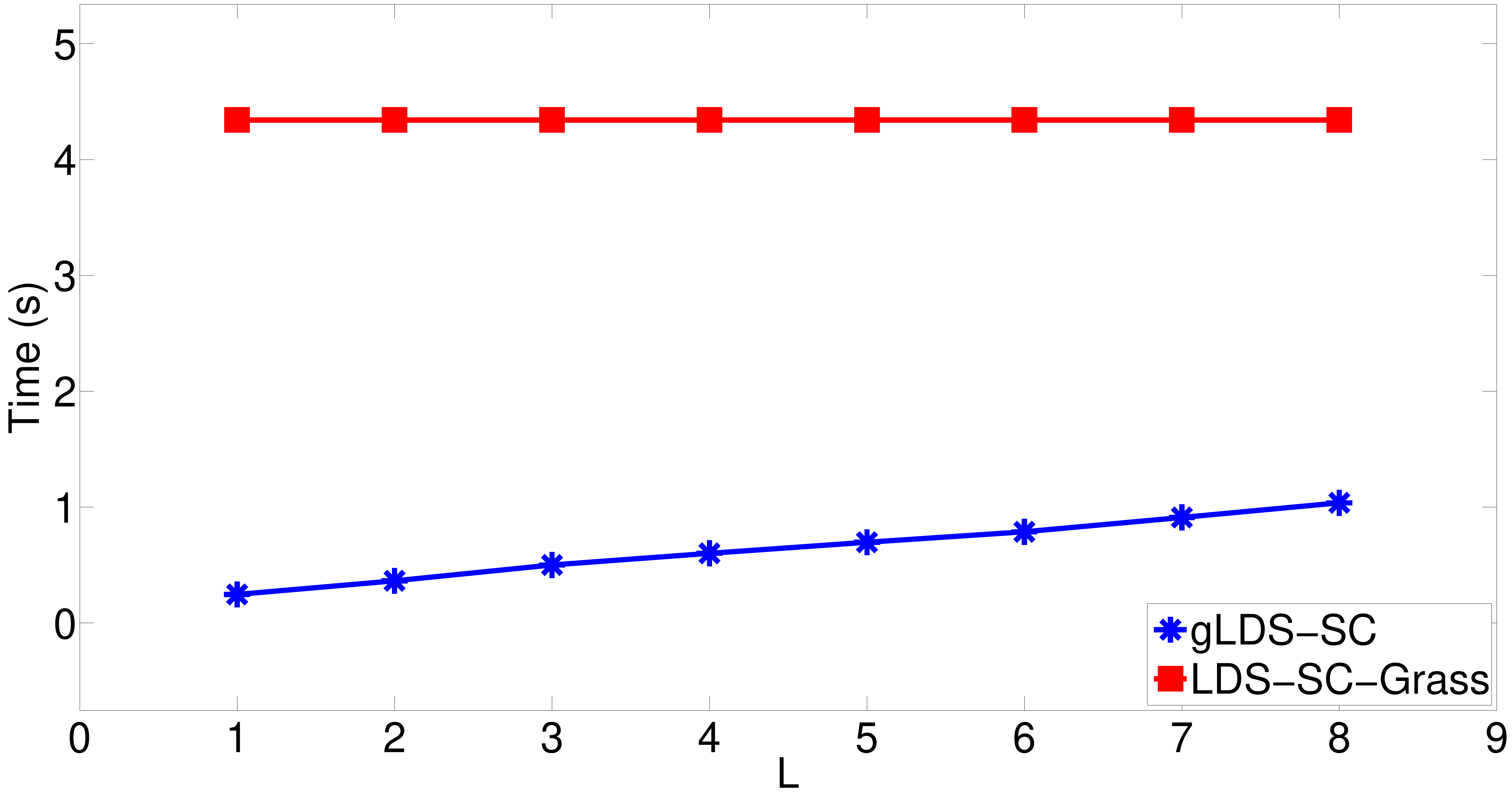}
}
\subfigure[]{
\includegraphics[width=4.8cm,height=3.6cm]{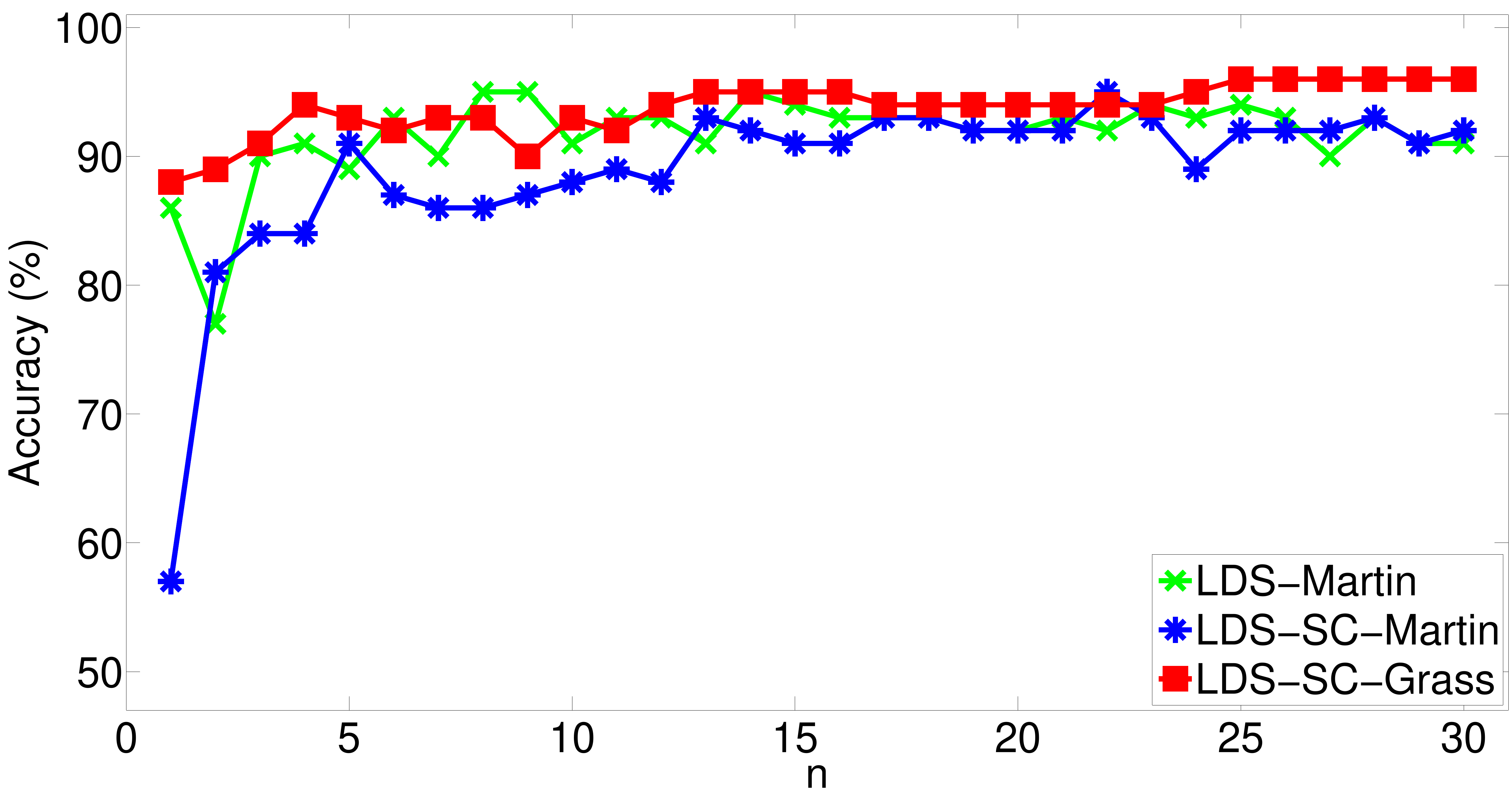}
}
\caption{Performance of the proposed sparse coding methods on \emph{SPR}. (a) The classification accuracies of LDS-SC-Grass and gLDS-SC versus observability order $L$ (here, $n=20$). (b) Training time of
LDS-SC-Grass and gLDS-SC versus the observability order $L$. (c) Performances of LDS-Martin, LDS-SC-Martin and LDS-SC-Grass versus state dimensionality $n$.}
\label{Fig:SPR}
\end{center}
\end{figure*}

\begin{figure*}[!ht]
\begin{center}
\subfigure[]{
\includegraphics[width=4.8cm,height=3.6cm]{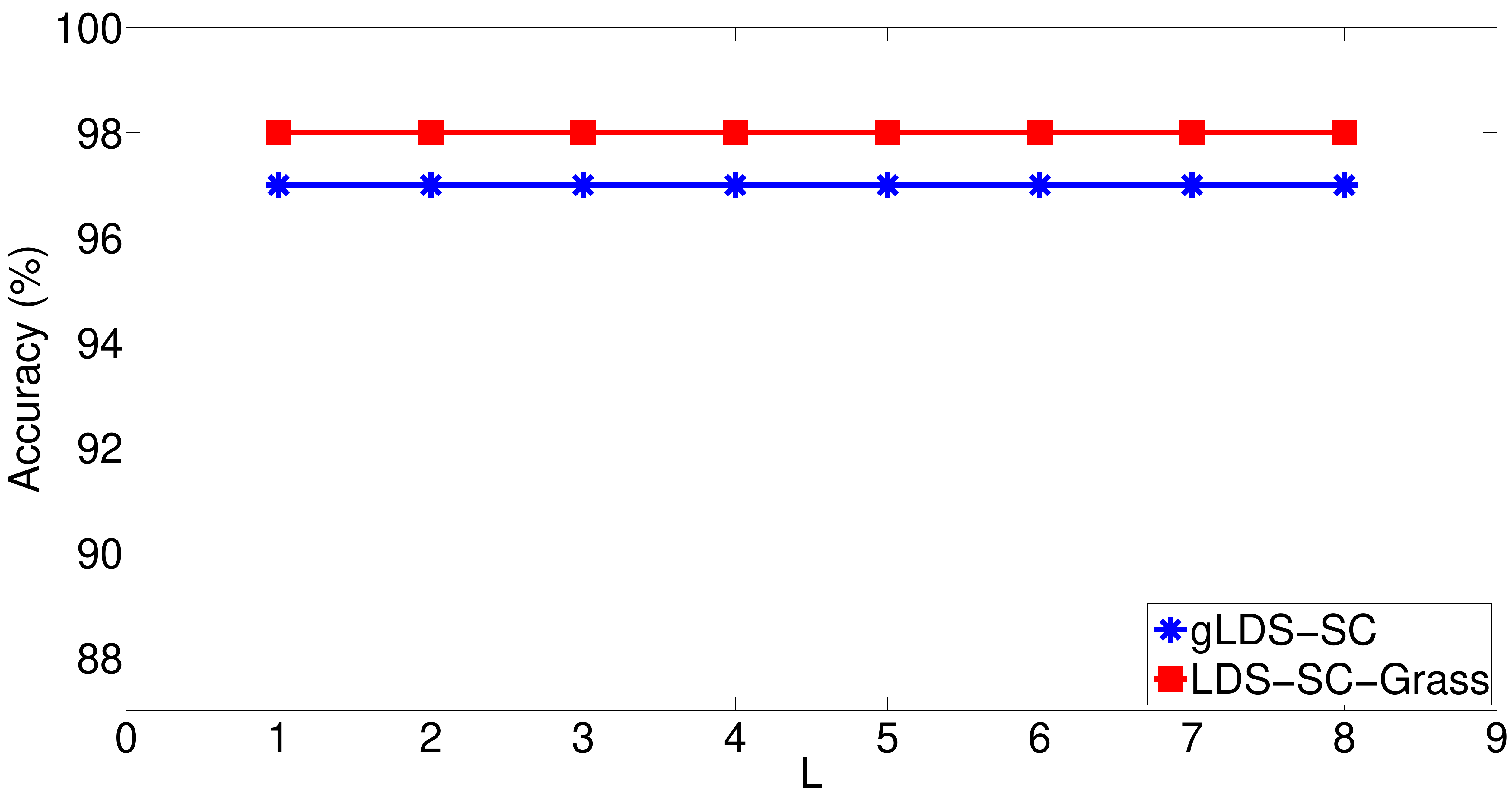}
}
\subfigure[]{
\includegraphics[width=4.8cm,height=3.6cm]{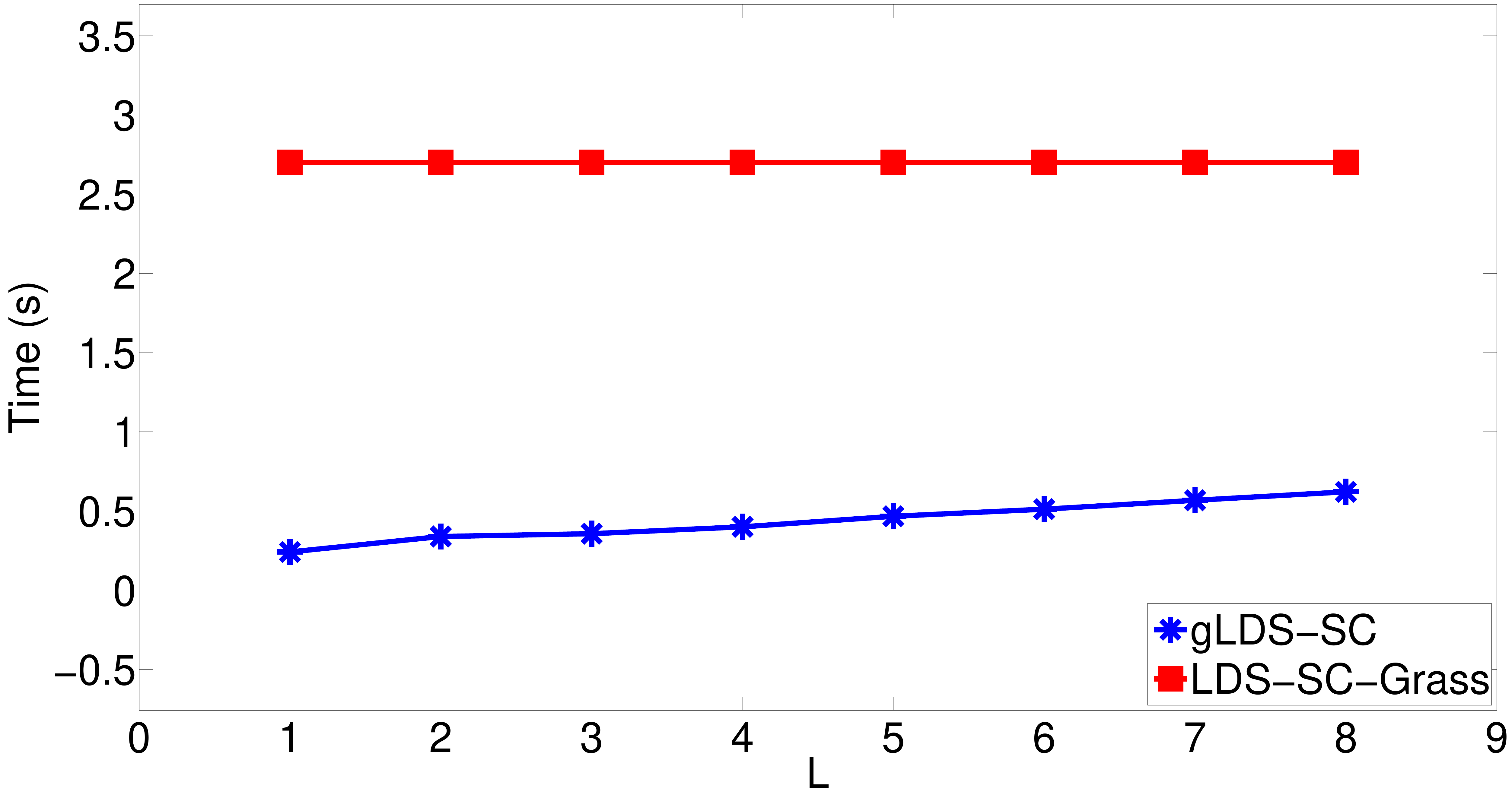}
}
\subfigure[]{
\includegraphics[width=4.8cm,height=3.6cm]{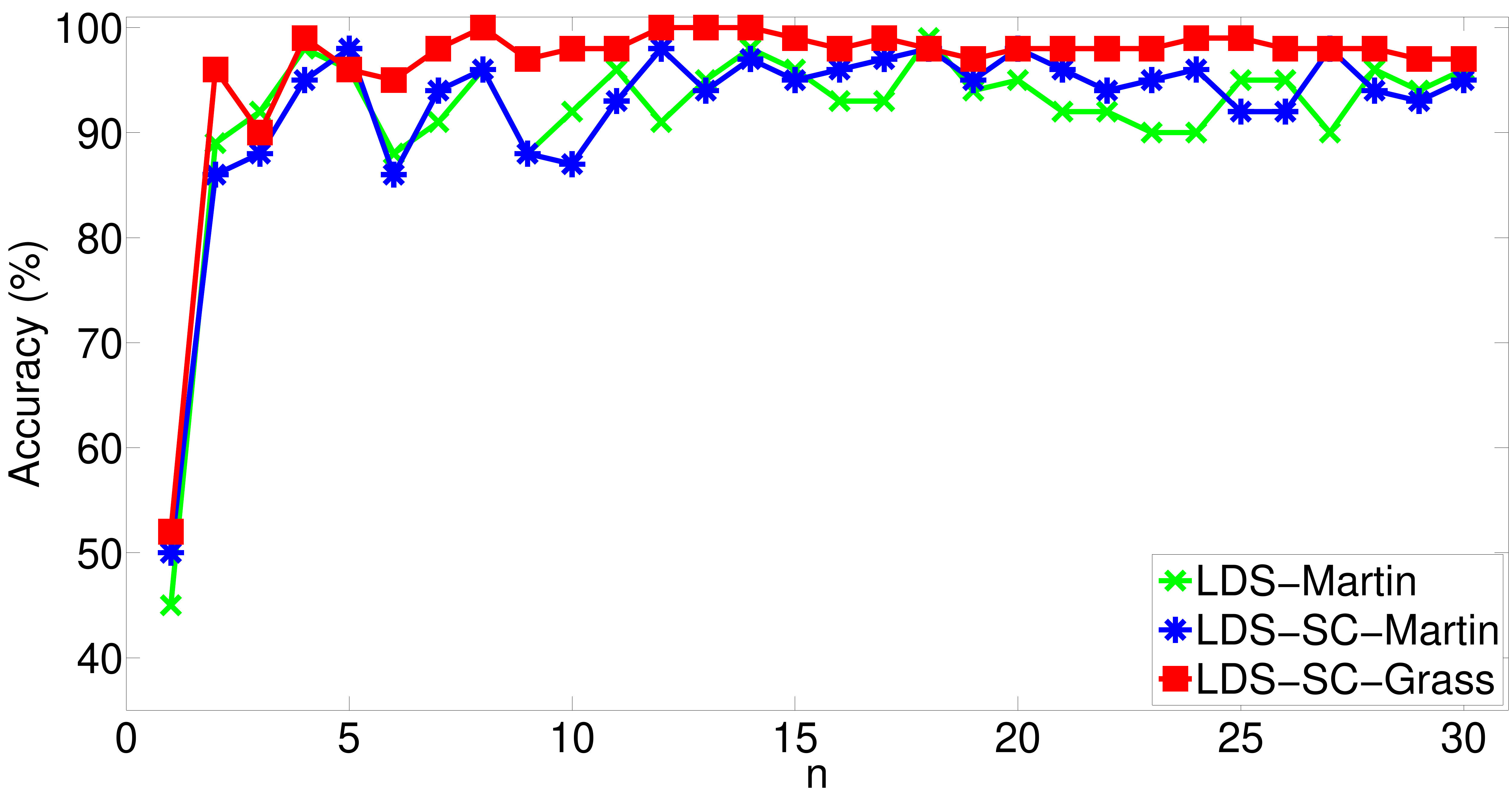}
}
\caption{Performance of the proposed sparse coding methods on \emph{BDH}. (a) The classification accuracies of LDS-SC-Grass and gLDS-SC versus observability order $L$ (here, $n=20$). (b) Training time of
LDS-SC-Grass and gLDS-SC versus the observability order $L$. (c) Performances of LDS-Martin, LDS-SC-Martin and LDS-SC-Grass versus state dimensionality $n$.}
\label{Fig:BDH}
\end{center}
\end{figure*}

\begin{table*}[!ht]
\begin{center}
\caption{Averaged classification accuracies of the proposed sparse coding methods compared with the state-of-the-arts.}
\label{Tab:SC}
\begin{threeparttable}
\begin{tabular}{|c|c|c|c|c|c|c|}
\hline
\multirow{2}{*}{Datasets} & \multirow{2}{*}{References}                              & \multirow{2}{*}{LDS-SVM} & \multirow{2}{*}{LDS-Martin} & \multicolumn{3}{c|}{Proposed models}         \\ \cline{5-7}
                          &                                                          &                          &                             & Our best      & LDS-SC-Martin & LDS-SC-Grass \\ \hline
\emph{Cambridge}              & {\ul 90.7}\tnote{1} ; 83.1\tnote{2}                            & 86.5                     & 86.9                        & \textbf{91.7} & 88.6          & 91.7         \\ \hline
\emph{UCSD}                   &  {94.5}\tnote{3} ; 87.8\tnote{4} & 94.5                     & 92.1                        & \textbf{94.9} & 94.9          & 94.5 \tnote{9}    \\ \hline
\emph{UCLA}                   & 96.0\tnote{3} ; {\ul 97.5}\tnote{5}                & 95.0                     & 93.0                        & \textbf{98.5} & 98.5          & 96.0         \\ \hline
\emph{SD}                     & 97\tnote{6} ; 92\tnote{7}                                                 & \textbf{98}              & 95                          & \textbf{98}   & 97            & 98           \\ \hline
\emph{SPR}                    & 91\tnote{6} ; 89\tnote{7}                                                 & 94                       & {\ul 95}                    & \textbf{96}   & 95            & 96           \\ \hline
\emph{BDH }                   & 81\tnote{6} ; 87\tnote{8}                                       & 80                       & {\ul 98}                    & \textbf{100}  & 98            & 100          \\ \hline
\end{tabular}
\begin{tablenotes}
    \footnotesize
    \item[1] KSLCC: \cite{Harandi_2015_CVPR}.
    \item[2] SSSC: \cite{mahmood2014semi}.
    \item[3] KL-SVM: \cite{chan2005probabilistic}. 
    \item[4] CS-LDS: \cite{sankaranarayanan2013compressive}.
    \item[5] KDT-SVM: \cite{chan2007classifying}.
    \item[6] ST-HMP: \cite{madry2014st}.
    \item[7] DTW: \cite{drimus2014design}. 
    \item[8] JKSC: \cite{jingwei15tactile}.
    \item[9] The result of LDS-SC-Grass on UCSD is better than our earlier work~\cite{wenbing2016sparse}, as here we employe the SN stabilization before coding while in \cite{wenbing2016sparse} we simply stabilize $A$ by dividing it with a scale factor.
\end{tablenotes}
\end{threeparttable}
\end{center}
\end{table*}

\subsection{Sparse coding}\label{Sec:exp_SC}

In this part, we assess the performance of the proposed sparse coding techniques by constructing
the LDS dictionary directly from the training data, meaning each atom in the dictionary is one sample from the training set.
Classification is done using the approach presented in \textsection~\ref{Sec:reconstruction_error_approach}.
The experiments are carried out on \emph{Cambridge}, \emph{UCSD}, \emph{UCLA}, \emph{SD}, \emph{SPR}, and \emph{BDH}.

\textbf{Non-stable vs. Stable.}
We start by studying the stability procedure proposed in \textsection~\ref{Sec:stable_LDS} for classification purposes. Thus, we extract dynamical features in two different ways: one with and one without SN.
Though our SN method avoids any optimization procedure, it is interesting to compare it against optimization-based methods
to show its full potential. Among LDS stabilization methods that benefit from optimization techniques, WLS shows to be fast while achieving  small reconstruction errors~\cite{wenbing2016learning}. Here, we only consider the diagonal form of WLS, \ie, DWLS.
The stable bound in DWLS is set to 0.99 to make the computation of the Lyapunov equation possible. The exacted features are fed to LDS-SC-Grass for classification.
As shown in Fig.~\ref{Fig:stability}, stabilization can promote the classification accuracies on \emph{Cambridge}, \emph{UCSD}, \emph{UCLA}, \emph{SPR} datasets while not helping on \emph{SD} and \emph{BDH} datasets.
In terms of the state reconstruction errors (see~\cite{wenbing2016learning} for details), SN underperforms in comparison to DWLS. However, when classification accuracy is considered, SN generally yields better results  while being remarkably faster. We conjecture that SN can preserve the discriminative information contained in the data sequences. Based to the results here, we will only perform SN on \emph{Cambridge}, \emph{UCSD}, \emph{UCLA}, \emph{SPR} datasets in the following experiments.

\noindent
\textbf{Infinite vs. Finite.}\\
As discussed in \textsection~\ref{Sec:SC}, LDS-SC-Grass can be understood as a generalization of gLDS-SC from finite-dimensional Grassmannian to infinite Grassmannian. Here, we are interested in the asymptotic behavior of gLDS-SC when the observability order $L$ increases. For this purpose, we report the results in Figures~\ref{Fig:Cambridge}-\ref{Fig:BDH}. As expected, the classification accuracy of gLDS-SC converges to that of LDS-SC when $L$ increases.
In \textsection~\ref{Sec:computational_complexity}, we have shown that the computational complexity of gLDS-SC ($O(L(NJ+J^2)mn^2)$) is $L$ times more than that of LDS-SC ($O((NJ+J^2)mn^2)$). Larger $L$ demands for more computational resources in gLDS-SC.
As shown in Fig.~\ref{Fig:SC} (b), gLDS-SC takes more coding time than LDS-SC when $L>2$. On tactile datasets, LDS-SC-Grass performs much slower than gLDS-SC when $L<8$. This is because LDS-SC-Grass additionally requires SVD decomposition and solving the Lyapunov equation, both are $O(n^3)$ computation-wise.
When the dimensions of the data $m$ is large (\eg, UCSD and UCLA datasets), additional computations have a small impact in the coding time.

\noindent
\textbf{Varying $n$.}\\
To evaluate the sensitivity of the coding algorithms with respect to the hidden dimensionality (\ie, $n$), we performed an experiment
using LDS-Martin, LDS-SC-Martin and LDS-SC-Grass methods.
From Fig.~\ref{Fig:SC} (c), we can see that the proposed methods, \ie, LDS-SC-Martin and LDS-SC-Grass, perform consistently when $n$ is greater than a certain value. Also both proposed methods outperform the LDS-Martin when $n$ is sufficiently large.
LDS-SC-Martin performs better than LDS-SC-Grass on \emph{UCLA} and \emph{UCSD}, indicating that applying a Gaussian kernel plus Martin distance on these two datasets is beneficial. However, on other datasets (including all the tactile datasets), LDS-SC-Grass is better than LDS-SC-Martin.
The take-home message here is that, while employing different kernel functions can lead to slight improvements, sparse coding is
a robust and powerful method for analyzing LDSs.

\noindent
\textbf{Comparison with the state-of-the-art.}\\
We compare the proposed sparse coding methods, \ie, LDS-SC-Martin and LDS-SC-Martin, against the state-of-the-art in this part.
Beside the dataset-dependent state-of-the-arts, we use two baselines, namely LDS-Martin and LDS-SVM.
For the proposed methods and also LDS-Martin and LDS-SVM, we vary the parameter $n$ and report the best results in Table~\ref{Tab:SC}.
We first note that the best results of our proposed algorithms outperform all other baselines and state-of-the-arts on all datasets.
On the \emph{UCSD} dataset, the method proposed in~\cite{chan2005probabilistic} achieves a similar performance to that of LDS-SC-Martin.
As discussed in our preliminary study~\cite{wenbing2016sparse}, one can combine the state covariance term into the sparse coding formulation (see Eq.~\eqref{Eq:hybrid-kernel}) to boost the accuries more. With such combination, the accuracy of LDS-SC-Grass increases from
98\% and 96\% to 100\% and 97\% on the \emph{SD} and \emph{SPR} datasets, respectively.

\subsection{Dictionary learning}
In this part, we analyze the effectiveness of the proposed dictionary learning algorithms. Experiments are carried out on the \emph{Cambridge} and \emph{DynTex++} datasets. For the \emph{Cambridge} dataset, we considered a different testing protocol compared to that of the sparse coding. In particular, we split the videos of each class into two non-overlapping and equal-sized sets and used the first half for learning the dictionary. The random splitting is repeated $10$ times and the average accuracies over $10$ trials are reported here. The sparse codes for training and test data with respect to a learned dictionary are fed to a linear SVM~\cite{fan2008liblinear} for classification.
The parameter $n$ is fixed to $10$ in all the experiments.

\begin{figure}[!h]
\begin{center}
\subfigure{
\includegraphics[width=6cm]{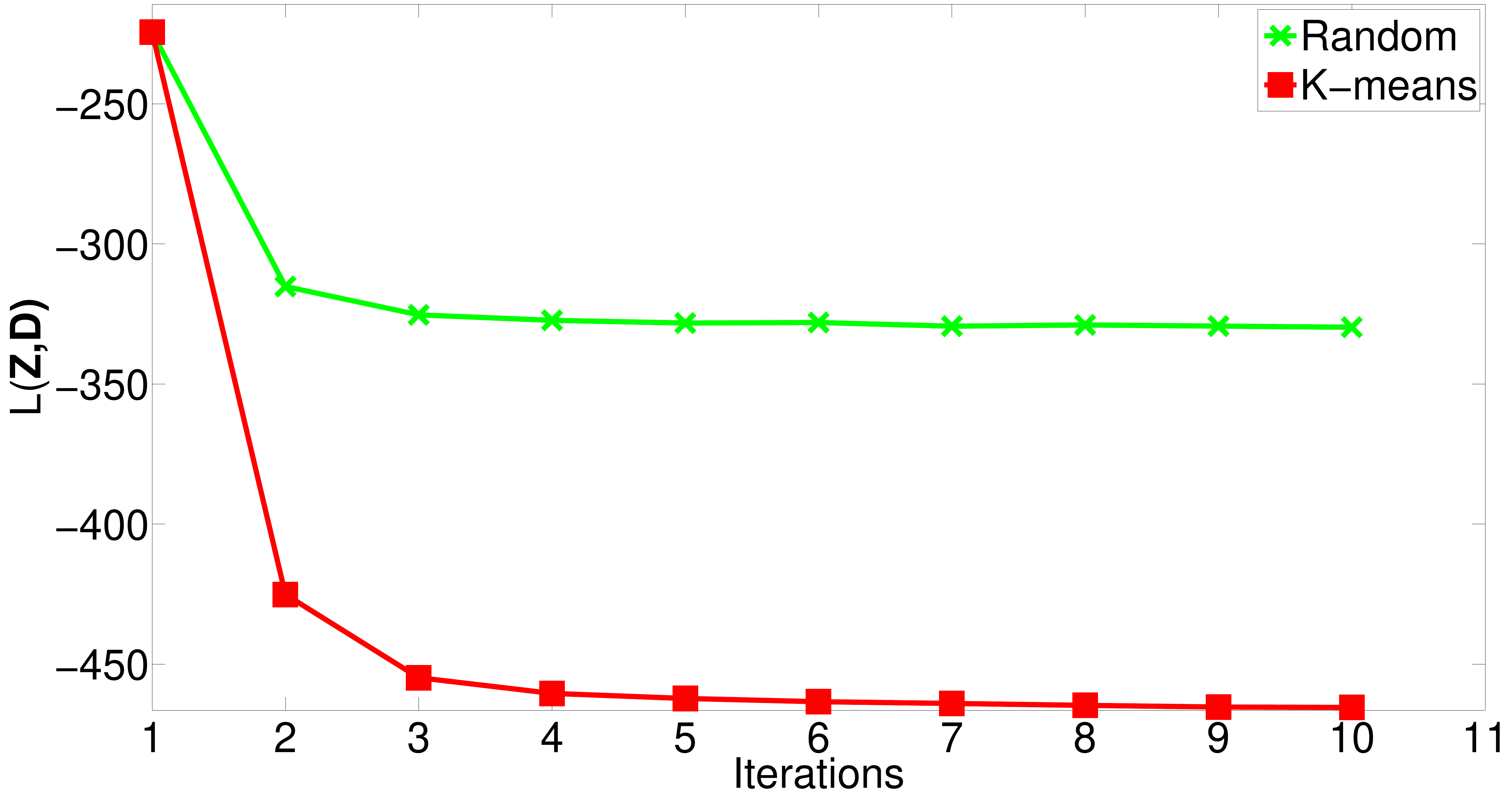}
}
\caption{The cost of LDS-DL with Algorithm~\ref{Alg:DL} on \emph{Cambridge}. The vertical axis represents the cost (Eq.~\eqref{Eq:DL}) and the horizontal axis denotes the iteration number where each iteration refers to a complete update of all dictionary atoms. $J=8$.}
\label{Fig:DL_convergence}
\end{center}
\end{figure}

\begin{figure}[!h]
\begin{center}
\subfigure{
\includegraphics[width=6cm]{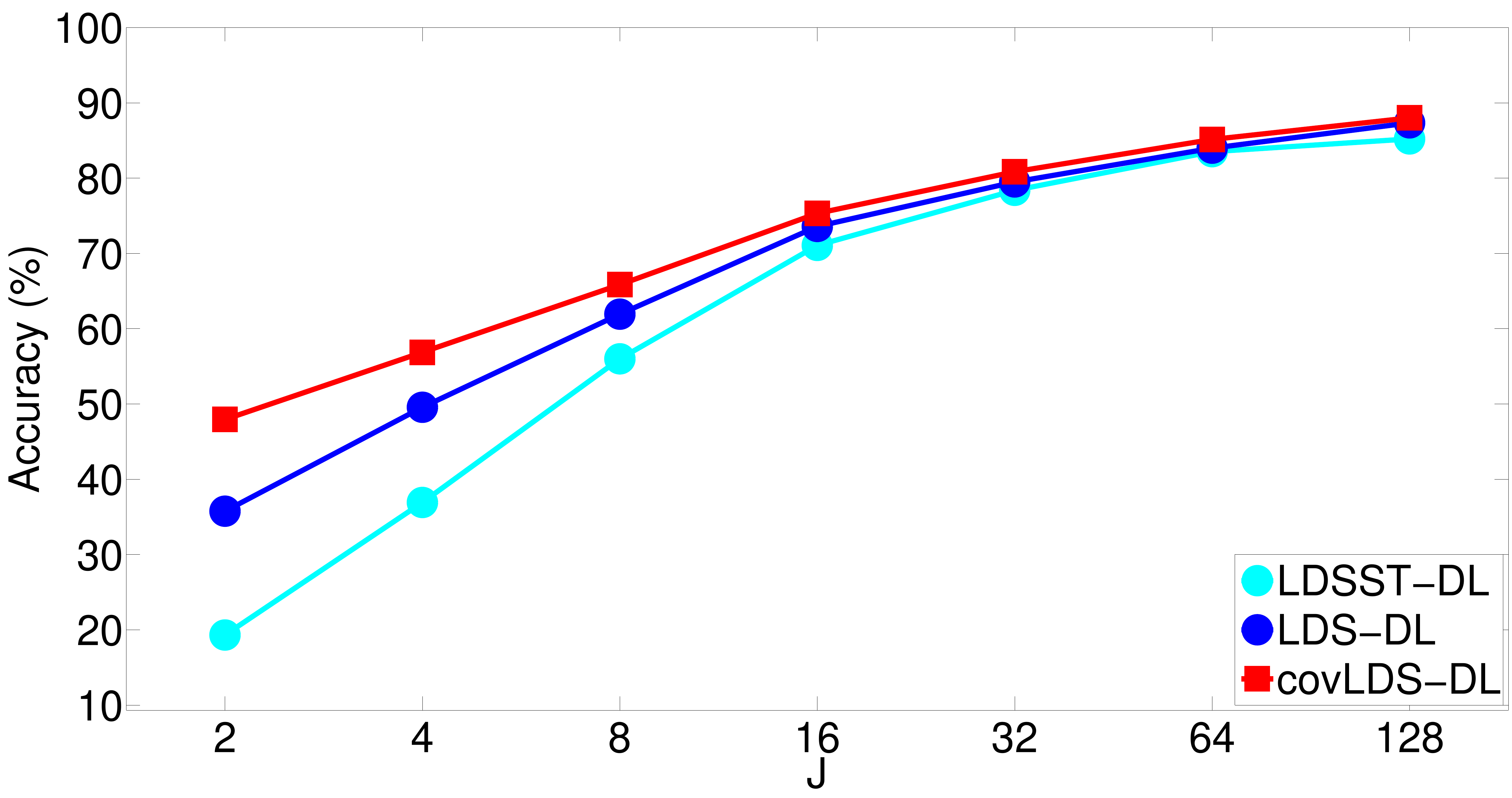}
}
\caption{Comparison between LDSST-DL, LDS-DL and covLDS-DL with varying values of $J$ on \emph{Cambridge}. For covLDS-DL, $\beta=0.2$.}
\label{Fig:Cambridge_SymSkew_vs_SymGrass}
\end{center}
\end{figure}

\begin{figure*}[!th]
\begin{center}
\subfigure{
\includegraphics[width=15cm]{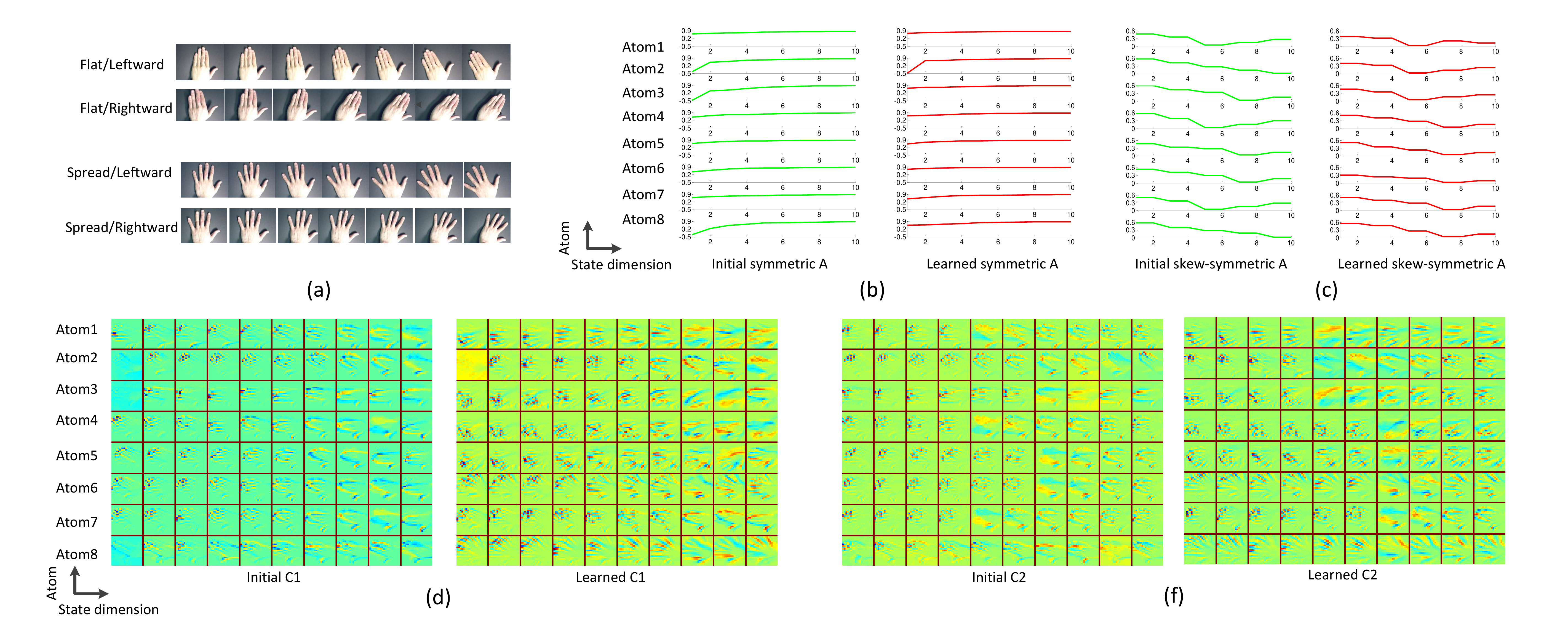}
}
\caption{ Visualization of  the initial and the learned dictionaries on \emph{Cambridge}. $J=8$. (a) Samples of the $4$ sub-categories in \emph{Cambridge}. (b-c) Plots of $\Mat{A}$: different plots display the values of the transition eigenvalues of different atoms; (d-f) Visualization of A: rows corresponds to atoms and columns to the state dimensions.}
\label{Fig:visualizeDictionary}
\end{center}
\end{figure*}

\begin{figure*}[!ht]
\begin{center}
\subfigure{
\includegraphics[width=3.6cm, height=3cm]{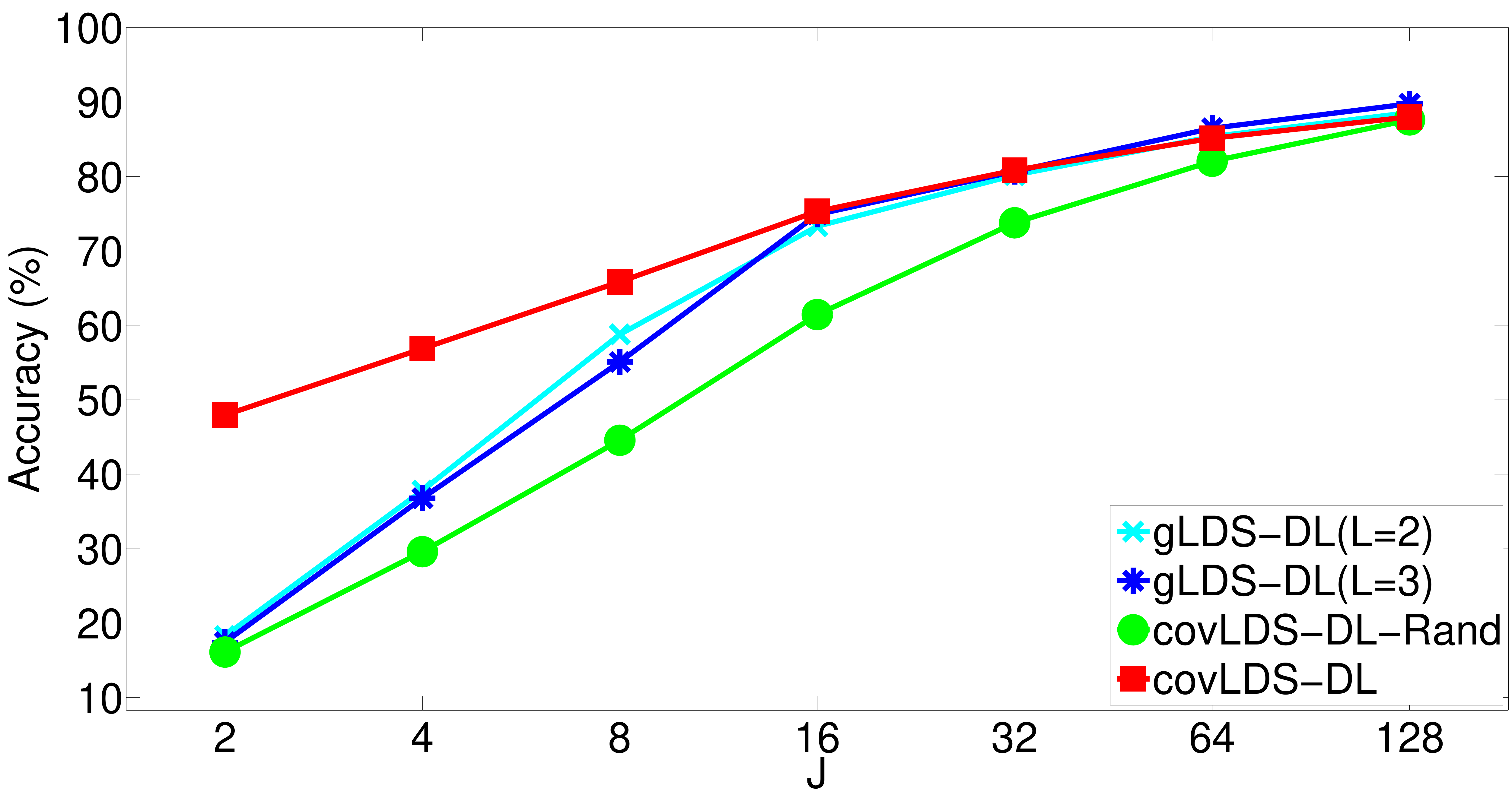}
}
\subfigure{
\includegraphics[width=3.2cm, height=3cm]{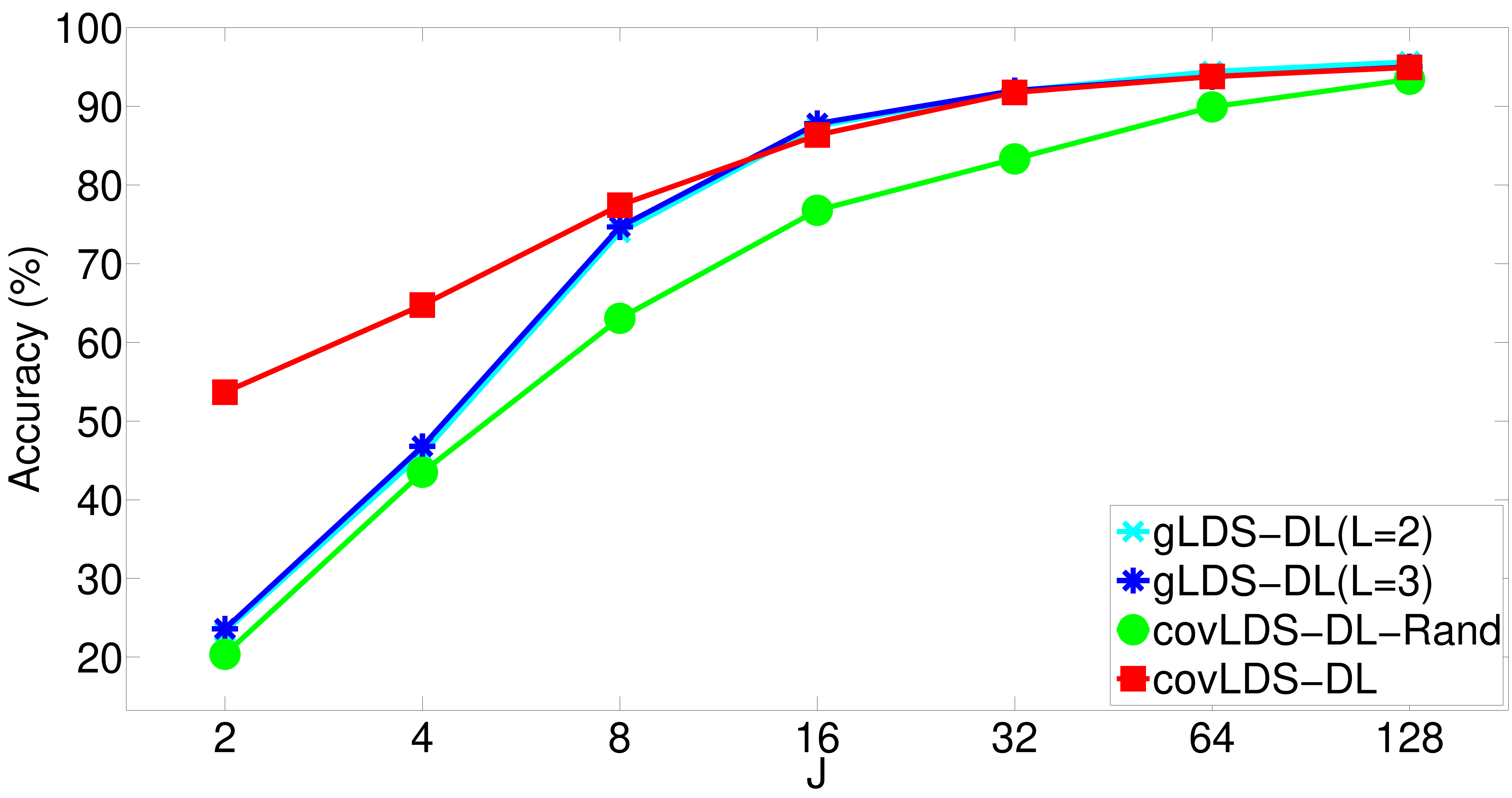}
}
\subfigure{
\includegraphics[width=3.2cm, height=3cm]{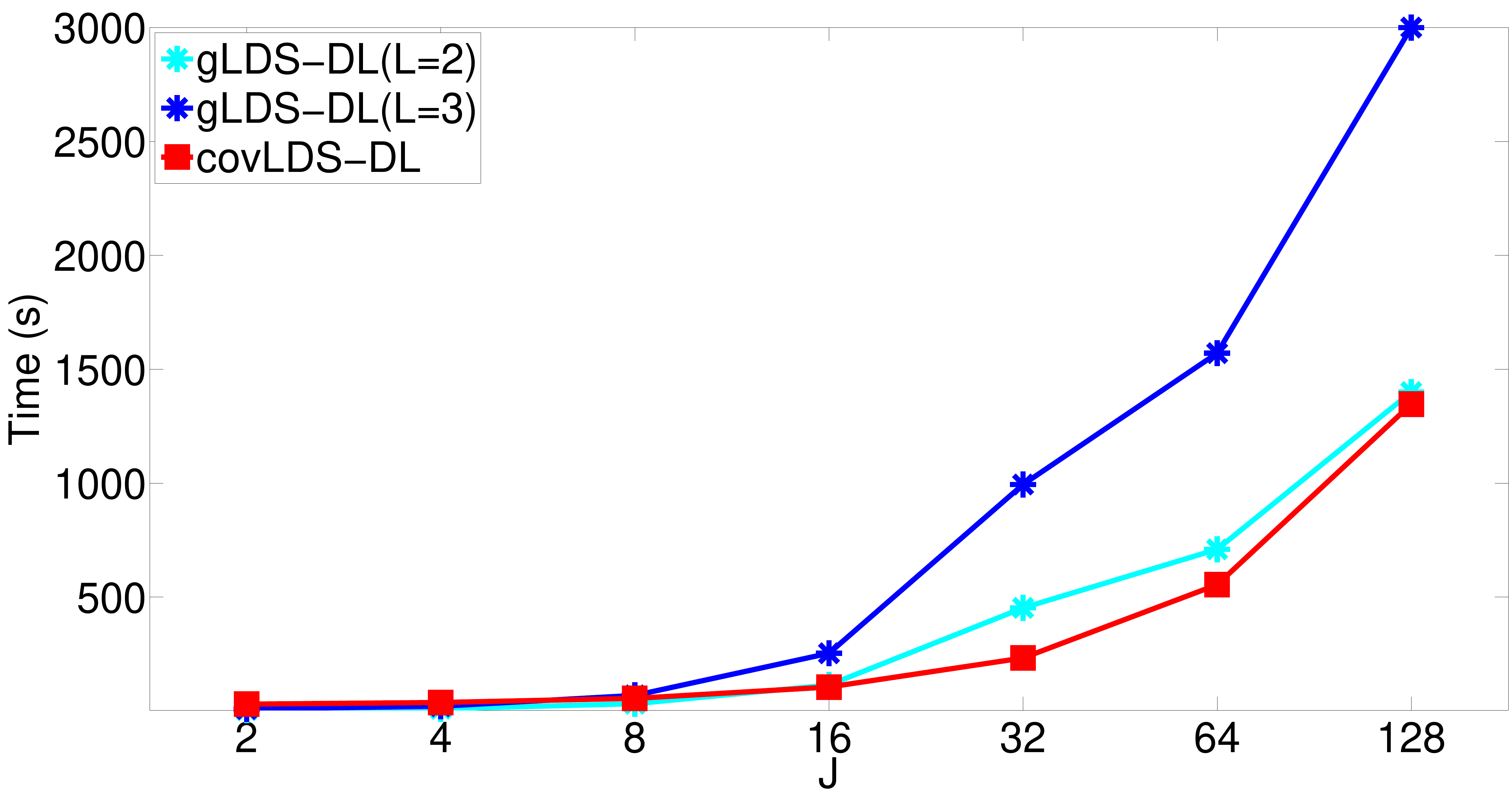}
}
\subfigure{
\includegraphics[width=3.2cm, height=3cm]{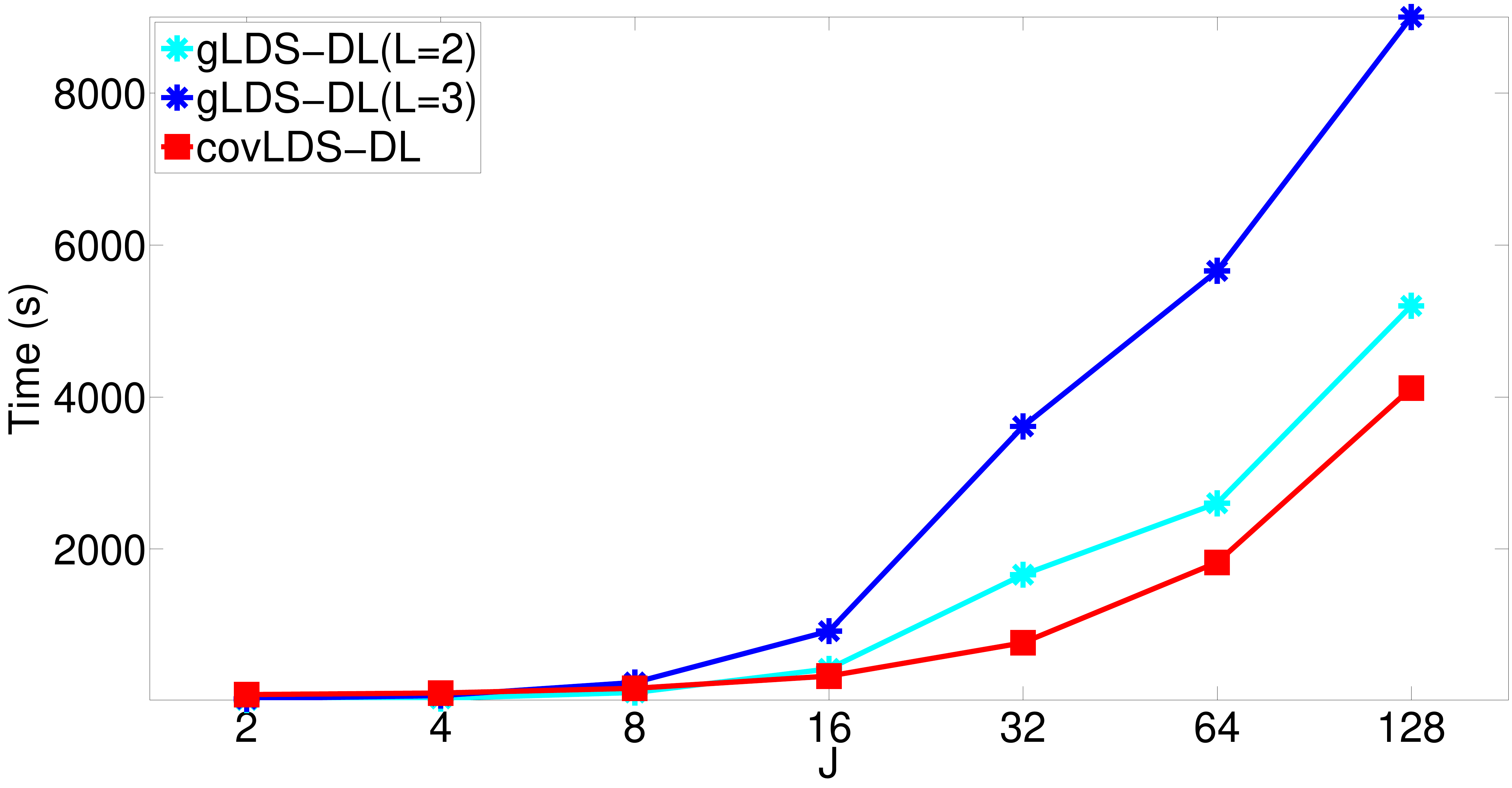}
}
\caption{Comparisons between covLDST-DLs and gLDS-DLs for various dictionary size $J$ on the \emph{Cambridge} and \emph{DynTex++}
(the $9$-classes subset) datasets.  The first two panels show the average classification accuracies, while the right two ones display the training time for a one update of all atoms.}
\label{Fig:DL}
\end{center}
\end{figure*}


\noindent
\textbf{Convergence analysis.}\\
At each step of our dictionary learning algorithm (Algorithm~\ref{Alg:DL}), the updated columns of the measurement matrices are optimal according to Theorems~\ref{Th:sym-column-U} and \ref{Th:skew-column-U}. Each diagonal element of the transition matrices is updated in the opposite direction of the gradient, thus reducing the cost function. To further show this, we randomly initialize the dictionary atoms with the training data and plot the convergence behavior of LDS-DL in Fig.~\ref{Fig:DL_convergence}. This plot suggests that the algorithm convergences in a few iterations.

In addition to random initialization, we are also interested in developing a K-means-like approach to cluster data sequences prior to dictionary learning.  Once such mechanism is at our disposal, we can make use of the clustering centers to initialize the dictionary. Recalling that the traditional K-means paradigm consists of two alternative phases: first assigning the data to the closest centers under the given distance metric and then updating the centers with the new assigned data~\cite{bishop2006pattern}. Extending the first step is straightforward for LDS clustering; the second one entails a delicate method for computing the clustering centers. Suppose the extended observability subspaces in a cluster are $\{\Mat{V}_i\}_{i=1}^{q}$. We define the center of this set as $\Mat{V}_c=\arg\min_{\Mat{V}_c}\parallel\Mat{V}_c\Mat{V}_c^{\mathrm{T}}-\frac{1}{q}\sum_{i=1}^q\Mat{V}_i\Mat{V}_i^{\mathrm{T}}\parallel_{F}^2$. Unfolding this equation actually leads to a special case of Eq.~\eqref{Eq:gamma} by setting the number of the dictionary atoms $J$ to be 1 and the values of the codes $\Mat{Z}_{i,r}$ to $1/q$. This means that we can calculate the system tuples of the centers with Algorithm~\ref{Alg:DL}. The convergence curve of LDS-DL with the K-means initialization is shown in Fig.~\ref{Fig:DL_convergence}. We find that the objective decreases significantly after the K-means initialization, leading to superior performances compared to random initialization. Thus, in the following experiments, we employ K-means to initialize the dictionary unless mentioned otherwise.

To learn an LDS dictionary, our preliminary study~\cite{wenbing2016sparse} opts for symmetric LDSs to model data.
In this study, we introduce the notion of a two-fold LDS, decomposing an LDS into symmetric and screw-symmetric parts.
In Fig.~\ref{Fig:Cambridge_SymSkew_vs_SymGrass}, we compare the two methods, namely LDSST-DL and LDS-ST, on the \emph{Cambridge} dataset for various size of dictionary. This experiment shows the superiority of LDS-DL over LDSST-DL. In \textsection~\ref{Sec:state_covariance}, we show how to consider the state covariances for learning a more discriminative dictionary. As shown in Fig.~\ref{Fig:Cambridge_SymSkew_vs_SymGrass}, equipping LDS-DL with state covariances, \ie, covLDS-DL, can further boost the performances.

\noindent
\textbf{Visualizing the Dictionary.}\\
The LDS-DL algorithm learns the measurement (\ie, $\Mat{C}$) and the transition (\ie, $\Mat{A}$) matrices explicitly and separately.
Thus, we can visualize the learned system tuples $(\Mat{A},\Mat{C})$ to demonstrate what patterns have been captured.  For simplicity, we perform LDS-DL on the $4$-class subset of \emph{Cambridge}, namely,  Flat-Leftward,  Flat-Rightward, Spread-Leftward, and Spread-Rightward. Dictionary atoms are initialized randomly by choosing $8$ videos from the class Flat-Leftward.

Fig.~\ref{Fig:visualizeDictionary} (a)
visualizes both the initial and the learned pairs. As the system tuples of the skew-symmetric dictionary are complex, we visualize the corresponded real pairs, \ie, $(\Mat{\Theta},\Mat{Q})$. Clearly, more spatial patterns such as the spread-hand shape and the hand-rightward state,  have been captured by the learned measurement matrices. There are also slight changes in transition matrices after training.  The transition matrices of different atoms have a small difference, indicating that  there is not much dynamic, presumably because the speed of hand and the sampling frequency of the camera are consistent.

\noindent
\textbf{Is training useful?}\\
To verify whether learning a dictionary is helpful towards classification purposes or not,  we compare covLDS-DL against a baseline, namely  covLDS-Rand, in which the dictionary atoms are chosen from the training set randomly (no training is involved).
In addition to  covLDS-Rand, we use  gLDS-DL with $L=2, 3$ as another baseline.  For fair comparisons, we use a linear SVM
and set $n=10$ for covLDS-DL, covLDS-Rand and gLDS-DL.

To compare covLDS-DL with covLDS-Rand and gLDS-DL, we use the \emph{Cambridge} and a smaller subset of  \emph{DynTex++} dataset
(only the videos from the first $9$ classes). By cross-validation, $\beta$ is 0.2 in the \emph{Cambridge} and 0.8 in the \emph{DynTex++} dataset  in covLDS-DL. Fig.~\ref{Fig:DL} shows that covLDS-DL consistently outperforms covLDS-Rand for various number of the dictionary atoms.  Compared to gLDS-DLs, covLDS-DL achieves higher accuracies when the dictionary size $J$ is small (\eg, $J<16$), and obtains on par performances when $J$ is large.  As discussed in \textsection~\ref{Sec:state_covariance}, the computational complexity of gLDS-DL is higher than covLDS-DL. We also plot the training time of gLDS-DLs and covLDS-DL in Fig.~\ref{Fig:DL}, showing that gLDS-DL is more computationally exhaustive as $L$ increases.

In addition to the $9$-classes subset, we also evaluate LDS-DL on the whole \emph{DynTex++} dataset. To speed up the convergence rate, we learn the dictionary in a hierarchical manner, \ie, separately learning $J_c$ dictionary atoms for each class.
Fig.~\ref{Fig:DL_hierarchical_accuracy} shows the performance of the hierarchical dictionary for various values of $J_c$.  LDS-DL achieves the accuracy of $91.7\%$ when $J_c=32$, which is better than that of the Grassmaniann-based method (\ie, $90.3\%$)~\cite{harandi2013dictionary} and comparable to Grassmannian-kernel-based method (\ie, $92.8\%$)~\cite{harandi2013dictionary}. By choosing $J_c=50$ the classification accuracy increases to $93.06\%$ (see Fig.~\ref{Fig:DL_hierarchical_accuracy}).

\begin{figure}[!h]
\begin{center}
\subfigure{
\includegraphics[width=6cm]{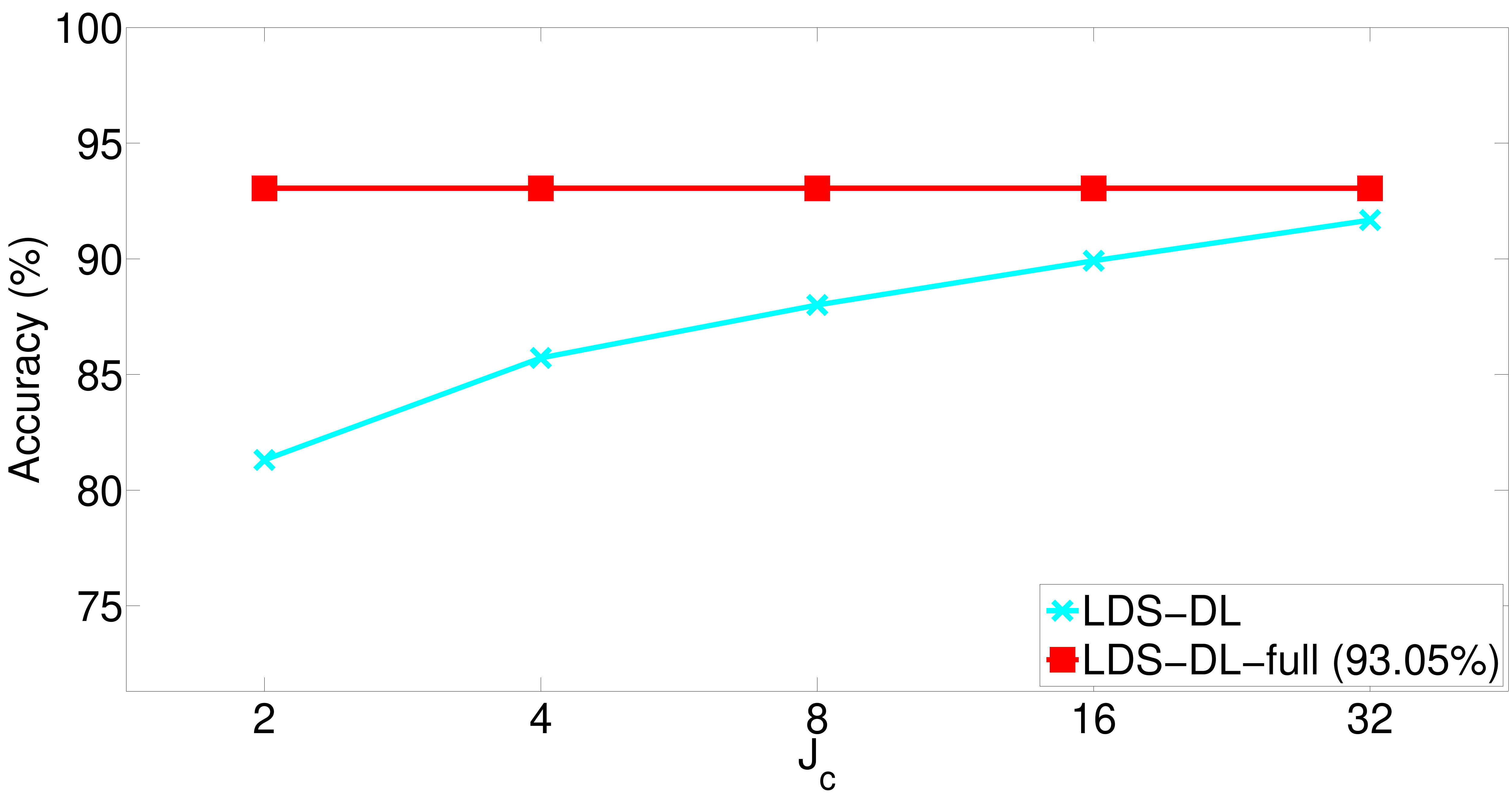}
}
\caption{Classification performance on the \emph{DynTex++} dataset.}
\label{Fig:DL_hierarchical_accuracy}
\end{center}
\end{figure}

\section{Conclusion}
\label{Sec:Conclusion}
In this paper, we address several shortcomings in modeling Linear Dynamical Systems (LDS). In particular, we
formulate the extended observability subspace of an LDS explicitly and devise an efficient method called SN to stabilize LDSs.
We then tackle two challenging problems, namely LDS coding and learning an LDS dictionary. In the former, the goal is to describe an LDS with
a sparse combination of a set of LDSs, known as a dictionary.
In the latter, we show how an LDS dictionary can be obtained from a set of LDSs.
Towards solving the aforementioned problems, we introduce the novel concept of two-fold LDSs and make use of it to obtain
closed-form updates for  learning an LDS dictionary.
Our extensive set of experiments shows the superiority of the proposed techniques compared to various state-of-the-art methods on different tasks including hand gesture recognition, dynamical scene classification, dynamic texture categorization and tactile recognition.

\newtheorem{property}[theorem]{Property}
\setcounter{equation}{24}
\setcounter{theorem}{0}
\appendix
\section{Proofs}
\label{Sec:proofs}
We present the proofs of Theorems \ref{Th:linear_combination}-\ref{Th:skew-column-U}. For better readability, we repeat the theorems before the proofs.
\begin{theorem}\label{Th:linear_combination}
Suppose $\Mat{V}_1,\Mat{V}_2,\cdots,\Mat{V}_M\in \mathcal{S}(n,\infty)$, and  $y_1,y_2,\cdots,y_M\in \mathbb{R}$, we have
\begin{eqnarray}
\nonumber \parallel\sum_{i=1}^{M}y_i\Pi(\Mat{V}_i)\parallel_{F}^2 &=& \sum_{i,j=1}^{M}y_iy_j\parallel\Mat{V}_i^{\mathrm{T}}\Mat{V}_j\parallel_{F}^2,
\end{eqnarray}
\end{theorem}

\begin{proof}
We denote the $t$-order sub-matrix of the extended observability matrix $\Mat{O}_i$ as $\Mat{O}_i(t) = [\Mat{C}_{i}^{\mathrm{T}},(\Mat{C}_{i}\Mat{A}_{i})^{\mathrm{T}},\cdots,(\Mat{C}_{i}\Mat{A}_{i}^{t})]^{\mathrm{T}}$. Suppose the factored matrix for orthogonalizing $\Mat{O}_i$ is $\Mat{L}_i$ (see \textsection 4.3 in the paper) and denote that $\Mat{V}_i(t)=\Mat{O}_i(t)\Mat{L}_i^{\mathrm{-T}}$. Then, we derive
\begin{eqnarray}\label{emnedding_proof}
\nonumber && \parallel\sum_{i=1}^{M}y_i\Pi(\Mat{V}_i)\parallel_{F}^2 \\
\nonumber
&& = \lim_{t\rightarrow\infty}\parallel\sum_{i=1}^{M}y_i\Mat{V}_i(t)\Mat{V}_i(t)^{\mathrm{T}}\parallel_{F}^2\\
\nonumber && =
\lim_{t\rightarrow\infty}\mathrm{Tr}\left(\sum_{i=1}^{M}y_i\Mat{V}_i(t)\Mat{V}_i(t)^{\mathrm{T}} \sum_{j=1}^{J}y_j\Mat{V}_j(t)\Mat{V}_j(t)^{\mathrm{T}}\right)\\
\nonumber  && = \lim_{t\rightarrow\infty}\sum_{i,j=1}^{M}y_iy_j\mathrm{Tr}\left(\Mat{V}_i(t)^{\mathrm{T}}\Mat{V}_j(t) \Mat{V}_j(t)^{\mathrm{T}}\Mat{V}_i(t)\right)\\
\nonumber  && = \sum_{i,j=1}^{M}y_iy_j\lim_{t\rightarrow\infty}\parallel\Mat{V}_i(t)^{\mathrm{T}}\Mat{V}_j(t)\parallel_F^2\\
\nonumber  && = \sum_{i,j=1}^{M}y_iy_j\lim_{t\rightarrow\infty}\parallel\Mat{L}_i^{-1}(\Mat{O}_i(t)^{\mathrm{T}} \Mat{O}_j(t))\Mat{L}_j^{\mathrm{-T}}\parallel_F^2\\
&& =
\sum_{i,j=1}^{M}y_iy_j\parallel\Mat{L}_i^{-1}\Mat{O}_{ij}\Mat{L}_j^{\mathrm{-T}}\parallel_F^2,
\end{eqnarray}
where the limitation value $\Mat{O}_{ij}=\lim_{t\rightarrow\infty}\Mat{O}_i(t)^{\mathrm{T}} \Mat{O}_j(t)=\Mat{O}_i^{\mathrm{T}}\Mat{O}_j$ exists and is computed by solving the Lyapunov equation similar to Eq.~(4).
\end{proof}

\begin{corollary}\label{metric}
For any $\Mat{V}_1,\Mat{V}_2\in\mathcal{S}(n,\infty)$, we have
\begin{eqnarray}
\nonumber \|\Pi(\Mat{V}_1)-\Pi(\Mat{V}_2)\|_{F}^2 &=& 2\left(n -\|\Mat{V}_1^{\mathrm{T}}\Mat{V}_2\|_F^2\right).
\end{eqnarray}
Furthermore, $\| \Pi(\Mat{V}_1)-\Pi(\Mat{V}_2)\|_{F}^2=2\sum_{k=1}^{n}\sin^2\alpha_k$, where $\{\alpha_k\}_{k=1}^n$ are subspace angles between $\Mat{V}_1$ and $\Mat{V}_2$. This also indicates that $0 \leq \| \Pi(\Mat{V}_1)-\Pi(\Mat{V}_2)\|_{F}^2 \leq 2n$.
\end{corollary}

The Frobenius distance $\parallel\Pi(\Mat{V}_1)-\Pi(\Mat{V}_2)\parallel_{F}^2$ is devised by setting $y_1=1$ and $y_2=-1$ in Eq.~\eqref{emnedding_proof}. As demonstrated in \cite{harandi2013dictionary}, the embedding $\Pi(\Mat{V})$ from the finite Grassmannian $\mathcal{G}(n,d)$ to the space of the symmetric matrices is proven to be diffeomorphism  (a one-to-one, continuous, and
differentiable mapping with a continuous and differentiable inverse); The  Frobenius distance between the two points $\Mat{V}_1$ and $\Mat{V}_2$  in the embedding space can be rewritten as
\begin{eqnarray}\label{Eq:Frobenius distance}
\parallel\Mat{V}_1\Mat{V}_1^{\mathrm{T}}-\Mat{V}_2\Mat{V}_2^{\mathrm{T}}\parallel_{F}^2=2\sum_{k=1}^{n}\sin^2\alpha_k,
\end{eqnarray}
where $\alpha_k$ is the $k$-th principal angle of the subspaces between $\Mat{V}_1$ and $\Mat{V}_2$.

We denote the space of the finite observability subspaces as $\mathcal{S}(n,L)$ by taking the first $L$ rows from the extended observability matrix. Clearly, $\mathcal{S}(n,L)$ is a compact subset of $\mathcal{G}(n,L)$. Hence $\mathcal{S}(n,L)$ maintains the relation in Eq. (\ref{Eq:Frobenius distance}); and the embedding $\Pi(\Mat{V})$ from $\mathcal{S}(n,L)$  to the space of the symmetric matrices is diffeomorphism. For our case, $\mathcal{S}(n,\infty)=\lim_{L\rightarrow\infty}\mathcal{S}(n,L)$. Theorem (\ref{Th:linear_combination}) proves the Frobenius distance defined in the embedding $\Pi(\mathcal{S}(n,\infty))$ to be convergent. Thus, we can obtain the relation between the Frobenius distance and the subspace angles in Corollary \ref{metric}, and prove that the embedding $\Pi(\mathcal{S}(n,\infty))$ is diffeomorphism, by extending the conclusions of $\mathcal{S}(n,L)$ with $L$ approaching to the infinity.

\begin{lemma}\label{canonical-representaion}
For a symmetric or skew-symmetric transition matrix $\Mat{A}$, the system tuple $(\Mat{A},\Mat{C})\in \mathbb{R}^{n \times n}\times\mathbb{R}^{m \times n}$ has the canonical form $(\Mat{\Lambda},\Mat{U})\in \mathbb{R}^{n}\times\mathbb{R}^{m \times n}$, where $\Mat{\Lambda}$ is diagonal and $\Mat{U}$ is unitary, \ie $\Mat{U}^\ast\Mat{U}= \Mat{I}$. Moreover, both $\Mat{\Lambda}$ and $\Mat{U}$ are real if $\Mat{A}$ is symmetric and complex if $\Mat{A}$ is skew-symmetric; for the skew-symmetric $\Mat{A}$, $(\Mat{\Lambda},\Mat{U})$ is parameterized by a real matrix-pair $(\Mat{\Theta},\Mat{Q})\in \mathbb{R}^{n}\times\mathbb{R}^{m \times n}$ where $\Mat{\Theta}$ is diagonal and $\Mat{Q}$ is orthogonal, and $n$ is even.
\end{lemma}
\begin{proof}
A symmetric or skew-symmetric matrix $\Mat{A}$ can be decomposed as $\Mat{A}=\Mat{U}\Mat{\Lambda}\Mat{U}^\ast$~\cite{Normal,Skew}.
Here $\Mat{U}$ is a unitary matrix and $\Mat{\Lambda}$ is a diagonal matrix storing the eigenvalues of $\Mat{A}$. Due to the invariance property of the system tuple, we attain that
\begin{eqnarray}
(\Mat{A},\Mat{C}) &=&     (\Mat{U}\Mat{\Lambda}\Mat{U}^\ast,\Mat{C}),\\
\nonumber &\sim&  (\Mat{\Lambda},\Mat{C}\Mat{U}),\\
          &=&     (\Mat{\Lambda},\Mat{U'}),
\end{eqnarray}
with ``$\sim$" denoting the equivalence relation. Clearly, $\Mat{U'}$ is also unitary. Thus, we ignore the difference between $\Mat{U'}$ and $\Mat{U}$ for consistency, and denote $\Mat{U'}$ as $\Mat{U}$ in the following context.

In particular, for a symmetric matrix, both $\Mat{U}$ and  $\Mat{\Lambda}$ are real; for an skew-symmetric matrix,
the diagonal elements of $\Mat{\Lambda}$ and the columns of $\Mat{U}$ are
\begin{eqnarray}\label{Eq:skew-decomposition-lambda}
[\Mat{\Lambda}]_{2k-1} = [\Mat{\Theta}]_k \cdot 1i;~~~[\Mat{\Lambda}]_{2k} = -[\Mat{\Theta}]_k \cdot 1i,
\end{eqnarray}
and
\begin{eqnarray}\label{Eq:skew-decomposition-U}
\left\{
\begin{array}{lll}
 &\left[\Mat{U}\right]_{2k-1} & = \frac{1}{\sqrt{2}}[\Mat{Q}]_{2k-1} + \frac{1i}{\sqrt{2}}[\Mat{Q}]_{2k}, \\
 &\left[\Mat{U}\right]_{2k}   & = \frac{1}{\sqrt{2}}[\Mat{Q}]_{2k-1} - \frac{1i}{\sqrt{2}}[\Mat{Q}]_{2k},
\end{array}
\right.
\end{eqnarray}
for $k=1,2,\cdots,\frac{n}{2}$ if $n$ is even\footnote{When $n$ is odd, the first $n-1$ columns of $\Mat{\Lambda}_{s}$ are the same as Eq. (\ref{Eq:skew-decomposition-lambda}); the $n$-th element of $\Mat{\Lambda}$ is $0$. Our developments can be applied verbatim to this case as well.}, where $\Mat{\Theta}\in\mathbb{R}^{\frac{n}{2}\times\frac{n}{2}}$ is a real diagonal matrix and $\Mat{Q}\in\mathbb{R}^{m\times n}$ is a real orthogonal matrix. Thus $(\Mat{\Lambda},\Mat{U})$ can be parameterized by $(\Mat{\Theta},\Mat{Q})$ which is of real value.
\end{proof}

\begin{theorem}\label{Th:DL_app}
For the two-fold LDS modeling defined in Eq.~(17), the sup-problem in Eq.~(16) is equivalent to
\begin{eqnarray}\label{Eq:DL_app}
\begin{aligned}
 & \min_{\Mat{U}_{r}^{(d)}, \Mat{\Lambda}_{r}^{(d)}}  &  &\sum_{k=1}^{n}[\Mat{U}_{r}^{(d)}]_k^*\Mat{S}_{r,k}[\Mat{U}_{r}^{(d)}]_k\\
 & ~~~\mathrm{s.t. }                       &  & (\Mat{U}_{r}^{(d)})^*\Mat{U}_{r}^{(d)}=\mathbf{I}_n;\\
 &&& -1<[\Mat{\Lambda}_{r}^{(d)}]_k<1,~1\leq k\leq n.
\end{aligned}
\end{eqnarray}
\vskip -0.1in \noindent
Here,  the matrix $\Mat{S}_{r,k}$ (see Table~2) is not dependent on the measurement matrix $\Mat{U}_{r}^{(d)}$.
\end{theorem}

\begin{proof}
We first derive several preliminary results prior to the proof of this theorem.

In the conventional LDS modeling, the exact form of the self-inner-product on extended observability matrices cannot be obtained due to the implicit calculation of the Lyapunov equation (Eq.~(4)). Thanks to Lemma~\ref{canonical-representaion}, the system tuple with symmetric or skew-symmetric transition matrix has the canonical representation $(\Mat{\Lambda},\Mat{U})$ where $\Mat{\Lambda}$ is diagonal and $\Mat{U}$ is unitary. Then, the self-inner-product of the corresponding extended observability matrix is computed as
\begin{eqnarray}
\nonumber \Mat{O}^{\mathrm{T}}\Mat{O}&=& \sum_{t=0}^{\infty}(\Mat{\Lambda}^{*})^t\Mat{\Lambda}^t, \\
&=& \mathrm{diag}([\frac{1}{1-|\lambda_{1}|^2}, \frac{1}{1-|\lambda_{2}|^2},\cdots,\frac{1}{1-|\lambda_{n}|^2}]),
\end{eqnarray}
where $\lambda_{k}=[\Mat{\Lambda}]_k$.

Further considering the SVD factorization $\Mat{O}^{\mathrm{T}}\Mat{O}=\Mat{U}_o\Mat{S}_o\Mat{U}_o^{\mathrm{T}}$, the factor matrix is obtained as $\Mat{L}=\Mat{U}_o\Mat{S}_o^{1/2}=\mathrm{diag}([\frac{1}{(1-|\lambda_{1}|^2)^{1/2}}, \frac{1}{(1-|\lambda_{2}|^2)^{1/2}},\cdots,\frac{1}{(1-|\lambda_{n}|^2)^{1/2}}])$. The orthogonal extended observability matrix is given by $\Mat{V}=\Mat{O}\Mat{L}^{-\mathrm{T}}$.

For simplicity, we here denote the canonical representation, the factor matrix and the extended observability subspace for dictionary atom $\Mat{D}_j$ as $(\Mat{\Lambda}_j,\Mat{U}_j)$, $\Mat{L}_j$ and $\Mat{V}_j$, respectively, by omitting the superscript $^{(d)}$. The kernel between dictionary atoms in Eq.~(15) is devised as
\begin{eqnarray} \label{Eq:krj_derivation}
\nonumber
&& K(\Mat{D}_r,\Mat{D}_j) \\
\nonumber
&=& \parallel\Mat{V}_r^{\mathrm{T}}\Mat{V}_j\parallel_F^2 \\
\nonumber
&=&
\parallel\Mat{L}_r^{-1}\sum_{t=0}^{\infty}(\Mat{\Lambda}_r^{*})^t\Mat{U}_r^{*} \Mat{U}_j\Mat{\Lambda}_j^t\Mat{L}_j^{-\mathrm{T}}\parallel_{F}^2 \\
\nonumber
&=&
\mathrm{Tr}\left( \Mat{L}_r^{-1}\left(\sum_{t_1=0}^{\infty}(\Mat{\Lambda}_r^{*})^{t_1}\Mat{U}_r^{*} \Mat{U}_j\Mat{\Lambda}_j^{t_1}\right)\Mat{L}_j^{-\mathrm{T}} \Mat{L}_j^{-1}\left(\sum_{t_2=0}^{\infty}(\Mat{\Lambda}_j^{*})^{t_2}\Mat{U}_j^{*} \Mat{U}_r\Mat{\Lambda}_r^{t_2}\right)\Mat{L}_r^{-\mathrm{T}}
\right)\\
\nonumber
&=&
\sum_{t_1=0}^{\infty}\sum_{t_2=0}^{\infty}\mathrm{Tr}\left(
\Mat{L}_r^{-1}(\Mat{\Lambda}_r^{*})^{t_1})\Mat{U}_r^{*} \Mat{U}_j\Mat{\Lambda}_j^{t_1}(\Mat{L}_j^{-\mathrm{T}} \Mat{L}_j^{-1})(\Mat{\Lambda}_j^{*})^{t_2}\Mat{U}_j^{*} \Mat{U}_r((\Mat{\Lambda}_r)^{t_2}\Mat{L}_r^{-\mathrm{T}})
\right)\\
\nonumber
&=&
\sum_{t_1=0}^{\infty}\sum_{t_2=0}^{\infty}\sum_{k=1}^n
(1-|\lambda_{r,k}|^2)^{\frac{1}{2}}(\lambda_{r,k}^{*})^{t_1}[\Mat{\hat{U}}_r]_k^{*} \Mat{U}_j\Mat{\Lambda}_j^{t_1}(\Mat{L}_j^{-\mathrm{T}} \Mat{L}_j^{-1})(\Mat{\Lambda}_j^{*})^{t_2}(\Mat{U}_j)^{*} [\Mat{U}_r]_k\lambda_{r,k}^{t_2}(1-|\lambda_{r,k}|^2)^{\frac{1}{2}}\\
\nonumber
&=&
\sum_{t_1=0}^{\infty}\sum_{t_2=0}^{\infty}\sum_{k=1}^n
[\Mat{U}_r]_k^{*} \Mat{U}_j((\lambda_{r,k}^{*})^{t_1}\Mat{\Lambda}_j^{t_1})
((1-|\lambda_{r,k}|^2)\Mat{L}_j^{-\mathrm{T}}\Mat{L}_j^{-1}) (\lambda_{r,k}^{t_2}(\Mat{\Lambda}_j^{*})^{t_2})
(\Mat{U}_j)^{*} [\Mat{U}_r]_k\\
\nonumber
&=&
\sum_{k=1}^n
[\Mat{U}_r]_k^{*} \Mat{U}_j\left(\sum_{t_1=0}^{\infty}(\lambda_{r,k}^{*})^{t_1}\Mat{\Lambda}_j^{t_1}\right)
((1-|\lambda_{r,k}|^2)\Mat{L}_j^{-\mathrm{T}} \Mat{L}_j^{-1}) \left(\sum_{t_2=0}^{\infty}\lambda_{r,k}^{t_2}(\Mat{\Lambda}_j^{*})^{t_2}\right)
(\Mat{U}_j)^{*} [\Mat{U}_r]_k\\
&=&
\sum_{k=1}^n
[\Mat{U}_r]_k^{*}\Mat{F}_{r,j,k}[\Mat{U}_r]_k,
\end{eqnarray}
where $\lambda_{j,k}=[\Mat{\Lambda}_j]_k$, $\Mat{F}_{r,j,k}=\Mat{U}_j\Mat{E}(\lambda_{r,k},\Mat{\Lambda}_{j})\Mat{U}_j^{*}$; and the function $\Mat{E}(\lambda,\Mat{\Lambda})$ returns a diagonal matrix with the elements given by
$
\Mat{E}(\lambda,\Mat{\Lambda}) = \mathrm{diag}( [\frac{(1-|\lambda|^2)(1-|\lambda_{1}|^2)}{|1-\lambda\lambda_{1}^{*}|^2},\cdots,\frac{(1-|\lambda|^2)(1-|\lambda_{n}|^2)}{|1-\lambda\lambda_{n}^{*}|^2}])
$ with $\lambda_{k}=[\Mat{\Lambda}]_k$.

The kernel function across the dictionary and data in Eq.~(15), \ie, $k(\Mat{D}_r,\Mat{X}_i)$  can be derived in a similar way as Eq.~\eqref{Eq:krj_derivation}.
Substituting above devised kernel values into Eq.~(15) deduces the objective function in Eq.~\eqref{Eq:DL_app}, thereby completing the proof of Theorem~\ref{Th:DL_app}.
\end{proof}

\begin{theorem}\label{Th:sym-column-U}
Let $[\Mat{U}_{1,r}^{(d)}]_{-k}\in \mathbb{R}^{m\times (n-1)}$ denote the sub-matrix obtained from $\Mat{U}_{1,r}^{(d)}$
by removing the $k$-th column, \ie,
\begin{align}
[\Mat{U}_{1,r}^{(d)}]_{-k} = \Big([\Mat{U}_{1,r}^{(d)}]_1;\cdots;[\Mat{U}_{1,r}^{(d)}]_{k-1}; [\Mat{U}_{1,r}^{(d)}]_{k+1};\cdots;[\Mat{U}_{1,r}^{(d)}]_n \Big).
\end{align}
Define $\Mat{W}=[\Mat{w}_{1},\cdots,\Mat{w}_{m-n+1}]\in\mathbb{R}^{m\times (m-n+1)}$ as the orthogonal complement basis of $[\Mat{U}_{1,r}^{(d)}]_{-k}$. If $\Mat{u}\in \mathbb{R}^{(m-n+1)}$ is the eigenvector of  $\Mat{W}^{\mathrm{T}}\Mat{S}_{r,k}\Mat{W}$ corresponding to the smallest eigenvalue, then $\Mat{W}\Mat{u}$ is the optimal solution of $[\Mat{U}_{1,r}^{(d)}]_{-k}$ for Eq.~(20).
\end{theorem}

\begin{proof}
Since $[\Mat{U}_{sym,r}^{(d)}]_k^{\mathrm{T}}[\Mat{U}_{sym,r}^{(d)}]_o=0$ for all $1\leq o\leq n, o\neq k$, then $[\Mat{U}_{sym,r}^{(d)}]_k$ lies in the orthogonal complement of the space spanned by the columns of $[\Mat{U}_{sym,r}^{(d)}]_{-k}$. Thus, there exists a vector $\Mat{u}\in \mathbb{R}^{(m-n+1)\times 1}$ satisfying $[\Mat{U}_{sym,r}^{(d)}]_k=\Mat{W}\Mat{u}$ and $\Mat{u}^{\mathrm{T}}\Mat{u}=1$. The objective function in Eq.~(20) becomes $\Mat{u}^{\mathrm{T}}(\Mat{W}^{\mathrm{T}}\Mat{S}_{r,k}\Mat{W})\Mat{u}$. Obviously, the optimal $\Mat{u}$ for minimizing this function is the eigenvector of the matrix $\Mat{W}^{\mathrm{T}}\Mat{S}_{r,k}\Mat{W}$ corresponding to the smallest eigenvalue.
\end{proof}

\begin{theorem}\label{Th:skew-column-U}
Let $[\Mat{Q}]_{-2}\in \mathbb{R}^{m\times (n-2)}$ be a sub-matrix of $\Mat{Q}$ obtained by removing the $(2k-1)$-th and $2k$-th columns, \ie,
\begin{eqnarray}
\label{Eq:Q-2}
[\Mat{Q}]_{-2}= \Big([\Mat{Q}]_1;\cdots;[\Mat{Q}]_{2k-2};[\Mat{Q}]_{2k+1};\cdots;[\Mat{Q}]_n \Big).
\end{eqnarray}
Define $\Mat{W}=[\Mat{w}_{1},\cdots,\Mat{w}_{m-n+2}]\in\mathbb{R}^{m\times (m-n+2)}$ as the orthogonal complement basis of $[\Mat{Q}]_{-2}$. If $\Mat{u}_1, \Mat{u}_2\in \mathbb{R}^{(m-n+1)}$ are the eigenvectors of  $\Mat{W}^{\mathrm{T}}\Mat{S'}_{r,k}\Mat{W}$ corresponding to the smallest two  eigenvalues, then $\Mat{W}\Mat{u}_1$ and $\Mat{W}\Mat{u}_2$ are the solutions of $[\Mat{Q}]_{2k-1}$ and $[\Mat{Q}]_{2k}$
in Eq.~(22), respectively.
\end{theorem}
\begin{proof}
The proof for this theorem is similar to that in Theorem \ref{Th:sym-column-U}.
\end{proof}

The rest of this section discusses the values of the two terms $\Mat{S}_{r,k}$ and $\delta_{r,k}$.

\begin{property}
We call a dictionary atom to be symmetric (resp. skew-symmetric) if its transition matrix is symmetric (resp. skew-symmetric). The matrix $\Mat{S}_{r,k}$ in Eq.~\eqref{Eq:DL_app} is claimed to satisfy:
\begin{enumerate}
  \item $\Mat{S}_{r,k}$ is real and symmetric, if the dictionary atom $\Mat{D}_r$ is symmetric;
  \item $\Mat{S'}_{r,k}=\frac{1}{2}(\Mat{S}_{r,2k-1}+\Mat{S}_{r,2k})$ is real and symmetric, if $\Mat{D}_r$ is skew-symmetric.
\end{enumerate}
\end{property}

\begin{proof}
We will see that $\Mat{F}_{r,j,k}$ is real and symmetric for two cases when 1) $\Mat{D}_j$ is symmetric or 2) $\Mat{D}_j$ is skew-symmetric and $\Mat{D}_r$ is symmetric. Clearly, $\Mat{F}_{r,j,k}=\Mat{U}_j\Mat{E}(\lambda_{r,k},\Mat{\Lambda}_{j})\Mat{U}_j^{*}$ is real and symmetric if $\Mat{D}_j$ is symmetric, since both $\Mat{U}_j$ and $\Mat{E}_{r,j,k}$ are real. If $\Mat{D}_j$ is skew-symmetric and $\Mat{D}_r$ is symmetric, $\Mat{E}(\lambda_{r,k},\Mat{\Lambda}_{j})$ should have the form as
\begin{eqnarray}\label{Eq:E}
\Mat{E}(\lambda_{r,k},\Mat{\Lambda}_{j}) = \left(
                   \begin{array}{ccccc}
                     a_1 &  &  &  &  \\
                      & a_1 &  &  &  \\
                      &  & \ddots &  &  \\
                      &  &  & a_{n/2} &  \\
                      &  &  &  &  a_{n/2}\\
                   \end{array}
                 \right),
\end{eqnarray}
where $a_i=\frac{(1-|\lambda_{r,k}|^2)(1-|\theta_i|^2)}{1+(\lambda_{r,k}\theta_{i})^2}$ and $\theta_i$ is the $i$-th diagonal value of $\Mat{\Theta}$ obtained by the decomposition of $\Mat{\Lambda}_j$ in Eq.~\eqref{Eq:skew-decomposition-lambda};  $n$ have assumed to be even; when $n$ is odd, we only need to add zeros in the final row. Recalling that $\Mat{U}_j$ have the decomposition by $\Mat{Q}$ in Eq.~\eqref{Eq:skew-decomposition-U}, then
\begin{eqnarray}
\Mat{F}_{r,j,k} = \Mat{U}_j\Mat{E}(\lambda_{r,k},\Mat{\Lambda}_{j})\Mat{U}_j^{*}=\Mat{Q}\Mat{T}\Mat{E}(\lambda_{r,k},\Mat{\Lambda}_{j})\Mat{T}^\ast\Mat{Q}^{\mathrm{T}},
\end{eqnarray}
where
\begin{eqnarray}
\Mat{T} = \left(
                   \begin{array}{ccccc}
                     \frac{1}{\sqrt{2}}  & \frac{1}{\sqrt{2}}   &  &  &  \\
                     \frac{1i}{\sqrt{2}} & \frac{-1i}{\sqrt{2}} &  &  &  \\
                      &  & \ddots &  &  \\
                      &  &  & \frac{1}{\sqrt{2}}   & \frac{1}{\sqrt{2}}  \\
                      &  &  &  \frac{1i}{\sqrt{2}} & \frac{-1i}{\sqrt{2}}\\
                   \end{array}
                 \right),
\end{eqnarray}
$\Mat{T}\Mat{E}(\lambda_{r,k},\Mat{\Lambda}_{j})\Mat{T}^\ast$ is verified to real and symmetric, thus proving that $\Mat{F}_{r,j,k}$ is real and symmetric.

If both $\Mat{D}_j$ and $\Mat{D}_r$ are skew-symmetric, $\Mat{E}(\lambda_{r,k},\Mat{\Lambda}_{j})$ does \textbf{not} have the form of Eq (\ref{Eq:E}) (the values of the odd diagonal elements are equal to those of the even ones); however, $\frac{1}{2}(\Mat{E}(\lambda_{r,2k-1},\Mat{\Lambda}_{j})+ \Mat{E}(\lambda_{r,2k},\Mat{\Lambda}_{j}))$ still have the same form as Eq (\ref{Eq:E}), meaning that $\frac{1}{2}(\Mat{F}_{r,j,2k-1}+ \Mat{F}_{r,j,2k})$ is real.

In sum, we have derived that 1) when $\Mat{D}_r$ is symmetric, $\Mat{F}_{r,j,k}$ is real and symmetric; 2) when $\Mat{D}_r$ is skew-symmetric, $\frac{1}{2}(\Mat{F}_{r,j,2k-1}+ \Mat{F}_{r,j,2k})$ is real and symmetric. Such results can be extrapolated to $\Mat{F'}_{r,i,k}$. As shown in Table~2 in the paper, $\Mat{S}_{r,k}$ is obtained by a linear combination of $\Mat{F}_{r,j,k}$, $\Mat{F'}_{r,i,k}$, where the combination coefficients are real. Hence, $\Mat{S}_{r,k}$ is real and symmetric for symmetric dictionary atom; and $\Mat{S'}_{r,k}=\frac{1}{2}(\Mat{S}_{r,2k-1}+\Mat{S}_{r,2k})$ is real and symmetric for skew-symmetric dictionary atom.
\end{proof}

In our dictionary learning algorithm (\textsection~6.2), for the update of the skew-symmetric dictionary, we decompose the dictionary tuple $(\Mat{\Lambda}_r, \Mat{U}_r)$ with the real tuple $(\Mat{\Theta}, \Mat{Q})$, and specify Eq.~\eqref{Eq:DL_app} as Eq.~(22), where $\delta_{r,k}$ is claimed to be a small term that can be ignored. We present more details here.

By the use of Eq. (\ref{Eq:skew-decomposition-lambda}-\ref{Eq:skew-decomposition-U}), the objective function in Eq.~(22) is devised as
\begin{eqnarray}
\nonumber
&&
[\Mat{U}_{r}]_{2k-1}^\ast\Mat{S}_{r,2k-1}[\Mat{U}_{r}]_{2k-1} + [\Mat{U}_{r}]_{2k}^\ast\Mat{S}_{r,2k}[\Mat{U}_{r}]_{2k}\\
\nonumber
&=&
\frac{1}{2}[\Mat{Q}]_{2k-1}^{\mathrm{T}}\Mat{S}_{r,2k-1}[\Mat{Q}]_{2k-1} + \frac{1}{2}[\Mat{Q}]_{2k}^{\mathrm{T}}\Mat{S}_{r,2k-1}[\Mat{Q}]_{2k} + \\
\nonumber
&&
\frac{1i}{2}[\Mat{Q}]_{2k-1}^{\mathrm{T}}\Mat{S}_{r,2k-1}[\Mat{Q}]_{2k} -
\frac{1i}{2}[\Mat{Q}]_{2k}^{\mathrm{T}}\Mat{S}_{r,2k-1}[\Mat{Q}]_{2k-1} + \\
\nonumber
&&
\frac{1}{2}[\Mat{Q}]_{2k-1}^{\mathrm{T}}\Mat{S}_{r,2k}[\Mat{Q}]_{2k-1} + \frac{1}{2}[\Mat{Q}]_{2k-1}^{\mathrm{T}}\Mat{S}_{r,2k}[\Mat{Q}]_{2k} - \\
\nonumber
&&
\frac{1i}{2}[\Mat{Q}]_{2k-1}^{\mathrm{T}}\Mat{S}_{r,2k}[\Mat{Q}]_{2k} +
\frac{1i}{2}[\Mat{Q}]_{2k}^{\mathrm{T}}\Mat{S}_{r,2k}[\Mat{Q}]_{2k-1}, \\
&=&
[\Mat{Q}]_{2k-1}^{\mathrm{T}}\Mat{S'}_{r,k}[\Mat{Q}]_{2k-1} + [\Mat{Q}]_{2k}^{\mathrm{T}}\Mat{S'}_{r,k}[\Mat{Q}]_{2k}+\delta_{r,k},
\end{eqnarray}
where $\Mat{S'}_{r,k}=\frac{1}{2}(\Mat{S}_{r,2k-1}+\Mat{S}_{r,2k})$; and $\delta_{r,k}=\frac{1i}{2}\left([\Mat{Q}]_{2k-1}^{\mathrm{T}}\Mat{S''}_{r,k}[\Mat{Q}]_{2k} - ([\Mat{Q}]_{2k-1}^{\mathrm{T}}\Mat{S''}_{r,k}[\Mat{Q}]_{2k})^\ast\right)$ with $\Mat{S''}_{r,k}=\Mat{S}_{r,2k-1}-\Mat{S}_{r,2k}$.

The optimal pairs $([\Mat{\Theta}]_{k}, [\Mat{Q}]_{2k-1})$ and $([\Mat{\Theta}]_{k}, [\Mat{Q}]_{2k})$ are given by solving
\begin{eqnarray}\label{Eq:DL-skew-1}
\begin{aligned}
 &\min_{[\Mat{Q}]_{2k-1:2k}, [\Mat{\Theta}]_{k}} &  & [\Mat{Q}]_{2k-1}^{\mathrm{T}}\Mat{S'}_{r,k}[\Mat{Q}]_{2k-1} + [\Mat{Q}]_{2k}^{\mathrm{T}}\Mat{S'}_{r,k}[\Mat{Q}]_{2k} + \delta_{r,k}\\
 &~~~~~~\mathrm{s.t. }  &  & \Mat{Q}^{\mathrm{T}}\Mat{Q}=\Mat{I}_n.
\end{aligned}
\end{eqnarray}

The orthogonal condition  $\Mat{Q}^{\mathrm{T}}\Mat{Q}=\Mat{I}_n$ provides that $[\Mat{Q}]_{2k-1}$ and $[\Mat{Q}]_{2k}$ line in the orthogonal complement of the space spanned by the columns of $[\Mat{Q}]_{-2}$ (Eq.~\eqref{Eq:Q-2}). As a consequence, there exist unit vectors $\Mat{u}_1,\Mat{u}_2 \in\mathbb{R}^{n\times n}$ satisfying $[\Mat{Q}]_{2k-1}=\Mat{W}\Mat{u}_1$ and $[\Mat{Q}]_{2k}=\Mat{W}\Mat{u}_2$. Problem~(\ref{Eq:DL-skew-1}) is rewritten as
\begin{eqnarray}\label{Eq:DL-skew-2}
\begin{aligned}
 &\min_{\Mat{u}_1, \Mat{u}_2} &  & \Mat{u}_1^{\mathrm{T}}\Mat{W}^{\mathrm{T}}\Mat{S'}_{r,k}\Mat{W}\Mat{u}_1 + \Mat{u}_2^{\mathrm{T}}\Mat{W}^{\mathrm{T}}\Mat{S'}_{r,k}\Mat{W}\Mat{u}_2 + \delta_{r,k}\\
 &~~~~~~\mathrm{s.t. }  &  & [\Mat{u}_1, \Mat{u}_2]^{\mathrm{T}}[\Mat{u}_1, \Mat{u}_2]=\Mat{I}_2,
\end{aligned}
\end{eqnarray}
where $\delta_{r,k} = \frac{1i}{2}\left(\Mat{u}_1^{\mathrm{T}}\Mat{W}^{\mathrm{T}}\Mat{S''}_{r,k}\Mat{W}\Mat{u}_2 - (\Mat{u}_1^{\mathrm{T}}\Mat{W}^{\mathrm{T}}\Mat{S''}_{r,k}\Mat{W}\Mat{u}_2)^\ast\right)$. This is actually a quadratic problem with orthogonal constraints, which can be addressed by the gradient-based method proposed in \cite{edelman1998geometry} or the method minimizing a quadratic over a sphere~\cite{hager2001minimizing}. In this paper, however, we remove the term  $\delta_{r,k}$; thus the optimal $\Mat{u}_1$ and $\Mat{u}_2$ are obtained as the eigenvalues of $\Mat{W}^{\mathrm{T}}\Mat{S'}_{r,k}\Mat{W}$ as shown in Theorem \ref{Th:skew-column-U}.

Our motivations to ignore the term $\delta_{r,k}$ are supported by the following property.
\begin{property}
The value of $\delta_{r,k}$ is small, and more specifically,
\begin{enumerate}
  \item If $\Mat{S''}_{r,k}$ is a real matrix, $\delta_{r,k}$ is equal to zero.
  \item If $\Mat{S''}_{r,k}$ is not real, $|\delta_{r,k}|$ is bounded by a small value.
\end{enumerate}
\end{property}

\begin{proof}
It is straightforward to verify 1) by the definition of $\delta_{r,k}$. For 2), we denote the the imaginary part of $\Mat{S''}_{r,k}$ as $\Mat{S''}_{im,r,k}$ and derive the magnitude $|\delta_{r,k}|$ by
\begin{eqnarray}\label{Eq:delta}
\nonumber
|\delta_{r,k}| &=& |\Mat{u}_1^{\mathrm{T}}\Mat{W}^{\mathrm{T}}\Mat{S''}_{im,r,k}\Mat{W}\Mat{u}_2| \\
\nonumber
&\leq&
\sum_{i=1}^{I}|\Mat{Z}_{r,i}|(\sum_{j\in\text{skew-symmetric atoms}, j\neq i}|\Mat{Z}_{j,i}|\|\Mat{F}_{r,j,2k-1}-\Mat{F}_{r,j,2k}\|_2 +\|\Mat{F}_{r,i,2k-1}^{(1)}-\Mat{F}_{r,i,2k}^{(2)}\|_2).\\
\end{eqnarray}
We can further derive
\begin{eqnarray}\label{Eq:F}
\nonumber
\|\Mat{F}_{r,j,2k-1}-\Mat{F}_{r,j,2k}\|_2 &=& \|\Mat{U}_j(\Mat{E}(\lambda_{r,2k-1},\Mat{\Lambda}_{j})-\Mat{E}(\lambda_{r,2k},\Mat{\Lambda}_{j}))\Mat{U}_j^{*} \|_2,\\
\nonumber
&\leq&
\|\Mat{E}(\lambda_{r,2k-1},\Mat{\Lambda}_{j})-\Mat{E}(\lambda_{r,2k},\Mat{\Lambda}_{j})\|_2,\\
&=&
\|\mathrm{diag}( b_1 ,\cdots, b_n)\|_2,
\end{eqnarray}
where $ b_i = \frac{(1-|\lambda_{r,2k-1}|^2)(1-|\lambda_{j,i}|^2)}{|1-\lambda_{r,2k-1}\lambda_{j,i}^\ast|^2}] -\frac{(1-|\lambda_{r,2k}|^2)(1-|\lambda_{j,i}|^2)}{|1-\lambda_{r,2k}\lambda_{j,i}^\ast|^2}$, for $i=1,2,\cdots,n$. Since both $\Mat{D}_r$ and $\Mat{D}_j$ are skew-symmetric, then
$$
\lambda_{r,2k-1}=|\lambda_{r,2k-1}|*1i; \lambda_{r,2k}=-|\lambda_{r,2k-1}|*1i;
\lambda_{j,i}=|\lambda_{j,i}|*1i; \lambda_{r,2k}=-|\lambda_{j,i}|*1i;
$$
$$
|1-\lambda_{r,2k-1}\lambda_{j,i}^\ast|=1-|\lambda_{r,2k-1}||\lambda_{j,i}|;
|1-\lambda_{r,2k}\lambda_{j,i}^\ast|=1+|\lambda_{r,2k}||\lambda_{j,i}|.
$$
Hence,
\begin{eqnarray}\label{Eq:b}
|b_i| &=& \frac{4(1-|\lambda_{r,2k-1}|^2)(1-|\lambda_{j,i}|^2)}{(1-|\lambda_{r,2k-1}||\lambda_{j,i}|)^2(1+|\lambda_{r,2k}||\lambda_{j,i}|)^2}|\lambda_{r,2k-1}||\lambda_{j,i}|.
\end{eqnarray}
It shows that, $|b_i|$ only reaches the value of 1 at the extreme case when $|\lambda_{r,2k-1}|=|\lambda_{j,i}|=1$; and practically it is close to zero if either $|\lambda_{r,2k-1}|$ or $|\lambda_{j,i}|$ is close to zero.

Substituting Eq.~\eqref{Eq:b} back to Eq.~\eqref{Eq:F} and Eq.~\eqref{Eq:delta}, we can conclude that $|\delta_{r,k}|$ is bounded by a small value. Moreover, our experiments also demonstrate that neglecting the term $\delta$ does not harm the convergence behavior of the dictionary learning algorithm (see figure~6).
\end{proof}

\section{The properties of two-fold LDSs}
\label{Sec:two-fold-LDSs}
The advantageous properties of two-fold LDSs are described in \textsection~6.2. Here we present that both $\Mat{A}_{sym}$ and $\Mat{A}_{skew}$ are guaranteed to be stable if $\Mat{A}$ is stabilized via the SN method. Note that the SN stabilization devises
\begin{eqnarray}
\parallel\Mat{A}\parallel_2=\sqrt{\mu_{max}(\Mat{A}^{\mathrm{T}}\Mat{A})}=\sigma_{max}(\Mat{A})<1,
\end{eqnarray}
where $\mu_{max}$ and $\sigma_{max}$ compute the maximized eigenvalue and singular-value, respectively. Thus,
\begin{eqnarray}
\parallel\Mat{A}_{sym}\parallel_2=\parallel\frac{1}{2}(\Mat{A}+\Mat{A}^{\mathrm{T}})\parallel_2\leq \frac{1}{2}(\parallel\Mat{A}\parallel_2+\parallel\Mat{A}^{\mathrm{T}}\parallel_2)<1.
\end{eqnarray}
Similarly, $\parallel\Mat{A}_{skew}\parallel_2<1$.

\bibliography{main}

\end{document}